\documentclass[a4paper,11pt]{article}

\usepackage{amsmath,amssymb,amsthm,bbm,bm}  
\usepackage{mathrsfs}
\usepackage{authblk}
\usepackage[british]{babel}
\usepackage{balance}
\usepackage[utf8]{inputenc}

\usepackage[sorting=none,doi=false,isbn=false,giveninits=true,style=nature,url=false]{biblatex} 
\newcommand{\citet}[1]{\textcite{#1}}    
\newcommand{\citep}[1]{\parencite{#1}}

\renewbibmacro*{name:andothers}{%
  \ifboolexpr{
    test {\ifnumequal{\value{listcount}}{\value{liststop}}}
    and
    test \ifmorenames
  }
    {\ifnumgreater{\value{liststop}}{1}{\finalnamedelim}{}%
     \printdelim{andothersdelim}\bibstring{andothers}}
    {}}
\DefineBibliographyStrings{british}{
  andothers = {et al.},
}

\usepackage{booktabs}
\usepackage{caption}
\usepackage{color}
\usepackage{comment}
\usepackage{dsfont}
\usepackage{enumitem}
\usepackage[T1]{fontenc}
\usepackage[top=2.5cm, bottom=2.5cm, left=2.5cm, right=2.5cm]{geometry}
\usepackage{graphicx}
\usepackage{graphbox}
\PassOptionsToPackage{hyphens}{url}  
\usepackage[colorlinks = true,
            linkcolor = ForestGreen,
            urlcolor  = blue,
            citecolor = blue,
            anchorcolor = blue]{hyperref}
            
\usepackage[capitalise]{cleveref}  
\usepackage{csquotes}  
\usepackage[mono=false]{libertine}

\usepackage{mathtools}
\usepackage{microtype}
\usepackage{multirow}
\usepackage{nicefrac}
\usepackage{physics}
\usepackage{subfigure}
\usepackage{wrapfig}
\usepackage[usenames,dvipsnames]{xcolor}
\usepackage{svg}
\usepackage{tikz}

\definecolor{C0}{HTML}{1f77b4}
\definecolor{C1}{HTML}{ff7f0e}
\definecolor{C2}{HTML}{2ca02c}
\definecolor{C3}{HTML}{d62728}
\definecolor{C4}{HTML}{9467bd}
\definecolor{C5}{HTML}{8c564b}
\definecolor{C6}{HTML}{e377c2}
\definecolor{C7}{HTML}{7f7f7f}
\definecolor{C8}{HTML}{bcbd22}
\definecolor{C9}{HTML}{17becf}

\definecolor{darkgreen}{rgb}{0,0.5,0}
\definecolor{CornflowerBlue}{RGB}{100,149,237}

\definecolor{violet}{HTML}{9ca1de}
\definecolor{reddish}{HTML}{ffa59e}
\definecolor{greenish}{HTML}{96b042}

\definecolor{gaussian}{HTML}{47908C}

\newtheorem{theorem}{Theorem}
\newtheorem{lemma}[theorem]{Lemma}
\newtheorem{proposition}[theorem]{Proposition}
\newtheorem{remark}[theorem]{Remark}

\newtheorem{definition}[theorem]{Definition}
\newtheorem{conjecture}[theorem]{Conjecture}

\newtheorem{innercustomthm}{Theorem}
\newenvironment{customthm}[1]
{\renewcommand\theinnercustomthm{#1}\innercustomthm}
{\endinnercustomthm}
\newtheorem{innercustomprop}{Proposition}
\newenvironment{customprop}[1]
{\renewcommand\theinnercustomprop{#1}\innercustomprop}
{\endinnercustomprop}

\newcommand{\legendline}[2][1pt]{%
  \textcolor{#2}{\raisebox{0.5ex}{\rule{0.8em}{#1}}}%
}

\newcommand{\dashedlegendline}[2][1pt]{%
  \textcolor{#2}{%
    \tikz[baseline=-0.6ex]{\draw[dashed, line width=#1] (0,0) -- (0.8em,0);}%
  }%
}

\title{A Fourier perspective on the learning dynamics of neural networks: from sample complexities to mechanistic insights}

\author{Fabiola Ricci}
\author{Claudia Merger}
\author{Sebastian Goldt
\thanks{\{fricci, cmerger, sgoldt\}@sissa.it}}
\affil{International School of Advanced Studies (SISSA), Trieste, Italy}

\date{\today}
\begin{document}
\maketitle

\begin{abstract}
  \noindent Neural networks trained with gradient-based methods exhibit a strong simplicity
bias: they learn simpler statistical features of their data before moving to
more complex features. Previous analyses of this phenomenon have largely focused
on settings with (quasi-)isotropic inputs. In this work, we study the simplicity
bias from a Fourier perspective, which allows us to include two key features of
natural images in the analysis: approximate translation-invariance and power-law
spectra. We first show experimentally that simple neural networks trained on
image classification tasks first rely on amplitude information -- related to
pair-wise correlations between pixels -- before exploiting phase information,
which encodes edges and higher-order correlations. In view of this, we introduce
a synthetic data model for translation-invariant inputs that allows precise
control over amplitudes and phases while remaining tractable. We rigorously
establish that for isotropic and high-dimensional inputs, classification based
on phase information alone is a genuinely hard task: online stochastic gradient
descent (SGD) cannot distinguish the structured inputs from noise within $n \ll
N^3$ steps, but needs at least $n \gg N^3 \log^2{N}$ steps. In contrast, we show
both experimentally and theoretically that power-law spectra can dramatically
accelerate the speed of learning phase information, even if the spectra do not
help with classification. Simulations with two-layer networks trained on
textures and with deep convolutional networks on ImageNet and CIFAR100 confirm
this non-trivial interaction between amplitudes and phases, providing
mechanistic insights into how deep neural networks can learn natural image
distributions efficiently.

\end{abstract}

\begin{figure*}[!htb]
    \centering
    \includegraphics[width=1.01\linewidth]{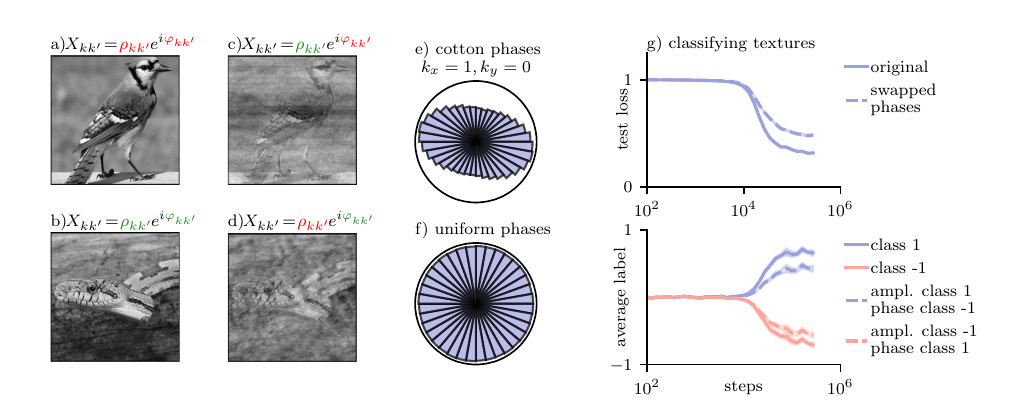}
    \caption{\label{fig:fig1} \textbf{Learning phase vs amplitude information.} \textbf{a-b)} Pictures of a bird and a snake from ImageNet. 
    \textbf{c-d)} Fourier image reconstruction with phases $\textcolor{red}{\varphi_{kk'}}$ from the ``bird'' and amplitudes ${\color[RGB]{34, 139, 34} \rho_{kk'}}$ from the ``snake'' and vice versa. 
    \textbf{e)} Phases of images from the ``cotton'' class of the ALOT texture dataset for patches of size $16\times16$ along the first Fourier mode in the $x$-direction. \textbf{f)} Uniform phase distribution. \textbf{g)} Performance of a classifier trained to distinguish between two texture classes: ``cotton'' vs. ``lace'' on patches of size $16\times16$ of the ALOT dataset. \textbf{Top:} test loss over training steps, on the original data (\legendline[1.5pt]{violet}) and on data where phase information has been swapped between classes (\dashedlegendline[1.5pt]{violet}). \textbf{Bottom:} average label assigned to samples from each class (\legendline[1.5pt]{violet}, \legendline[1.5pt]{reddish}) and samples with mixed amplitude/phase information (\dashedlegendline[1.5pt]{violet}, \dashedlegendline[1.5pt]{reddish}). All curves are averages over $8$ random initializations of the classifier.}
    \vspace*{-1em}
\end{figure*}

\section{Introduction}

A universal theme in the learning dynamics of neural networks trained with
gradient-based methods are \emph{simplicity biases}: for example, neural networks
reliably learn simpler statistical features of their data before moving to more
complex features, both in image classification~\citep{kalimeris2019sgd,
ingrosso2022data, refinetti2023neural} and in next-token
prediction~\citep{rende2024distributional, belrose2024neural, favero_how_2025,
garnier-brun2025how}. A rigorous characterisation of these biases is essential
to understand how neural networks can learn ``natural'' data distributions like
text and images efficiently, and which features they learn from their data.

From a theoretical point of view, simplicity biases were first analysed in simple models of neural networks learning a target function over Gaussian inputs~\citep{saad1995a, saxe2014exact,
saxe2019mathematical, abbe2021staircase, abbe2023sgd, dandi2024two, berthier2025learning}, in
auto-encoders~\citep{kogler2024compression}, and in the kernel
regime~\citep{farnia2018spectral, rahaman2019spectral}. More recently, a
\emph{distributional simplicity bias}, whereby neural networks first rely on
pair-wise input statistics before exploiting higher-order input correlations was
analysed in non-Gaussian models of inputs~\citep{ingrosso2022data,
merger_learning_2023, bardone2024sliding, ricci2025reduce}. However, the input models in these
works do not account for several important properties of natural images, such as
their (approximate) translation-invariance or the power-law decay of their power
spectra~\citep{vanderschaaf1996, hyvarinen2009natural}.

Here, we introduce a different perspective on modelling inputs for the theoretical analysis of neural networks by taking a Fourier point of view. Natural images are approximately translation-invariant, so frequency-space representations are both a simple and a natural choice for this data. The Fourier representation is particularly useful for investigating distributional simplicity biases because it provides a clean separation of pair-wise and higher-order statistics: for translation-invariant inputs like image patches, the amplitudes of the Fourier coefficients determine pair-wise correlations, while higher-order correlations are encoded only by the phases~\citep{hyvarinen2009natural}; see \cref{sec:fourier-model} for precise definitions. 

The different roles of amplitudes and phases can also be seen directly: in
\cref{fig:fig1} a-b), we show images of a bird and a snake, taken from ImageNet.
Next to them, we show two images that we obtained by combining the amplitudes of
one image with the amplitudes of the other in Fourier space, and then
transforming back into pixel space.
In both cases, the resulting images look like noisy versions of the image
whose phases we reused. This is a classic illustration of the importance of
phase information for human perception, which has long been established in
cognitive science~\citep{Oppenheim1981, piotrowski82phase}. 

In this work, we analyse the learning dynamics of simple neural networks trained
on image classification tasks from a Fourier perspective. We ask: how hard is it
to learn from amplitudes and from phases, respectively? Are they learnt
sequentially? And what is the impact of other salient features of images, like
the power-law decay of their spectrum~\citep{hyvarinen2009natural}, on their
learning dynamics? 

Our \textbf{first main contribution} is to demonstrate that neural networks
sequentially learn to exploit first amplitude, then phase information. To this
end, we train a fully connected two-layer classifier to distinguish between
image patches containing textures ``cotton'' and ``lace'' from the ALOT dataset
\citep{burghouts_material-specific_2009}. The textures in this dataset are
approximately translation-invariant, while the phases are anisotropic
(\cref{fig:fig1} e-f)). We evaluate the performance of these classifiers on a
held-out test set and on the same test set with the image phases swapped between
classes, analogous to the bird and snake example in \cref{fig:fig1} a-d). We
compare the test performance on the original and the phase-swapped datasets in
\cref{fig:fig1} g); details on the training procedure are provided in
\cref{app:details_training} in the appendix. We find that initially -- in the first phase of
learning after a search phase -- the network achieves comparable performance on
both test sets, indicating that it is using only the amplitude information to
distinguish the two classes. As training progresses, the performance on the
phase-swapped test set deteriorates in comparison to the original data, meaning
that the network has learnt to exploit phase information. This demonstrates
sequential learning of first amplitude, then phase information in neural
networks. In what follows, we build a theoretical model that elucidates that
amplitude and phase information correspond to lower and higher-order statistics
of a dataset, respectively, and quantifies the sample complexity at which this
information becomes available to a neural network. Our \textbf{further main
contributions} are as follows:

\begin{itemize}
\item We introduce a synthetic model for non-Gaussian, translation-invariant inputs which allows us to precisely control pair-wise and higher-order statistics via amplitude and phase manipulations (\cref{sec:fourier-model}).
\item We prove that when the inputs have isotropic covariance, weakly recovering
information carried exclusively by the phases requires a sample complexity on
the order of $n \gg N^3$ for online SGD, making this a hard task in high
dimensions (\cref{sec:isotropic}).
\item For non-isotropic inputs, we show experimentally that a power-law spectrum
speeds up the recovery of phase information in shallow and deep convolutional
neural networks trained on a variety of datasets
(\cref{sec:experiments-power-law}).
\item We show theoretically that neural networks can achieve quasi-linear sample
complexity for weak recovery of phase information when inputs have a power-law
spectrum (\cref{sec:theory-power-law}).
\end{itemize}

From a technical point of view, much of our results rely on an analysis of the stochastic dynamics of SGD by identifying suitable invariants of the loss, most notably the \textit{information exponent} introduced and developed by \citet{benarous2021online, benarous2022high} in the context of isotropic inputs with a single non-trivial direction. A key innovation in the present work is an analysis of the population loss when the non-Gaussian inputs have a general non-isotropic covariance matrix, which is essential to establish our results on the speed-up of learning due to the distribution of amplitudes in real images. 

\subsection*{Further related works} 

\paragraph{Frequency bias of deep neural networks} Previous work established
simplicity biases in Fourier space~\citep{rahaman2019spectral}. The crucial
difference to our work is that they consider a simplicity bias in terms of the
frequencies of the different Fourier modes of the data. Here, we are instead
interested in sequential learning of amplitudes vs.\ phases, irrespective of the
frequency of the given mode. For works studying (linear) CNNs from a Fourier
perspective, see \citet{pinson2023linear, gunasekar2018implicit}.
\vspace*{-.5em}

\paragraph{Perceptual role of amplitude vs.\ phase information} Cognitive
scientists have long been interested in the different perceptual roles played by
phases and amplitudes in visual perception. While early work suggested that
humans use only pair-wise information to discriminate structured visual stimuli
from pure noise~\citep{julesz1962visual}, more recent work showed that for
humans, phase information is perceptually dominant and determines the structure
of an image~\citep{Oppenheim1981, piotrowski82phase}. More recently,
\citet{tkavcik2010local, caramellino2021rat} have investigated the perceptual
role of pair-wise and higher-order correlations more systematically. In contrast
to humans, \citet{geirhos2018imagenet} empirically validates that
ImageNet-trained CNNs are strongly biased towards recognising textures rather
than shapes. 
\vspace*{-.5em}

\paragraph{The impact of power-law covariance on SGD} There has been a series of
works on the impact of power-law spectra on the learning dynamics with SGD using
methods related to ours.
\citet{paquette2024fourplusthree} studied scaling laws in a power-law random
feature model trained with SGD. \citet{braun2025fast} studied the dynamics of
phase retrieval on Gaussian inputs with power-law covariance.
\citet{arous2026learning, ren2026emergence, defilippis2025scaling} studied
scaling laws of the loss when training on isotropic Gaussian inputs with labels
provided by a two-layer teacher network whose second-layer weights follow a
power-law.  The key novelty of our work is to include the power-law in the
spectrum of \emph{non-Gaussian} inputs, which cannot be reduced to a standard
Gaussian additive model, see Lemma \ref{app:pixelspace}, and to look at SGD in the
feature learning regime.

\section{The Fourier data model}
\label{sec:fourier-model}

Our goal is to establish rigorously how phases are learnt by a simple neural
network, and to clarify the role of the amplitudes in shaping the learning
dynamics. To that end, we consider the standard setting of a binary
discrimination task with high-dimensional inputs $\mathcal{D} = (x^{\mu},
y^{\mu})_{\mu}^n \subseteq \mathbb{R}^N \times \mathbb{R}$ on which we train a
single neuron or student $\sigma$ with weights $w \in \mathbb{R}^N$ using the
correlation loss
\begin{equation}\label{correlationloss}
        L(w; x^{\mu}, y^{\mu}) = 1 - y^{\mu} \sigma(w \cdot x^{\mu}),
\end{equation}
see for example~\citet{benarous2021online, damian2023smoothing, dandi2024benefits, bardone2024sliding} for recent examples of similar setups.

\subsection{The standard path to structured data} 

A natural approach to characterise the difficulty of learning certain data
structures, like pair-wise correlations between elements of $x^{\mu}$, is to let
the student distinguish noise from structured inputs in the form of hypothesis
test. We can follow this approach and assume that inputs with the label
$y^{\mu}= -1$ are pure noise and sampled from a multivariate Gaussian
distribution~$\mathbb{P}_0$ with mean zero and identity covariance. Meanwhile,
inputs with label $y^{\mu} = +1$ are drawn from a distribution~$\mathbb{P}$ that
has a single preferred, or planted, direction. At the level of the inputs, this
planted direction is typically constructed by taking a Gaussian noise vector and
adding a ``spike'' $u$, e.g.
\begin{equation}
    \label{eq:wishart}
    x^\mu = \beta g^\mu u + z^\mu,
\end{equation}
where $z^\mu$ is a draw from a multivariate Gaussian distribution, while the spike $u \in \mathbb{R}^N$  is a fixed direction chosen from some prior, $\beta>0$ is a signal-to-noise ratio and $g^\mu$ is a scalar standard normal random variable. By training the student to discriminate these two input classes, we essentially perform noise contrastive estimation~\citep{gutmann2010noise} and the student will ideally recover the planted direction $u$.

\subsection{The Fourier data model}%

Here, we propose a new data model that will allow us to capture some of the salient structural properties of images and enables us to surgically manipulate amplitudes and phases. 

Our \textbf{baseline distribution} is again the multivariate Gaussian, except that we impose a non-trivial, translation-invariant covariance matrix. Specifically, we set $\mathbb{P}_0 = \mathcal{N}(0, \Sigma)$, where the real covariance matrix $\Sigma \in \mathbb{R}^{N\times N}$ is required to be stationary with periodic boundaries conditions i.e.\  $\Sigma_{ij} = c_{(i-j) \text{ mod }N}$, for some real coefficients $(c_0, \cdots, c_{N-1})$. Equivalently, $\Sigma$ is a circulant matrix (see Lemma \ref{def:circulant}).
In what follows, we denote the inputs sampled from $\mathbb{P}_0$ as \textit{translation-invariant}.

Inputs under the \textbf{planted distribution} $\mathbb{P}$ are then obtained by a suitable modification, in Fourier space, of inputs drawn from the base distribution $\mathbb{P}_0$. We could see $\mathbb{P}$ as a non-Gaussian counterpart of $\mathbb{P}_0$, since by construction they share the same low-order statistics -- mean and covariance matrix -- but the signal which distinguishes the two classes of inputs is encoded in the higher-order cumulants of $\mathbb{P}$. As a consequence, the inputs sampled from $\mathbb{P}$ are also translation-invariant. We stress that the two classes of inputs are indistinguishable based on their cumulants up to second-order.

We build $\mathbb{P}$ by first sampling $( z^{\mu})_{\mu=0}^{n} \subseteq \mathbb{R}^N$ from the baseline distribution~$\mathbb{P}_0$ and
considering their Discrete Fourier Transform (DFT) denoted by $Z^{\mu} = \text{DFT}(z^{\mu})$. The complex Fourier coefficients can always be decomposed into amplitudes $\rho^{\mu}_k$ and phases $\varphi^{\mu}_k$, i.e.\
\begin{equation*}
    Z^{\mu}_k = \rho^{\mu}_k \; e^{i \varphi^{\mu}_k}
\end{equation*}
for the frequencies $k = 0, \dots, N-1$. 
In particular, the phases are independent and uniformly distributed in $[-\pi, \pi)$, since $\mathbb{P}_0$ is Gaussian and translation-invariant (cf. Lemma \ref{app:Fourierdistribution}). 
We could now change the distribution of either the amplitudes or the phases. 
Amplitude manipulations change both the covariance matrix and higher-order correlations of the inputs, phase manipulations do not. 
Indeed, as long as the data remain translation-invariant, the phases control only the higher-order correlations (HOCs) of the inputs leaving the covariance matrix unchanged (cf. Lemma \ref{app:translation-invariance}), so we will focus on that. Specifically, we are going to alter the distribution of the phases as follows: for any $\varepsilon > 0$, we choose a frequency $k_0 \neq \{0, N/2\}$ and define the new inputs in Fourier space, $X^{\mu} \in \mathbb{C}^N$, as 
    \begin{equation}
    \label{eq:model}
    X^{\mu}\hspace{-0.3em} =\hspace{-0.3em}
    \begin{pmatrix}
    \rho^{\mu}_0\\
    \vdots\\
    \rho^{\mu}_{N-1}
    \end{pmatrix}
    \hspace{-0.3em} \exp \hspace{-0.3em}
    \begin{pmatrix}
    i\,\varphi^{\mu}_0\\
    \vdots\\
    i\,
    \psi^{\mu}_{k_0}\\
    \vdots\\
    i\,\varphi^{\mu}_{N-1}
    \end{pmatrix}
\qquad
\psi^{\mu}_{k_0}\hspace{-0.25em} = \hspace{-0.25em}\varphi^{\mu}_{k_0} \hspace{-0.3em} \hspace{-0.25em} + \hspace{-0.25em} \varepsilon\underbrace{ f({\varphi^{\mu}_{k_0}})}_{\text{HOCs}} + U^{\mu}, 
\end{equation}
where $f = f(\varphi_{k_0}^{\mu})$ is an arbitrary but fixed function of the phase $\varphi_{k_0}^{\mu}$ and $U^{\mu}$ is sampled from a discrete random variable, independent from $\varphi^{\mu}_k$, which takes values uniformly among the angles $\{0,\pi/2,\pi,3\pi/2\}$ associated to the fourth roots of unity. 
To ensure that the new inputs are still real in pixel space, we ask that $X^{\mu}_{N-k_0} = \overline{X}^{\mu}_{k_0}$.
We define the planted distribution $\mathbb{P}$ as the distribution, in pixel space, of the Inverse Discrete Fourier Transform (IDFT) of $X^{\mu}$, i.e.\ $x^{\mu} = \text{IDFT}(X^{\mu})$.

We now discuss the motivations behind the choices of $f({\varphi^{\mu}_{k_0})}$ and $U^{\mu}$.
First of all, the presence of the phase-dependent modification $f({\varphi^{\mu}_{k_0})}$ 
leads to a non-Gaussian distribution of inputs by introducing non trivial higher-order correlations in pixel space. No such effect occurs when $f$ is independent of the phases, meaning that a simple ``spike'' would not serve our purpose (cf. Lemma \ref{app:perturbation}). For concreteness, we will perform our theoretical analysis with $f(\varphi_{k_0}^{\mu}) = \sin{(\varphi_{k_0}^{\mu})}$.
On the other hand, the presence of $U^{\mu}$ enforces the covariance matrix of $X^{\mu}$ to stay diagonal, so that the inputs are still translation-invariant after the alteration of the phases. In absence of the corrector $U^{\mu}$, an alteration of the phases in general breaks translation-invariance of the inputs, as proved in the next lemma.
 
\begin{lemma}\label{app:nongaussianity}
    If $X^{\mu}$ is drawn according to the Fourier data model \eqref{eq:model} with phase transformation given by $\psi^{\mu}_{k_0} = \varphi^{\mu}_{k_0} + \varepsilon f({\varphi^{\mu}_{k_0}})$, then $x^{\mu} = \mathrm{IDFT}(X^{\mu})$ is not translation-invariant.
\end{lemma}
\begin{proof}
    For a generic function of the phases $f$, the entry $(k_0, N-k_0)$ of the covariance matrix of $X^{\mu}$ is given by
    \begin{equation*}
          \mathbb{E}[X_{k_0}^{\mu} \overline{X}_{N-k_0}^{\mu}] = \mathbb{E}[Z_{k_0}^{\mu \,2}] \; \mathbb{E}[e^{2i \varepsilon f({\varphi_{k_0}^{\mu})}}] \neq 0,
    \end{equation*}
    which contradicts $x^{\mu}$ being translation-invariant (see Lemma \ref{app:translation-invariance} for a reminder that the covariance matrix of translation-invariant inputs is diagonal in Fourier space).
\end{proof}
Therefore, one can actually prove (cf. Lemma \ref{app:second}) that the modified inputs $x^{\mu}$ are still translation-invariant and have the same covariance matrix of the original inputs drawn from $\mathbb{P}_0$. However, the new data points $x^{\mu}$ are non-Gaussian and have a ``signal'' in their fourth-order statistics (cf. Lemma \ref{app:fourth}).
We can note that, by construction of $U^{\mu}$, we get a non-vanishing factor $\mathbb{E}[e^{i\,4 U^{\mu}}] = 1$, which is the reason why the new inputs $x^{\mu}$ are allowed to exhibit non-trivial fourth-order cumulants.

We finally note that manipulating the phases means that, in pixel space, there is a two-dimensional subspace of $\mathbb{R}^N$ along which the projections of the inputs drawn from $\mathbb{P}$ have non-Gaussian statistics (cf. Lemma \ref{app:fourth}) and hence differ from the projections of inputs drawn from $\mathbb{P}_0$. This subspace is spanned by the pair of DFT basis vectors $(u^k)_k$ and sine DFT basis vectors $(v^k)_k$ corresponding to the modified mode $k_0$, where, for the frequencies \mbox{$k = 1, \dots, \lfloor \frac{N-1}{2} \rfloor$}, we have defined
\begin{equation}
        u^k_n = \sqrt{\frac{2}{N}} \cos{\biggl(\frac{2\pi kn}{N}\biggr)}, \hspace{0.8em} v^k_n = \sqrt{\frac{2}{N}} \sin{\biggl(\frac{2\pi kn}{N}\biggr)},
\end{equation}
for any $n = 0, \dots, N-1$
(cf. \cref{fig:projections} in the appendix for a visualisation).
The ``DFT basis'' $(u^k, v^k)_k$ 
diagonalises the covariance matrix of any translation-invariant input, see Lemma \ref{app:diagonalisation}. In what follows, we will drop the explicit dependence on the frequency $k_0$, write $u$ and $v$ - instead of $u^{k_0}$ and $v^{k_0}$ - and call them ``DFT phase vectors''.

\begin{remark}
    In pixel space, the inputs $(x^{\mu})_{\mu}^n \subseteq \mathbb{R}^N$ sampled from the Fourier data model \eqref{eq:model} cannot be seen as distributed according to a single-index model in the usual sense (e.g. \cref{eq:wishart}). Instead, for $k = 0, \dots, N-1$, we have that
    \begin{equation}
        x_k^{\mu} = z_k^{\mu} + \frac{2\rho^{\mu}_{k_0}}{N} \Biggl[ \cos{\biggl( \frac{2\pi k k_0}{N} + \varphi_{k_0}^{\mu} + \varepsilon f{(\varphi_{k_0}^{\mu}) + U^{\mu}} \biggr)} \Biggr],
    \end{equation}
    as proved in Lemma \ref{app:pixelspace}. Note that the latent variables $\varphi_{k_0}^{\mu}$ and $U^{\mu}$ always enter through a non-linear function (the cosine). This is due to the modulation of the phases of the inputs in Fourier space, which are linked non-linearly to the inputs in pixel space.
\end{remark}

\section{Sample complexities for classifying isotropic inputs}
\label{sec:isotropic}

In this section we rigorously quantify the sample complexities of online SGD on the correlation loss~\eqref{correlationloss}, in the case of isotropic inputs $x^{\mu}$. Recall that, by construction of the Fourier data model \eqref{eq:model}, to successfully distinguish the two classes of inputs, SGD has to learn the subspace of $\mathbb{R}^N$ spanned by the DFT phase vectors $u$ and $v$, along which the projections of the inputs drawn from $\mathbb{P}_0$ and $\mathbb{P}$ differ in distribution. 
Additionally, in this section we assume that the inputs of both the classes are isotropic, so with identity covariance matrix. We prove here that this choice of considering whitened inputs, together with the fact that the DFT phase vectors are encoded in higher-order cumulants of the data (cf. Lemma \ref{app:fourth}),
leads to a genuinely hard computational task in high dimensions.

Online SGD is implemented by sampling a new data point at each step from $\mathcal{D} = (x^{\mu}, y^{\mu})_{\mu=1}^n$.
For a suitable learning rate $\delta_N >0$, each iteration of the spherical online SGD, for $w_0 \sim \text{Unif}(\mathbb{S}^{N-1})$, is defined as
\begin{equation}\label{algo}
    \hspace{-1em}\begin{cases}
        \widetilde{w}_t = w_{t-1} + \delta_N \,
        \nabla_{\text{sph}} L(w; x^{t}, y^t) \bigl|_{w = w_{t-1}} 
        \; t \geq 1,\\
        w_t = \widetilde{w}_t/\|\widetilde{w}_t\|,
    \end{cases}
\end{equation}
where the spherical gradient is given by \mbox{\small{$\nabla_{\text{sph}} L(w,\cdot) = (\mathds{1} - w w^{\top})\nabla_w L(w, \cdot)$}}.
To tackle this classification problem, we ask: how fast does SGD recover a signal encoded in the phase of high-dimensional translation-invariant inputs? To answer, a key object to look at is the \textit{population correlation loss} defined simply as
\begin{equation}\label{def:population}
    L(w) = \mathbb{E}[L(w; x^{\mu}, y^{\mu})],
\end{equation}
where the expectation is taken over the data distribution. 
Related to that, a fundamental invariant is the \textit{information exponent}, which is the order $k^*$ of the first term having a non-zero coefficient in a suitable expansion of the population loss (cf. Definition \ref{def:information_exponent}). This quantity was introduced by \citet{benarous2021online} in the context of supervised learning, where they showed that it governs the sample complexity of online SGD in single-index models. They prove that for $k^* \geq 3$, no recovery of the signal is possible within $n \ll N^{k^*-1}$ samples, and a partial alignment occurs when $n \gg N^{k^*-1} \log^2{N}$. Conversely, the signal is recovered at linear sample complexity - up to logarithmic factors - when $k^* = 2$. Other useful invariants are the \emph{leap index} of \citet{abbe2023sgd, dandi2024two}, and the \textit{generative exponent} of \citet{damian2024computational}, which govern the sample complexity for more complex teacher-student models.

Using the approach of \citet{benarous2021online}, we prove in \cref{thm:hard} that SGD requires a cubic sample complexity to align with the subspace spanned by the DFT phase vectors, where projections of the inputs drawn from the two classes differ in distribution. This hardness in learning is due to the fact that the inputs are translation-invariant and then the phase modification is encoded only in fourth-order cumulants (cf. Lemma \ref{app:fourth}), leading to a high information exponent $k^* = 4$. Note that the same scalings are needed to solve the tensor PCA problem \citep{richard2014statistical} for a fourth-order tensor and the Independent Component Analysis (ICA) problem \citep{ricci2025feature, ricci2025reduce}.

\begin{figure*}[t!]
  \centering
\includegraphics[height=.2\linewidth, width=.64\linewidth]{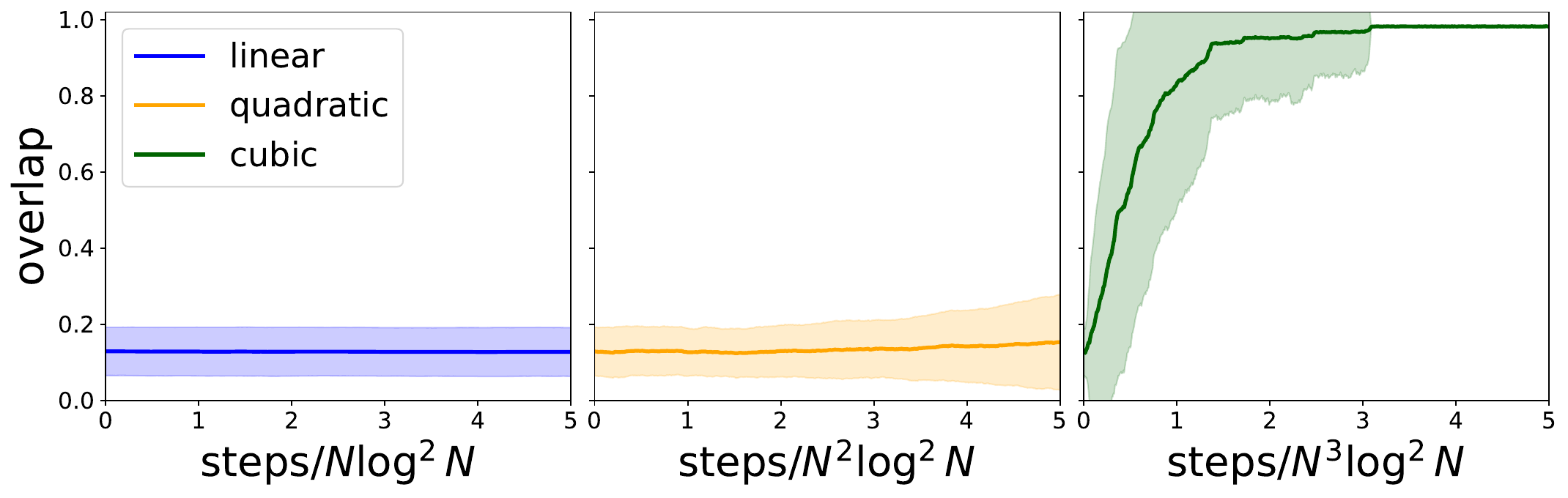}%
  \includegraphics[ width=.35\linewidth]{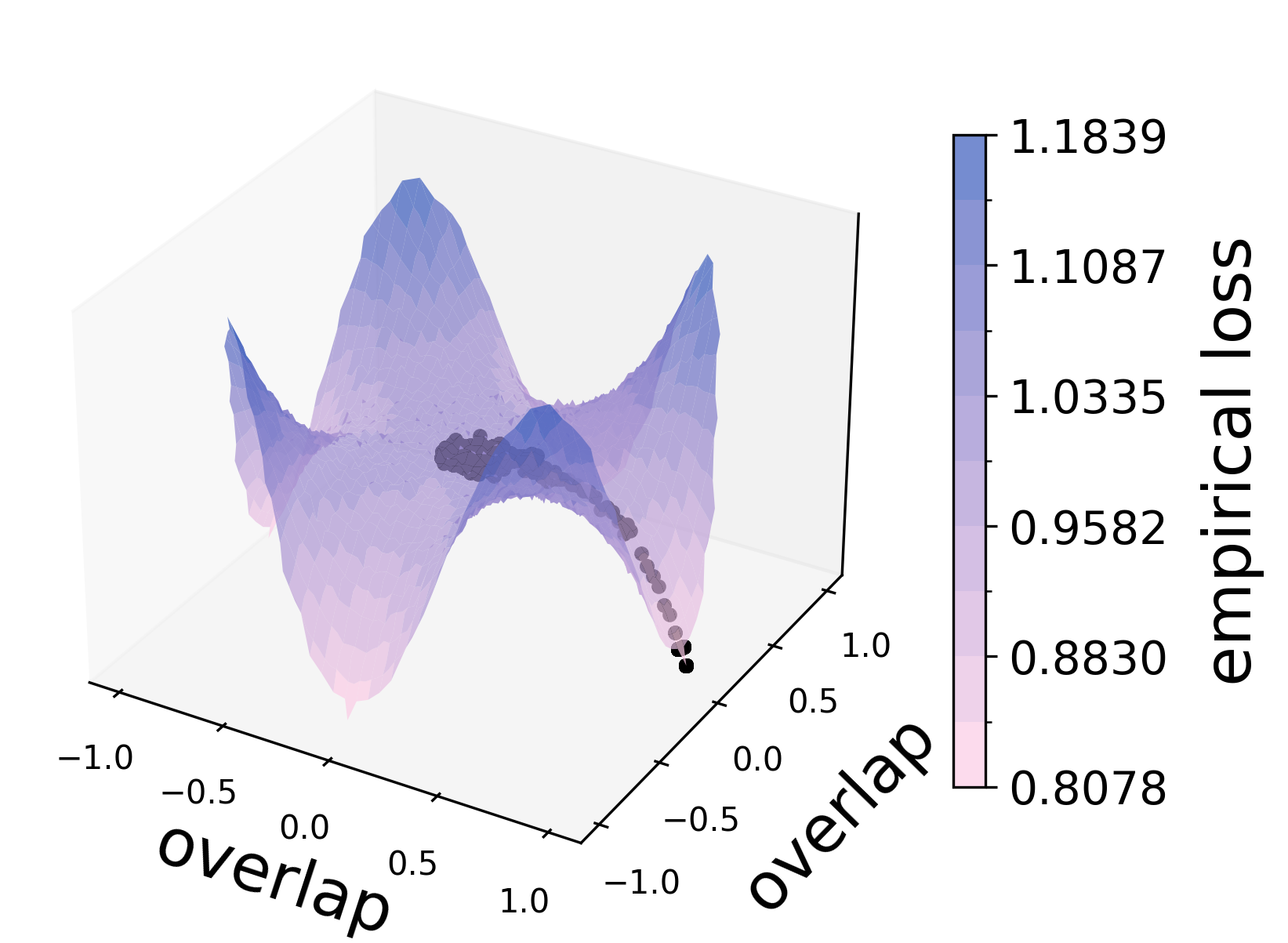}%
  \caption{\label{fig:fig2} 
  \textbf{Performance of SGD in classifying isotropic inputs on the Fourier data model.} \textbf{(Left)} We run online SGD applied to the correlation loss \eqref{correlationloss} with isotropic inputs drawn from the Fourier data model \eqref{eq:model}.
  SGD does not weakly recover the signal at linear (\legendline[1.5pt]{blue}) or quadratic (\legendline[1.5pt]{orange})
   sample complexity, whereas it converges to the subspace spanned by the DFT phase vectors in the cubic (\legendline[1.5pt]{darkgreen}) regime. On the $y$-axes, we see the norm of the projection of the perceptron weight in the said ``non-Gaussian'' subspace. \textbf{(Right)} Empirical loss landscape as a function of the overlaps between the weight vector and the DFT phase vectors, i.e. $w \cdot u$ and $w \cdot v$. Additionally, we display one run of SGD ($\bullet$) converging to a global minimum of the loss only after a long search phase around the origin, as predicted by \cref{thm:hard}.}
\end{figure*}

We assume that the activation function $\sigma: \mathbb{R} \to\mathbb{R}$ is continuously differentiable and even. The constraint on the parity is a technical assumption, which
does not change qualitatively the theoretical analysis, but simplifies parts of it. We prove \cref{thm:hard} in Appendix \ref{app:hardnessoflearning}.
\begin{theorem}\label{thm:hard}
    Sample $(x^{\mu})_{\mu =1}^n \subseteq \mathbb{R}^N$ from the Fourier data model \eqref{eq:model} with $\Sigma = \mathds{1}$. Apply online $\mathrm{SGD}$ \eqref{algo} to the correlation loss \eqref{correlationloss} and define the overlaps \mbox{$\alpha_{u,n} = w_n \cdot u$} and $\alpha_{v,n} = w_n \cdot v$, where $(u, v)$ are the $\rm{DFT}$ phase vectors. Then, if 
    \begin{footnotesize} $1/(N^2 \log^2{N}) \hspace{-0.3em}\ll \delta_N \hspace{-0.3em} \ll 1/(N \log{N})$
    \end{footnotesize} 
    and {\boldmath $n \gg N^3 \log^2 N$}, there is \textbf{weak recovery} for $(u, v)$ i.e., for some $\eta >0$,
    \begin{small}
    \begin{equation*}
        \lim_{N \to \infty} \hspace{-0.3em} P( \,\lvert \alpha_{u,n} \rvert \geq \eta) = 1 \quad \text{and} \quad
        \lim_{N \to \infty} \hspace{-0.3em} P( \,\lvert \alpha_{v,n} \rvert \geq \eta) = 1.
    \end{equation*}
    \end{small}
    Conversely, when 
    \begin{footnotesize}$
    \delta_N \ll 1/(N \log{N})$\end{footnotesize}
    and $\boldsymbol{n \ll N^3}$, there is \textbf{no recovery} for $(u,v)$ i.e., in probability,
    \begin{equation*}
        \lvert \alpha_{u,n} \rvert, \lvert \alpha_{v,n} \rvert \xrightarrow[N \to +\infty]{}0.
    \end{equation*}
\end{theorem}
\begin{proof}[Sketch of the proof]
    The correlation population loss \eqref{def:population} is essentially the sum of
    \begin{math}
        \mathbb{E}_{\mathbb{P}_0}[\sigma(w \cdot z)]
    \end{math}
    and 
    \begin{math}
        \mathbb{E}_{\mathbb{P}}[\sigma(w \cdot x)]
    \end{math}.
    Since $\|w\|=1$, the former term is constant and equals the zeroth-order Hermite coefficient $c_0^{\sigma}$ of the activation function (cf. Definition \ref{def:herm exp}).
    The latter term can be expanded by using the \textit{likelihood ratio}
    \begin{math}
        \ell(s) = \mathrm{d}\mathbb{P}/\mathrm{d}\mathbb{P}_0(s)
    \end{math}
    of the planted distribution $\mathbb{P}$ with respect to the baseline distribution $\mathbb{P}_0$ (see Definition \ref{def:likelihood}), which depends only on the projections of the inputs along the DFT phase vectors. Hence, we get that
    \begin{equation}
        \mathbb{E}_{\mathbb{P}}[\sigma(w \cdot x)] = \mathbb{E}_{\mathbb{P}_0}[\sigma(w \cdot x) \ell(v \cdot x, u \cdot u)].
    \end{equation}
    We can now expand the activation function $\sigma$ and the likelihood ratio $\ell$ in Hermite polynomials. Since $\mathbb{P}$ and $\mathbb{P}_0$ share the same first and second-order statistics,
    the Hermite coefficients $c_{ij}^{\ell}$ of the likelihood ratio are such that $c_{ij}^{\ell} = 0$, if $i + j \leq 2$, with the only exception of the null term $c_{00}^{\ell}=1$ which, multiplied by $c_0^{\sigma}$, cancels with 
    \begin{math}
        \mathbb{E}_{\mathbb{P}_0}[\sigma(w \cdot z)]
    \end{math}. 
    Since any odd Hermite coefficient of the activation function vanishes, we get
    \begin{equation}
        L(w) = 1 - c_1 \, (\alpha_v^4 + \alpha_u^4) + c_2\, \alpha_{v}^2 \alpha_u^2 + \rm{h.o.t.},
    \end{equation}
    for suitable coefficients $c_1, c_2 \in\mathbb{R}$.
    Hence, the information exponent \citep{benarous2021online} is $k^* = 4$.
\end{proof}

\section{The power-law decay of Fourier amplitudes accelerates learning phase information}
\label{sec:experiments-power-law}

We have established that learning phase information is a computationally hard task in high dimensions, requiring at least cubic sample complexity and long run-times of SGD. However, in neural networks trained on real data, no such huge complexity gap appears to exist: in our experiments with simple neural networks shown in \cref{fig:fig1} g) we found no large gap between the point where the network's performance initially improves and where it starts to exploit phase information. So how come neural networks access phase information quickly on real images, when it is hard to learn from phase information alone?

A key difference between the setting of \cref{thm:hard} and real images is the distribution of Fourier amplitudes of the data. For \cref{thm:hard}, we assumed that the covariance matrix is the identity, which corresponds to flat amplitudes in Fourier space. However, images typically have a covariance whose power-spectrum is not flat, but instead exhibits a power-law decay, with Fourier amplitudes decaying roughly as $1/f$, where $f$ is the frequency of the mode. We illustrate this effect for the ``cotton'' class of the texture dataset we have already treated in the introduction in \cref{fig:fig3} a), where we see that 
the diagonal of the covariance matrix in Fourier space follows an approximate power-law behavior, increasing in magnitude with the size of the image patches. For translation-invariant inputs, those diagonal entries are the second moments of the Fourier amplitudes or, equivalently, the eigenvectors of the covariance matrix of the inputs in pixel space (cf. Remark \ref{app:scalings}).


\begin{figure*}
    \includegraphics[width=\linewidth]{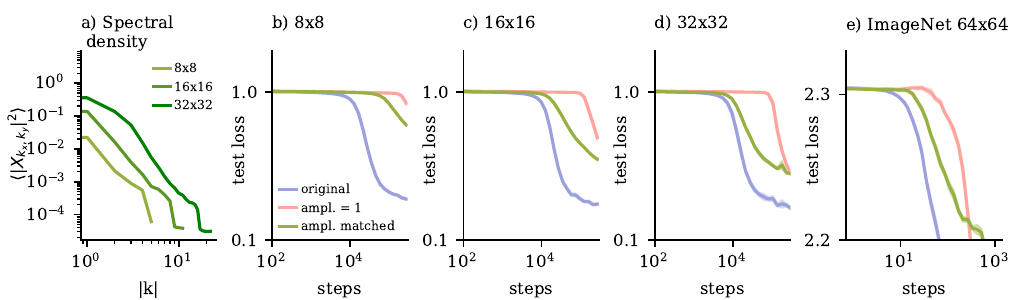}
    \caption{\label{fig:fig3} \textbf{Shared principal subspace speeds up learning.} \textbf{a)} Average squared Fourier amplitudes of image patches of ``cotton'' class from ALOT texture dataset, averaged over wave vectors of equal length $|k|=\sqrt{k_x^2+k_y^2}$, for patches of increasing size. \textbf{b-d)} Test losses of classifiers trained on distinguishing ``cotton'' vs. ``lace'' on original data (\legendline[1.5pt]{violet}), data where all Fourier amplitudes of both classes have been set to one (\legendline[1.5pt]{reddish}) and data where the amplitudes of the second class are replaced by the amplitudes of the first class (\legendline[1.5pt]{greenish}). Test loss curves are averages over $8$ random initializations of the classifiers.
    \textbf{e)} The same experiment as b-d), but using a Resnet18 on a super-classed version of ImageNet with ten classes. Cf. \cref{fig:cifar100} in the appendix to visualise the same behaviour on classes from CIFAR-100.}
\end{figure*}

To investigate the effect of the power-law amplitudes on the learning dynamics, we trained two-layer fully connected neural networks on the same texture data (cf. \cref{fig:fig3} b-d)), but modify the \textit{amplitudes} of the images in Fourier space. First, we show the test loss of classifiers trained on the original image data without any modification (blue curves (\legendline[1.5pt]{violet}) in \cref{fig:fig3} b-d)). We then train classifiers on the same images after setting the amplitudes of all images to one (pink curves (\legendline[1.5pt]{reddish}) in \cref{fig:fig3} b-d)). This results in a considerable slow-down in the onset of learning, as the networks can now only rely on phase information on the image data, illustrating the idea of \cref{thm:hard} on a realistic dataset. Finally, we construct a third dataset, where we keep one class of the images unchanged, but replace the amplitudes of the second class with the amplitudes of the first class. On the level of the amplitude statistics, the two classes are hence exactly identical: the only difference between them are their phases. However, in contrast to before, their Fourier spectra are no longer flat, but follow the curves in \cref{fig:fig3} a). We show the corresponding test loss curves in green (\legendline[1.5pt]{greenish}) in \cref{fig:fig3} b-d). Even though the network can only distinguish the two classes based on the phase information, we observe a significant speed-up with respect to the case of flat Fourier spectra. Hence, the mere amplification of specific directions belonging to the principal subspace in Fourier space leads to a considerable speed-up in how fast the neural network can detect and exploit phase information.  

This effect is not just limited to fully-connected networks: in \cref{fig:fig3} e), we show the result of the same experiment in a Resnet18 trained on a super-classed version of ImageNet with ten classes (cf. \cref{fig:cifar100} too see the same experiment on classes from CIFAR-100).
Even though convolutions are designed to be responsive to edges and other image features that are encoded in the phases, we see again a slow-down of the onset of learning when amplitudes are flattened, while setting the amplitudes of all the training images to the average amplitudes speeds up learning. Note that we normalised images in all the experiments to ensure they have the same norm; full details on the training of these classifiers are given in \cref{app:details_training} in the appendix. In conclusion, the speed-up emerges consistently across various datasets (textures, ImageNet and CIFAR-100), architectures (Resnet18 and two-layer classifiers), optimisers (mini-Batch SGD and Adam) and losses (mean-squared and cross-entropy).

In the following section, we extend our theoretical model to explain how a principal subspace whose magnitude scales with the dimension can lead to a speed-up in learning the phase, even if it is shared among classes and does therefore not provide additional information for discriminating the classes.


\section{Theoretical analysis for non-isotropic inputs}
\label{sec:theory-power-law}

In this section, we extend our theory to account for the effect of power-law
spectra that we found experimentally. We do this by considering inputs sampled from the Fourier data model \eqref{eq:model} with a translation-invariant, non-isotropic covariance. We first expand the population loss, and analyse how the spectrum interacts with the phase to yield an effective signal-to-noise ratio (\cref{sec:loss-non-isotropic}). We then discuss how the power-law distribution of amplitudes can speed up learning of phase information (\cref{sec:dynamics-non-isotropic}).

\subsection{Expansion of the population correlation loss for non-isotropic inputs}
\label{sec:loss-non-isotropic}

Consider inputs drawn from the Fourier data model \eqref{eq:model}, with a circulant covariance matrix.
The key difficulty in analysing the population correlation loss \eqref{def:population} in this setting is due to the interaction between the non-Gaussianity and the anisotropy of the inputs. 
Indeed, as in the isotropic case, due to non-Gaussianity of $\mathbb{P}_0$, the population loss can be written as the expectation over the baseline distribution $\mathbb{P}_0$ of the product between the activation $\sigma$ and the likelihood ratio $\ell$. Due to anisotropy, after expanding both of them in Hermite polynomials, one first needs to rescale \cref{eq:rescaling} and resum \cref{eq:sum} our formulas before applying the orthogonality properties (Lemma \ref{app:ort}); otherwise, one cannot write the expectation (so the loss) in terms of the relevant overlaps. Note also in \cref{eq:rescaled_like} that we do not compute the Hermite coefficients for the likelihood ratio $\ell$, but for a rescaled version of it. 

While previous work has considered inputs with a rank-one perturbation of the identity matrix for the covariance~\citep{mousavi2023gradient, bardone2024sliding}, a full-rank anisotropic covariance matrix has not yet been treated in this framework, to the best of our knowledge. From a technical point of view, note that we do not consider the likelihood ratio of $\mathbb{P}$ with respect to the standard normal distribution \citep{bardone2024sliding}, since we allow for an extensive number $O(N)$ of eigenvectors of the covariance of the inputs, which cannot hence be treated in the said likelihood ratio.


\begin{proposition}\label{prop:expansionloss}
    Consider $w \in \mathbb{R}^N$, $\alpha_u = w \cdot u$ and $\alpha_v = w \cdot v$, where $(u,v)$ are the $\rm{DFT}$ phase vectors. Let $\lambda_{k_0}$ be the $k_0$-th eigenvalue of a circulant matrix $\Sigma$ and sample the inputs $x^{\mu}$ from the Fourier data model \eqref{eq:model} with covariance $\Sigma$. 
    Then, the population correlation loss \eqref{def:population} is such that 
    \begin{align}\label{eq:lossexp}
        L(w)\hspace{-0.1em} & = 1 -\hspace{-0.2em}\lambda_{k_0}^2
        \biggl(c_{04} \alpha_u^4 + c_{22} \alpha_u^2\alpha_v^2 + c_{04}\alpha_v^4 \biggr) \biggl[c_4^{\sigma}
        + c_6^{\sigma} \frac{(\sigma_{\Sigma}^2 -1)}{2} 
        \biggr]\hspace{-0.4em} + \rm{h.o.t.},
    \end{align}
    with $\sigma_{\Sigma} = \sqrt{w^{\top}\Sigma w}$. The coefficients $c_{04}, c_{22}$ and $ c_{04}$ depend on the likelihood ratio between $\mathbb{P}$ and $\mathbb{P}_0$, while $c_{4}^{\sigma}$ and $c_{6}^{\sigma}$ are the $4$-th and $6$-th order Hermite coefficient of the activation function $\sigma$.
\end{proposition}
We prove Proposition \ref{prop:expansionloss} in \cref{app:expansionloss} in the appendix. By construction, $\lambda_{k_0}$ is an eigenvalue of the DFT phase eigenvectors $u$ and $v$ of the covariance matrix. 
Note that $\lambda_{k_0}$, which is task-relevant since it corresponds to the frequency of the modified phase,
plays the role of an effective signal-to-noise ratio induced by the structure of the inputs. The other eigenvalues $(\lambda_{m})_{m}^{N-1}$ corresponding to the DFT eigenvectors $(u^m, v^m)_{m}^{N-1}$ appear in the quadratic form induced by the covariance through the overlaps $\alpha_{u_m} = w \cdot u^m$ and $\alpha_{v_m} = w \cdot v^m$, i.e.\
\begin{equation}
\sigma_{\Sigma}^2- 1 = (\lambda_{k_0} -1)(\alpha_u^2 + \alpha_v^2) + \sum_{m = 1}^{N-1} (\lambda_m \hspace{-0.2em}-\hspace{-0.2em}1)(\alpha_{u_m}^2\hspace{-0.2em} + \hspace{-0.1em}\alpha_{v_m}^2).
\end{equation}
\subsubsection{Scalings of the loss for a low-rank eigenstructure}\label{sec:heuristics}
Even if we prove Proposition \ref{prop:expansionloss} for any scaling of the eigenvalues $\lambda_m$, 
we are particularly interested in the case where some top eigenvalues are much larger than the bulk eigenvalues, as occurs for power-law-decaying spectra. i.e.\ when $\lambda_m \propto m^{-\alpha}$, for some $\alpha > 0$.
Indeed, it is known that, in covariances with power-law spectra with exponent $\alpha > 0$, different measures of dimensionality like the \textit{effective dimension} or the \textit{Inverse Participation Ratio} (IPR) of the eigenvalues concentrate on an $O(1)$ quantity for power-law decays with exponent $\alpha > 1$ \citep{wortsman2025kernel, bartlett2020benign, cheng2024dimension}. Note that this is the typical regime of natural images \citep{vanderschaaf1996, Field1987spectra, hyvarinen2009natural}.
Motivated by this observation, we consider an effective low-rank eigenstructure that simplifies the analysis and maintains the correct scaling of the effective dimension. Precisely, we model a finite set $O(1)$ of leading eigenvalues as extensive, i.e. $O(N)$, while the tail of the spectrum stays $O(1)$.    
We contrast this realistic setting with the one in which also the leading eigenvalues remain $O(1)$, as in the case of flattened amplitudes, by gaining a key insight into the impact of the power-law decay of amplitudes on the learning dynamics.

Due to the uniform initialisation of SGD on the sphere in high dimensions, the following scalings for the overlaps at initialisation are appropriate: $\alpha_u, \alpha_v \approx 1/\sqrt{N}$, as well as 
$\alpha_{u_m}, \alpha_{v_m} \approx 1/\sqrt{N}$. 
Hence, from \cref{eq:lossexp}, the population loss scales as
\begin{equation}\label{eq:scalingloss}
    L(w) \approx \frac{\lambda_{k_0}^2}{N^2}  +  \frac{\lambda_{k_0}^2}{N^2} (\sigma_{\Sigma}^2 -1).
\end{equation}

\paragraph{Near-isotropic inputs:} If $\lambda_{k_0}, \lambda_m = O(1)$, we get that $L(w) \approx 1/N^2$, which is the scaling of $\alpha_u^4$ and $\alpha_v^4$ at initialisation. It corresponds to information exponent $k^* = 4$ -- the same found in \cref{thm:hard} for the isotropic inputs. We therefore expect online SGD to require a cubic sample complexity to distinguish structured inputs from noise (cf. Remark \ref{expect}).

\begin{paragraph}{Power-law decaying inputs:} If $\lambda_{k_0} \approx \sqrt{N}, 1 \lesssim \lambda_m \lesssim N$, we get that $L(w) \approx 1/N$. Here, $\lambda_{k_0}$ is one of the top eigenvalues, but possibly not the leading one. We observe that the effect of the extensive eigenvalue $\lambda_{k_0}$ is reducing the information exponent from $k^* = 4$ to $k^* = 2$. In the setting of \citet{benarous2021online}, this means recovering the signal at quasi-linear sample complexity.
Note that the presence of a low-dimensional eigenstructure does not reduce the dynamics to a finite exploration, as it is often the result of commonly used in practice preprocessing procedure (like PCA). Instead, the effect of the power-law decay in the spectrum is to induce an extensive signal-to-noise ratio $\lambda_{k_0}^2 \approx N$ in the loss which ``couples'' the signal and the quadratic form $\sigma_{\Sigma}^2$.
\end{paragraph}

\subsection{Dynamics for non-isotropic inputs with a low-rank eigenstructure}
\label{sec:dynamics-non-isotropic}

We can now analyse the implications of the power-law decay of the amplitudes for the dynamics of learning, building on previous works by \citet{benarous2022high, benarous2025spectral}. 
Motivated by the discussion in \cref{sec:heuristics} on the effective dimension of the eigenvalues for power-law decaying spectra, we consider here a low-rank eigenstructure, as in \cref{sec:heuristics}.

Let $M+1 \in \mathbb{N}$ be the (finite) number  of non-trivial, i.e. non-unit, top eigenvalues $( ( \lambda_{m})^M_{m=1}, \lambda_{k_0})$, whose associated eigenvectors $((u^m, v^m)_{m=1}^M, u, v)$ span the principal subspace of the inputs. For simplicity, we set the remaining eigenvalues to one. 
We track the evolution during training of the summary statistics
 \begin{equation}\label{summarystatistics}
    \boldsymbol{\alpha} = (\alpha_{u}, \alpha_v, (\alpha_{u_m}, \alpha_{v_m})_{m=1}^M , \omega_{\perp}),
 \end{equation} 
where $\alpha_u = w \cdot u, \alpha_v = w \cdot v, \alpha_{u_m} = w \cdot u^m$ and $\alpha_{v_m} = w \cdot v^m$.
The latter statistic is $\omega_{\perp}= w \cdot w_{\perp}$, where $w_{\perp}$ is the projection of the weight vector onto the space orthogonal to the subspace spanned by the principal components. 

Within this framework, we analyse the dynamics of non-normalised online SGD \ref{eq:sgdnon} initialised at a measure $\mu_0$
(cf.~\cref{sec:non-iso}), and we add to the population correlation loss a quartic penalisation term for the norm of the weight vector, i.e.
\begin{math}
    L'(w) = L(w) + \beta \|w\|^4, \beta \geq 0
\end{math}.
We choose non-normalised online SGD in order to match the conventions of \cref{thm:effective}, originally from \citet{benarous2022high}. Even if we expect that the normalisation does not change the predicted dynamics, proving a normalised counterpart of \cref{thm:effective} is beyond the scope of this work. On the same line, we add the penalisation term to mimic the effect of the previously considered normalisation without dominating the dynamics; the existence of the penalty does not have an impact on the predictions for the time scales over which the phases are learnt (cf. the computations in  Lemmas \ref{app:populationdrift} and \ref{app:rescaledpopulationdrif}).

At the level of the population dynamics, our results Conjectures \ref{conjecture:near-isotropic} and \ref{conjecture:power-law} are based on the computation of the \textit{population drift} $A_{\boldsymbol{\alpha}}$, which is the operator such that $A_{\boldsymbol{\alpha}}(\boldsymbol{\alpha}) = \lim_{N \to +\infty} \mathcal{A}_N \boldsymbol{\alpha}$, with 
\begin{equation}
    \mathcal{A}_N = \sum_{i}\partial_i L' \, \partial_i.
\end{equation}
The population drifts for Conjectures \ref{conjecture:near-isotropic} and \ref{conjecture:power-law} are rigorously computed in Lemmas \ref{app:populationdrift} and \ref{app:rescaledpopulationdrif}, respectively. Note that, in the latter case, one computes the drift for a suitable rescaling of the original summary statistics \eqref{summarystatistics}, which allows to avoid a diverging drift $A_{\boldsymbol{\alpha}} = +\infty$.

Beyond the population dynamics, we do a step further towards sample complexities: we look at the \textit{effective dynamics} for infinitely many linear steps of online SGD (cf. Theorem \ref{thm:effective} in Section \ref{app:literaturereview} in the appendix). Among the three ingredients needed by this theorem, namely the population drift, the \textit{population corrector} and the \textit{effective volatility}, we do not carry on the computations for the effective volatility. Indeed, to the best of our knowledge, it has been only computed in simpler data models, like Gaussian additive models e.g. single-index models or Gaussian mixtures \citep{benarous2022high, benarous2025spectral}, while we have built a genuinely non-Gaussian data model. This is why our results are stated as conjectures and not theorems.

We describe the dynamics of learning phase information in the two regimes of ``small'' (Conjecture \ref{conjecture:near-isotropic}), i.e.~$O(1)$, and ``large'' (Conjecture \ref{conjecture:power-law}), i.e. $O(N)$, top eigenvalues. We first conjecture that if the leading eigenvalues are $O(1)$, online SGD fails to recover phase information at linear sample complexity.
\begin{figure*}[t!]
  \centering
  \raisebox{0.2cm}{\includegraphics[width=.25\linewidth]{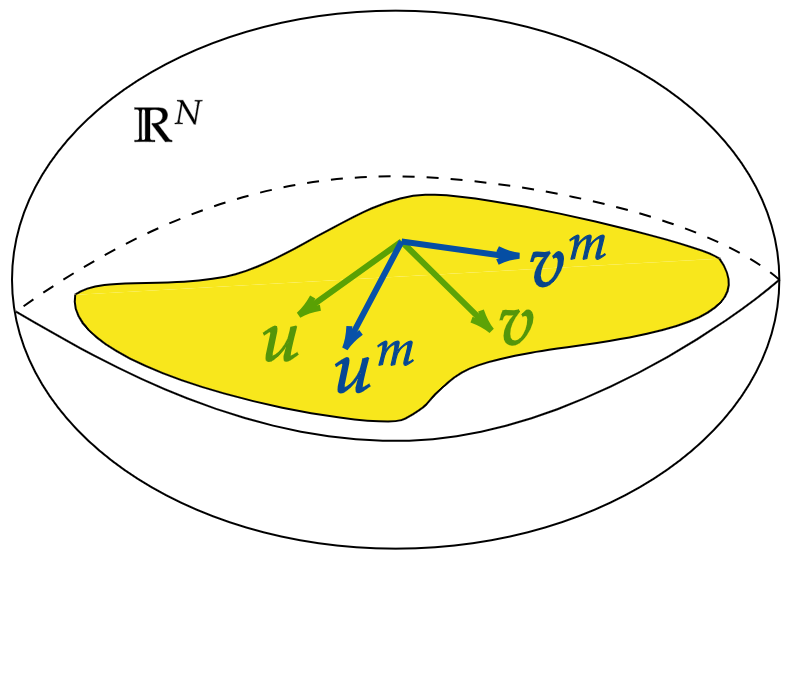}}%
  \includegraphics[ width=.75\linewidth]{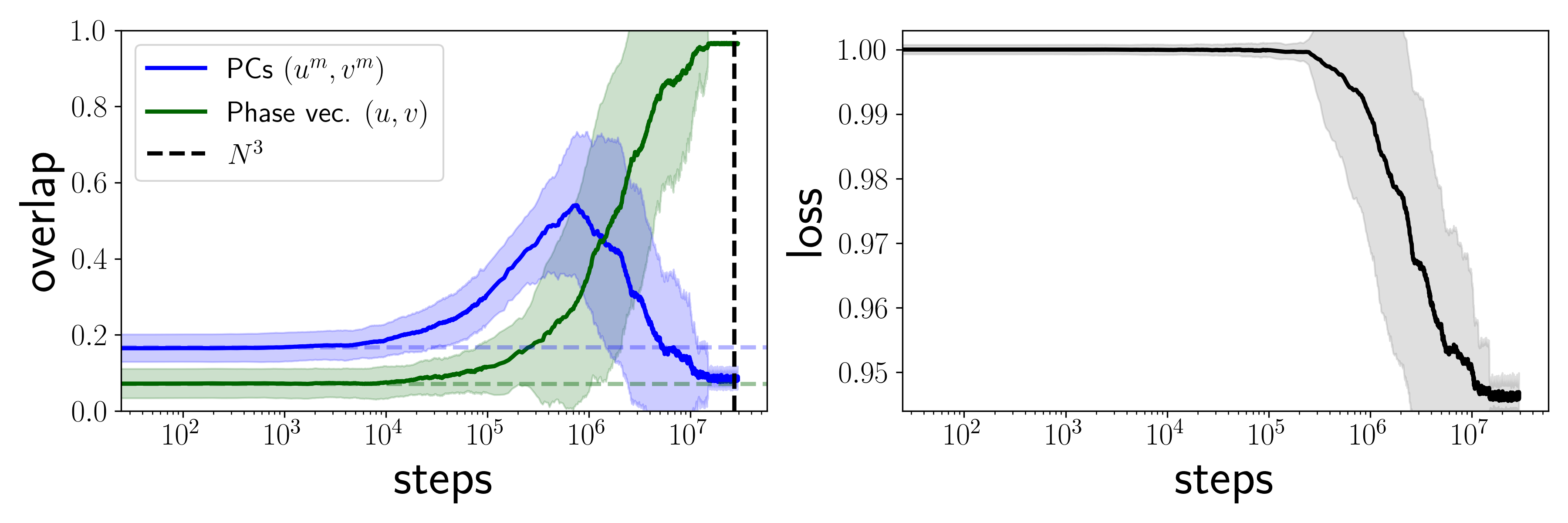}%
  \captionsetup{format=plain}
  \caption{\label{fig:fig4} 
  \textbf{Performance of SGD in classifying non-isotropic inputs on the Fourier data model.} 
  (\textbf{Left}) Cartoon of the principal subspace of power-law-decaying inputs sampled from the Fourier data model \eqref{eq:model}, spanned by the DFT phase vectors $(u, v)$ and a finite number of other principal components $(u^m, v^m)$. (\textbf{Middle}) At first, online SGD quickly recovers the whole principal subspace, including the DFT phase vectors. Then, the other principal components PCs (\legendline[1.5pt]{blue}), which are not useful for classification, are forgot and SGD converges within $n \lesssim N^3$ steps to the subspace spanned by the DFT phase vectors (\legendline[1.5pt]{darkgreen}). Compare with \cref{fig:fig2}, where no recovery happens within cubic sample complexity. Here, $N = 300$. (\textbf{Right}) Empirical loss function during training (\legendline[1.5pt]{black}).} 
\end{figure*}

\begin{conjecture}[Near-isotropic inputs]\label{conjecture:near-isotropic}
    Consider $\lambda_{k_0}, \lambda_m = O(1)$. Then, for $\delta_N = O(1/N)$ and \mbox{$\mu_0 = \mathcal{N}(0, \mathds{1}/N)$}, online $\rm{SGD}$ applied to the penalised loss $L'$ fails to weakly recover both the $\rm{DFT}$ phase vectors and the other principal components at linear sample complexity.
\end{conjecture}
\begin{proof}[Derivation]
    By Theorem \ref{thm:effective}, the effective dynamics of the summary statistics, at linear sample complexity, is given by the population drift $A_{\boldsymbol{\alpha}}(\boldsymbol{\alpha})$ (computed in Lemma \ref{app:populationdrift}), the effective volatility (which is expected to be zero like in the cases of the Gaussian mixture models and the tensor PCA model \citep{benarous2025spectral, benarous2022high}) and a vanishing population corrector (the summary statistics are linear in the weights and then $\mathcal{L}_N \boldsymbol{\alpha} = 0$, cf. \cref{operators}). Also, 
    the dynamics has to be initialised deterministically at
    \mbox{\begin{math}
        \nu = \lim_{N \to +\infty} \, \boldsymbol{\alpha}_\# \,\mu_0 = \delta_{0} \otimes \dots \otimes\delta_0 \otimes \delta_1 
    \end{math}}.
    Then, $\boldsymbol{\alpha}(t) = \boldsymbol{\alpha}_N(t)$ converges for $N \to +\infty$ to the solution of  $\dot{\boldsymbol{\alpha}}(t) = -A_{\boldsymbol{\alpha}}(\boldsymbol{\alpha}_t)$, initialised at $\boldsymbol{\alpha_0} = 0$, meaning that SGD does not escape the initial conditions (at linear sample complexity).
\end{proof}

Note that the initialisation of the high-dimensional effective dynamics $\nu$ is non-vanishing anymore if we zoom into a microscopic neighbourhood of the summary statistics of the form 
\begin{equation}
    \label{app:rescaledstatistics}
    \boldsymbol{m} = (\sqrt{N}\alpha_u, \sqrt{N}\alpha_v,\; (\sqrt{N}\alpha_{u_m}, \sqrt{N}\alpha_{v_m})_{m=1}^M,\; \omega_{\perp}).
\end{equation}
Indeed, for example, 
$\lim_{N \to +\infty} {m_{v}}_{\#} \mu_0 = \mathcal{N}(0,1).$
However, when $\lambda_{k_0}, \lambda_m = O(1)$, the population drift $A_{\boldsymbol{m}}$ for the rescaled statistics turns out to be identically zero (cf. Remark \ref{app:vanishing}). Conversely, if the top eigenvalue $\lambda_{k_0} \approx \sqrt{N}$ is extensive, corresponding to a signal-to-noise $\lambda_{k_0}^2 \approx N$ in the loss (cf. \cref{eq:scalingloss}), the population drift for the rescaled statistics is non-vanishing anymore.
We therefore turn to the case of extensive eigenvalue $\lambda_{k_0}$. 

Note that, if the mode that carries the phase modification is the one with the largest eigenvalue, it could be identified with simple spectral methods like PCA. Instead, we consider a more interesting and realistic case: the mode $k_0$ is part of the leading subspace without having the largest eigenvalue.
We conjecture that if the leading eigenvalues are $O(N)$, online SGD recovers phase information at quasi-linear sample complexity.
\begin{conjecture}[Power-law decaying inputs]
    \label{conjecture:power-law}
    Consider $\lambda_{k_0} \approx \sqrt{N}$ and $\sqrt{N} \lesssim \lambda_m \lesssim N$. Then, for $\delta_N = \Theta(1/N)$ and $\mu_0 =  \mathcal{N}(0, \mathds{1}/N)$, online $\rm{SGD}$ applied to the penalised loss $L'$ weakly recovers all the principal components, including the $\rm{DFT}$ phase vectors, at quasi-linear sample complexity.
\end{conjecture}
\begin{proof}[Derivation]
    For extensive $O(N)$ leading eigenvalues, the population drift of the rescaled dynamics is not identically zero (cf. Lemma \ref{app:rescaledpopulationdrif}). Indeed, for suitable constants $c_1, \dots, c_4 \in \mathbb{R}$, the ODEs for the evolution of the rescaled statistics $m_u$ and $m_{v_m}$ reads as
\begin{equation}\label{eq:OU}
    \begin{aligned}
        & \dot{m}_u = c_1 m_u^3 +  c_2 m_v^2 m_u 
        + \sum_{m=1}^M c_3 (m_{u_m}^2 + m_{v_m}^2)
        + 4\beta \|w\|^4 m_u,\\
        & \dot{m}_{u_m} = c_4 \sum_{i=0,2,4}
        \frac{m_v^i m_u^{4-i}}{i!(4-i)!}
        c_{i,4-i}^{\ell}\, m_{u_m}
        + 4\beta \|w\|^4 m_{u_m},
\end{aligned}
\end{equation}
and similar for $m_v$ and $m_{v_m}$.
Moreover, due to the rescaling of the summary statistics, the limiting measure at initialisation $\nu$ is no longer deterministically zero. 
The fact that the summary statistics exhibit non-vanishing dynamics under $\sqrt{N}$ rescaling suggests that, in the original scale, online SGD escapes the search phase in quasi-linear time.
\end{proof}

We leave it to future work to make our conjectures fully rigorous by also computing the diffusive terms and studying the resulting Ornstein-Uhlenbeck process \cref{eq:OU}, as discussed in Section \ref{app:literaturereview}.
The key insight of Conjecture \ref{conjecture:power-law} is that the dynamics of SGD will recover the leading subspace, which reflects the concentration of variance in the principal subspace of real images and it is shared among the two classes of inputs, in quasi-linear sample complexity. Once online SGD has recovered this subspace, it is easier for SGD to converge to the mode that carries the phase modification.

We illustrate this effect with a small experiment on data sampled from the Fourier data model where we train a perceptron on a discrimination task where the covariance of both inputs has $6$ modes with large eigenvalues (with the same scalings as in Conjecture \ref{conjecture:power-law}), of which one carries the phase modification. In the middle of \cref{fig:fig4}, we show how the overlap of the perceptron weight with the principal subspace, excluding the mode $k_0$, increases (blue line (\legendline[1.5pt]{blue})) before the perceptron converges to the mode with the non-trivial phase (green line (\legendline[1.5pt]{darkgreen})), and consequently forgets the principal components which are not task-relevant. 
Note that online SGD weakly recovers the whole principal subspace ((\legendline[1.5pt]{blue}) and (\legendline[1.5pt]{darkgreen})) at quasi-linear sample complexity and then converges to the signal within $n \lesssim N^3$ samples which, in the case of isotropic inputs, are not even sufficient for weak recovery (cf. \cref{fig:fig2}). 
Also, note that the loss only starts decaying after $10^6$ steps, rather than already at $10^5$ steps when the whole subspace is already weakly recovered. 

Importantly, in prior work by \citet{benarous2024stochastic}, a similar rich behavior emerges only under fine tuned signal-to-noise separation conditions, whereas here it is induced by the scalings of the amplitudes in the Fourier domain, leading to an interesting dynamics of online SGD even in the case of a single perceptron weight. Moreover, in contrast to prior work by \citet{mousavi2023gradient}, where online SGD fails in presence of anisotropic covariances, we show with the present work that a non-trivial eigenstructure can speed up recovery of the relevant directions, instead of being a limitation.

With this analysis, we show that classification between non-isotropic inputs drawn from the Fourier data model \eqref{eq:model} with non-trivial phases is performed efficiently by SGD in presence of power-law decaying inputs, which incidentally happens to be the case of real images. Indeed, this structure leads to an extensive signal-to-noise ratio $\lambda_{k_0}^2 \approx N$ for the loss.
We conclude then that the signal-to-noise ratio of a -- by itself, hard to learn -- phase recovery problem combined with a power-law on the amplitudes conspire to yield scaling behaviors similar to the ones of a tensor PCA problem~\citep{richard2014statistical, benarous2018pca} with an effective signal-to-noise ratio that makes it learnable at (quasi-)linear sample complexity.

\section{Conclusions and future perspectives}

Starting from the observation that, for translation-invariant inputs, Fourier
amplitudes determine second-order statistics while phases encode higher-order
correlations, we introduced a synthetic Fourier data model that makes this
separation explicit and controllable. Our main interest was in the sample
complexity of learning from the phases, since phases are perceptually more
important (see \cref{fig:fig1}) and it is the non-Gaussian statistics that shape
neural representations~\citep{olshausen1996emergence, hyvarinen2009natural,
ingrosso2022data,refinetti2023neural}. Analysing the dynamics of SGD on a
classification task in which the relevant signal is carried only by the phases,
we established sample complexities for extracting information from the phases in
several scenarios. While learning from phases is a hard task when the covariance
is (nearly) isotropic, we found both experimentally and theoretically that
power-law spectra like the ones found in natural images can dramatically speed
up learning from the phases. This speed-up is reminiscent of staircase effects
in learning higher-order polynomials of the target
functions~\citep{abbe2021staircase, abbe2023sgd, dandi2024two,
berthier2025learning} or higher-order statistics of the data
distribution~\citep{ingrosso2022data, bardone2024sliding, mendes2026solvable}.
Here, a crucial difference is that we see this speed-up even if the
spectra themselves are not discriminative between the classes. We confirmed our
theoretical results by simulations in more realistic settings with two-layer and
deep convolutional networks trained on textures, ImageNet and CIFAR100.

The present work could be extended by deriving information theoretic thresholds
for learning phase information, and by establishing algorithmic thresholds for
the Fourier data model for more general classes of learning algorithms, like
low-degree polynomials~\cite{hopkins2018statistical}. Going further, it would be
intriguing to perform a similar analysis in a different basis for images like
wavelets~\citep{mallat1999wavelet}. Unlike Fourier modes, wavelets are localized
both in space and frequency, and therefore provide a more refined description of
edges, corners, and multiscale image features. Likewise, an extension to other
texture models with controllable statistics \citep{victor2012local, PortillaS00,
depaolis2026perceptual} is a promising direction for future research. Extending
our approach to these models may lead to a broader understanding of which data
structures are easy or hard for gradient-based neural networks to learn from, and
why natural images seem to live in a particularly favourable regime.

\clearpage

\section*{Acknowledgements}
We thank Giulio Biroli, Bruno Loureiro, Jonathan Victor, Gasper Tkačik, and Davide Zoccolan for inspiring discussions, and Matteo Santoro for the conversations on Fourier analysis.
SG acknowledges funding from
the European Research Council (ERC) under the European Union’s Horizon 2020
research and innovation programme, Grant agreement ID 101166056, and funding
from the European Union--NextGenerationEU, in the framework of the PRIN Project
SELF-MADE (code 2022E3WYTY -- CUP G53D23000780001). SG and FR acknowledge
funding from Next Generation EU, in the context of the National Recovery and
Resilience Plan, Investment PE1 -- Project FAIR ``Future Artificial Intelligence
Research'' (CUP G53C22000440006).

\printbibliography

\newpage
\appendix
\onecolumn
\numberwithin{equation}{section}
\numberwithin{figure}{section}
\numberwithin{theorem}{section}



\newpage
\section{Experimental details}\label{app:details_training}
In this appendix, we collect detailed information on how we ran the
experiments of this paper.

\subsection{\Cref{fig:fig1}}%
\label{app:details-figure1}

We use greyscale images from the ``ALOT'' dataset \citep{burghouts_material-specific_2009}, which we downsample to image patches of sizes $8\times 8$, $16 \times 16$ and $32 \times 32$. We then flatten these images into vectors of lengths $N=8^2,16^2,32^2$ and train a two-layer neural network parametrized by 
\begin{align*}
    h_i &= \sum_j (w_1)_{ij} \, x_j + (b_1)_i, \\
    f_{\theta}(x) & = \sum_i (w_2)_i \sigma 
    \left( h_i \right) + b_2, \,
\end{align*}     
where $w_1 \in \mathbb{R}^{k\times N}, w_2 \in \mathbb{R}^{k},  b_1 \in \mathbb{R}^{k}$ and $b_2 \in \mathbb{R}$ are all trainable parameters, and we choose $k=30$. We train to classify textures of type ``cotton'' (label $=1$) from textures of type ``lace'' (label$=-1$). For all experiments, we used the mean squared error as objective function, and trained with learning rate $\eta=10^{-3}$. We show examples of both textures in \cref{fig:Experimental_supp} a-b) and the test accuracies of the classifiers evaluated on a held-out test set of $20 \%$ of the data in \cref{fig:Experimental_supp} c). The empirical test loss is shown in \cref{fig:fig1} g). We also compute the DFT of the test images, swap the phases sample-wise between classes, and evaluate the models' performance on these phase-swapped test set. This is also shown in \cref{fig:fig1} g).

\begin{figure*}[b!]
    \includegraphics[width=\linewidth]{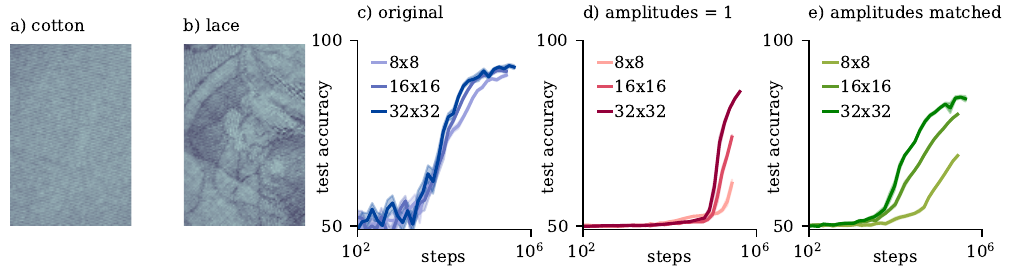}
    \caption{\textbf{a-b)} Examples of images from ALOT dataset for the classes ``cotton'' and ``lace''. \textbf{c)} Test accuracy of trained neural networks on original data. \textbf{d)} Test accuracies of neural networks on the same data, but where the amplitudes of all images have been set to one. \textbf{e)} Test accuracy on the same dataset, but where the amplitudes of class $-1$ are the same as the amplitudes of the class $1$. All curves are averages over 8 random initializations.} \label{fig:Experimental_supp}
\end{figure*}

\subsection{\Cref{fig:fig2}}%
\label{app:details_fig2}
\textbf{Left:} We show here that online SGD    \eqref{algo} cannot access easily the information encoded in the phase of isotropic translation-invariant inputs drawn accordingly to the Fourier data model \eqref{eq:model}. We run SGD on inputs of dimension $N = 100$, up to $N^4$ steps, and we perform 60 averages. The parameters of the Fourier data model are $\varepsilon = 1.2$ and the modified entry is $k_0 = 6$. The activation function is the $4$-th Hermite polynomial $\sigma(s) = s^4 -6 s^2 + 3$ and the learning rate for SGD is $\delta_N = 10^{-3}/N$.
\textbf{Right:} The empirical loss landscape has been computed in the same setting. The level sets on the loss landscape are shown in \Cref{fig:landscape}.

\subsection{\Cref{fig:fig3}}%
\label{app:details_fig3}

\subsubsection{Panels a-d)}
We use the same texture dataset as in \cref{fig:fig1} (see \cref{app:details-figure1}) cropped into patches of sizes $8\times 8$, $16 \times 16$ and $32 \times 32$. We compute the (orthogonal) DFT of all images. In \cref{fig:fig3} a), we show their average squared Fourier amplitudes. 
Then, we train classifiers to distinguish between the ``cotton'' and ``lace'' texture types, using the same architecture as in \cref{app:details-figure1}, in three different cases. First, we use the original texture data. Second, we set all Fourier amplitudes equal to one, and transform back. 
Finally, we transplant the amplitudes of the ``cotton'' class onto the ``lace'' class sample by sample, then transform back. In this last case, we have built two classes of inputs which share the same second order statistics, but have different Fourier phases.
The test accuracies of the networks trained on the said image data are displayed in \cref{fig:Experimental_supp} c-e). The corresponding test losses are shown in \cref{fig:fig3} b-d), with (\legendline[1.5pt]{violet}) for the original data, (\legendline[1.5pt]{reddish}) for the whitened data and (\legendline[1.5pt]{greenish}) for the data with matched amplitudes.

\begin{figure*}[t!]
    \centering
    \includegraphics[width=0.6\linewidth]{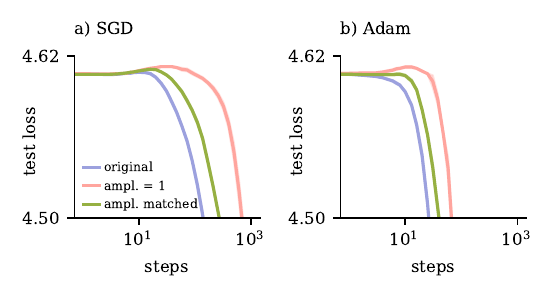}
    \vspace{-0.5em}
    \caption{\textbf{Test loss curves of a Resnet18 trained on CIFAR100} with a) SGD and b) Adam on original data (\legendline[1.5pt]{violet}), data where all Fourier amplitudes have been set to one (\legendline[1.5pt]{reddish}) and data where the amplitudes of all classes are replaced by the amplitudes of the first class (\legendline[1.5pt]{greenish}). All the curves are averaged over three independent runs. We use the cross-entropy loss. \textit{Parameters: }for SGD, we use learning rate $5\cdot10^{-3}$, momentum = 0.9, weight decay of $5\cdot 10^{-3}$, batch size of 128. For Adam, we used learning rate $10^{-3}$, $\beta$ = (0.9, 0.99), and no weight decay. Cf. \cref{fig:fig3} to visualise the same trend on textures and ImageNet. } \label{fig:cifar100}
\end{figure*}

\subsubsection{Panel e) ImageNet experiment with Resnet18}

We trained a standard Resnet18 architecture on a super-classed version of ImageNet with ten classes\footnote{Dog (n02084071), Bird (n01503061),
Insect (n02159955), Monkey (n02484322),
Car (n02958343), Cat (n02120997),
Truck (n04490091), Fruit (n13134947),
Fungus (n12992868), Boat (n02858304)} provided by the \texttt{robustness} library of \citet{robustness}. We performed the same amplitude manipulations as in the texture experiments. We trained the network using vanilla SGD with a learning rate of $10^{-3}$, weight decay of $5 \cdot 10^{-4}$, momentum of 0.9, and mini-batch size of 128. In all three data sets, we set the mean image of each class to zero, and we normalised inputs in all three data sets (vanilla images, images with flat amplitudes, images with shared amplitudes) to have the same mean and standard deviation.

\subsection{\Cref{fig:fig4}}%
\label{app:details_fig4}
\textbf{(Left)} Cartoon of the principal subspace of inputs sampled from the Fourier data model, in case of non-isotropic power-law-decaying inputs.
\textbf{(Middle)}
The ambient dimension is $N = 300$. The parameters for the Fourier data model \eqref{eq:model} are $\varepsilon = 1.2$ and $k_0 = 6$. The other $M = 5$ frequencies for which the corresponding eigenvalue $\lambda_m$ is extensive are $k_1 = 15, k_2 = 24, k_3 = 20, k_4 = 9$ and $k_5 = 18$. The extensive eigenvalues are $(\lambda_{k_0}, \lambda_{k_1}, \lambda_{k_2}, \lambda_{k_3}, \lambda_{k_4}) = (N^{1/2}, N^{1.2/2}, N^{1.1/2}, N^{1.3/2}, N^{1.4/2},N^{0.9/2})$. Normalised online SGD is performed with $\delta_N = 0.03/N$, for 60 averages. The circulant matrix for $\mathbb{P}_0$ and $\mathbb{P}$ has eigenvalues $(\lambda_{k_0}, \lambda_{k_1}, \lambda_{k_2}, \lambda_{k_3}, \lambda_{k_4})$, and all the rest set to one. As always, the principal subspace is shared among the classes $\mathbb{P}_0$ and $\mathbb{P}$.
\textbf{(Right)} For the same parameters, we show the empirical loss function.

\section{Preliminaries}
\subsection{Notation}
W use the standard notations $o,O,\Theta,\Omega,\omega$ for asymptotics. We recall them here for sequences, but they generalize in a straight forward way in the case of functions.\\
Let $(a_k)_{k\in\mathbb N}$, $(b_k)_{k\in\mathbb N}$ be two real-valued sequences. Then,
\begin{align*}
    a_k\in o(b_k) &\  \iff \ \lim_{k\to \infty}\frac {a_k}{b_k}=0,\\ 
    a_k\in O(b_k) &\  \iff \exists \,C>0\in \mathbb R \quad \forall k>k_0 \quad |a_k|\le C |b_k|,\\ 
     a_k\in \Theta(b_k) &\  \iff \exists \, C_1,C_2> 0 \quad \forall k>k_0 \quad C_1 \lvert b_k \rvert \le \lvert a_k \rvert \le C_2 \lvert b_k \rvert, \\
       a_k\in \Omega(b_k) &\  \iff \exists\, C>0\in \mathbb R \quad \forall k>k_0 \quad |a_k|\ge C |b_k|, \\
      a_k\in \omega(b_k) &\  \iff \ \lim_{k\to \infty}\frac {|a_k|}{|b_k|}=\infty. 
\end{align*}
Occasionally, we use the shorthands $a_k \ll b_k$ for $a_k=o(b_k)$, $a_k\lesssim b_k$ for $a_k=O(b_k)$ and $a_k \asymp b_k$ for $a_k=\Theta(b_k)$.
\begin{paragraph}{Pushforward}
Given a measurable map $T: X \to Y$ and a probability measure $\mu$ on $X$, we denote by $T_\#\mu$ the pushforward of $\mu$ by $T$ defined by 
\begin{math}
    T_\#\mu(A) = \mu(T^{-1}(A), 
\end{math}
for any measurable set $A \subset Y$.
\end{paragraph}

\subsection{Fourier analysis}
\begin{definition}[DFT]
    The Discrete Fourier Transform ($\rm{DFT}$) is the linear map \mbox{$\mathcal{F} : \mathbb{R}^N \to \mathbb{C}^N$} such that for $k = 0, \dots, N-1$ we have
    \begin{equation*}
        \mathcal{F}(x)_k = \sum_{n=0}^{N-1} x_n e^{-2\pi in\frac{k}{N}},
    \end{equation*}
    for any $x \in \mathbb{C}^N$.
    We denote with $F \in \mathbb{C}^{N \times N}$ the $\rm{DFT}$ matrix defined as
    \begin{math}
        F_{kn} = e^{-i 2\pi kn /N}.
    \end{math}
    The Inverse Discrete Fourier Transform ($\rm{IDFT}$) is instead $\mathcal{F}^* : \mathbb{C}^N \to \mathbb{C}^N$ such that for any $X \in \mathbb{C}^N$ we have
    \begin{equation*}
        \mathcal{F^*}(X)_n = \frac{1}{N}\sum_{k=0}^{N-1} X_k e^{2\pi i k \frac{n}{N}}.
    \end{equation*}
    In what follows, we will use the notations $X$ = $\rm{DFT}(x) = \mathcal{F}(x)$ and $x$ = $\rm{DFT}(X) = \mathcal{F}^*(X)$.
    Another possible choice would be the normalised $\rm{DFT}$ matrix defined as $\widehat{F} = \frac{1}{\sqrt{N}}F$.
\end{definition}
Note that, for real inputs $x \in \mathbb{R}^N$, their DFT exhibits conjugate symmetry, that is $X_k = \overline{X}_{N-k}$. 
As a consequence, the non-redundant frequencies are only $k = 0, \dots, \lfloor N/2 \rfloor$. 
The frequencies $k = 0$ and $k = N/2$, if $N$ is even, are called DC and Nyquist frequencies, respectively. They correspond to the real DFT Fourier coefficients $X_0$ and $X_{N/2}$. Being purely real, these entries are not candidates to exhibit a perturbation of the phase in the sense of the Fourier data model \eqref{eq:model}, for which we consider $k_0 \notin \{0,N/2\}$.

The next definition identifies a set of particular vectors in $\mathbb{R}^N$. Since these vectors are orthonormal, they form an orthonormal basis of $\mathbb{R}^N$, which we refer to as the (real) DFT basis.
\begin{definition}[DFT basis]\label{def:DFTbasis}
    The cosine $\rm{DFT}$ basis vectors $(u^k)_k$ and the sine $\rm{DFT}$ basis vectors $(v^k)_k$ in $\mathbb{R}^N$, for the frequencies $k = 1, \dots, \lfloor \frac{N-1}{2} \rfloor$, are defined by 
    \begin{equation*}
        u^k_n = \sqrt{\frac{2}{N}} \cos{\biggl(\frac{2\pi kn}{N}\biggr)} \qquad \text{ and }\qquad v^k_n = \sqrt{\frac{2}{N}} \sin{\biggl(\frac{2\pi kn}{N}\biggr)},
    \end{equation*}
    for any $n = 0, \dots, N-1$. 
    In addition, we define the constant basis vectors $\psi^{0}, \psi^{N/2}$ such as
    $\psi^{0}_n = 1/\sqrt{N}$ and
    $\psi^{N/2}_n = (-1)^n/N,$
    for $n = 0, \dots, N-1$.
    The $\rm{DFT}$ basis consists of the constant basis vectors $\psi^0, \psi^{N/2}$, the cosine $\rm{DFT}$ basis vectors $u^k$ and the sine $\rm{DFT}$ basis vectors $v^k$.
\end{definition}

\subsection{Translation-invariant inputs}\label{app:translationinvariance}
In many fields like visual processing and statistical modeling, it is very common \citep{simoncelli2001images}, \citep{hyvarinen2009natural} to assume that natural images and textures exhibit wide-sense stationarity, meaning that their second-order statistics are translation-invariant or, to put it differently, that their covariance matrix is \textit{Toeplitz}, in the sense that each entry of the matrix depends only on the relative distance of the corresponding indices. 
\begin{definition}[Toeplitz matrix]\label{def:toeplitz}
    A matrix $\Sigma \in \mathbb{C}^{N\times N}$ is ``Toeplitz'' if
    there exists a function $c: \mathbb{Z} \to \mathbb{C}$ such that 
    \begin{equation*}
        \Sigma_{ij} = c(i-j).
    \end{equation*}
    A zero-mean random vector is said to be ``wide-sense stationary'' if its covariance matrix is Toeplitz.
\end{definition}
For matrices, \textit{circularity} is a more restrictive property with respect to the just mentioned notion of invariance. Indeed, circulant matrices form a proper subset of Toeplitz ones, since they impose additional constraints on the boundary conditions. The entries of a Toeplitz matrix are constant along its diagonals, while each row of a circulant matrix is a circular shift of the previous one, in the following sense. 
\begin{definition}[Circulant matrix]\label{def:circulant}
    A matrix $\Sigma \in \mathbb{C}^{N \times N}$ is said to be ``circulant'' if each entry is such that $\Sigma_{ij} = c((i-j) \rm{\text{ mod }N)}$, for a function $c: \mathbb{Z} \to \mathbb{C}$. This means that there exist $N$ coefficients $c_0, \dots, c_{N-1}$ such as
    \begin{equation*}
        \Sigma =
        \begin{pmatrix}
            c_0     & c_{N-1} & c_{N-2} & \cdots & c_1 \\
            c_1     & c_0     & c_{N-1} & \cdots & c_2 \\
            c_2     & c_1     & c_0     & \cdots & c_3 \\
            \vdots  & \vdots  & \vdots  & \ddots & \vdots \\
            c_{N-1} & c_{N-2} & c_{N-3} & \cdots & c_0
        \end{pmatrix}.
    \end{equation*}
\end{definition}
In the Fourier data model \eqref{eq:model}, we ask that the covariance matrix of the original inputs $z \in \mathbb{R}^N$ is circulant, rather than only Toeplitz. We require this in order to have the covariance matrix of the inputs diagonalized by the DFT basis defined in Definition \ref{def:DFTbasis}. 
For this reason, in this work, we say that the inputs are \textit{translation-invariant} if they satisfy Definition \ref{def:translationinputs}.
We emphasize that the notion of translation-invariance adopted here is not the standard one, which often refers to random vectors whose distribution is invariant under translation. However, in the classical sense, translation invariance implies wide-sense stationarity, in the sense of Definition \ref{def:toeplitz}. Moreover,
in high-dimensional settings, there are some well-established results (e.g. \citep{Zhu2016circulant}) proving that Toeplitz matrices can be approximated by circulant matrices in a spectral sense, meaning that the eigenvalue distribution of a proper circulant approximation captures the bulk of energy in the Toeplitz covariance matrix. These results are built on classical limit theorems from \citep{Szego1961circulant} and \citep{Bottcher1990circulant}. In view of these considerations, the proposed definition of translation-invariant inputs is appropriate in the context of modeling high-dimensional natural images or textures.
\begin{definition}[Translation-invariant inputs]\label{def:translationinputs}
    A zero-mean random vector $x \in\mathbb{R}^N$ is said to be ``translation-invariant'' if its covariance matrix is circulant. 
\end{definition}
Note that, because of the symmetry of any covariance matrix, a translation-invariant random vector has covariance $\Sigma$ such that
    \begin{math}
        c_{(i-j) \text{ mod }N} = 
        \Sigma_{ij} = \Sigma_{ji} = 
        c_{(j-i) \text{ mod }N},
    \end{math}
    for some coefficients $(c_0, \dots, c_{N-1})$.
    Hence, given $\tau = (i-j) \text{ mod }N$, it holds $(j-i) \text{ mod }N = N -\tau$, from which it follows that $c_{\tau} = c_{N-\tau}$ for $\tau = 1, \dots, N-1$. 
    Then, the only free entry $c_0$ is on the main diagonal, while the other entries $(c_1, \dots, c_{N-1})$ are forced to be symmetric about the center, leading to the following structure:
    \begin{equation*}
\Sigma =
\begin{pmatrix}
c_0 & c_1 & c_2 & \cdots & c_2 & c_1 \\
c_1 & c_0 & c_1 & \cdots & c_3 & c_2 \\
c_2 & c_1 & c_0 & \cdots & c_4 & c_3 \\
\vdots & \vdots & \vdots & \ddots & \vdots & \vdots \\
c_2 & c_3 & c_4 & \cdots & c_0 & c_1 \\
c_1 & c_2 & c_3 & \cdots & c_1 & c_0
\end{pmatrix}.
\end{equation*}

In the next lemma, we prove that it is possible to diagonalise any circulant covariance matrix through the DFT basis. This is the main reason why circulant matrices are fundamental in studying the structure of inputs distributed according to the Fourier data model \eqref{eq:model}.
\begin{lemma}[Diagonalisation of circulant matrices]\label{app:diagonalisation}
    Any symmetric circulant matrix is diagonalized by the $\rm{DFT}$ basis.
\end{lemma}
\begin{proof}
For any circulant matrix $\Sigma$, it holds that $\Sigma_{ij} = c_{(i-j) \text{ mod } N}$. It is well known that for any $k = 1, \dots, \lfloor \frac{N-1}{2} \rfloor$, the vectors 
\begin{equation*}
    w^k = \frac{1}{\sqrt{N}}
    \begin{pmatrix}
        1\\
        \omega^k\\
        \omega^{2k}\\
        \vdots\\
        \omega^{(N-1)k}
    \end{pmatrix},
\end{equation*}
where $\omega = e^{i 2\pi/N}$ is the fundamental $N$-th root of unity, are (complex) eigenvectors of $\Sigma$. This is proven for example in the standard reference \citep{davis1994circulant}.
Moreover, $\Sigma$ is symmetric and then its eigenvalues $(\lambda_0, \dots, \lambda_{\lfloor \frac{N-1}{2}\rfloor})$ are real.
Since
\begin{equation*}
    w^k_n = \frac{1}{\sqrt{N}} \cos{\biggl( \frac{2\pi kn}{N} \biggr)} + i \frac{1}{\sqrt{N}} \sin{\biggl( \frac{2\pi kn}{N} \biggr)},
\end{equation*}
the DFT basis vectors defined Definition \ref{def:DFTbasis} are (real) eigenvectors of $\Sigma$ corresponding to the real eigenvalue $\lambda_k$.
\end{proof}

The next lemma is at the basis of the already mentioned disentanglement between low-order and higher-order correlations of translation-invariant inputs or, in Fourier terms, between Fourier amplitudes and Fourier phases.
\begin{lemma}\label{app:translation-invariance}
    Consider a translation-invariant random vector $z \in \mathbb{R}^N$ and its $\rm{DFT}$, i.e. $Z = \rm{DFT}(z)$, with $Z_k = \rho_k e^{i\varphi_k}$ for $k = 0, \dots, N-1$. Then, its covariance matrix $\rm{Cov}[Z]$ is diagonal and such that 
    \begin{equation*}
        \rm{Cov}[Z]_{kk} = \mathbb{E}[\rho_k^2].
    \end{equation*}
\end{lemma}
\begin{proof}
    Since $z$ has circulant covariance, for any $n,m = 0, \dots, N-1$, denote its entries by the coefficients $c_{(n-m) \text{ mod } N} = \mathbb{E}[z_n z_m]$. Therefore, 
    \begin{align*}
        \mathbb{E}[Z_k \overline{Z}_l] = \sum_{n,m=0}^{N-1} e^{-2\pi i (kn + lm) /N} \underbrace{\mathbb{E}[z_n z_m]}_{c_{(n-m) \text{ mod } N}} = 
        \sum_{\tau = 0}^{N-1} c_{\tau} \sum_{m=0}^{N-1}e^{-2\pi i(k(m+\tau) -ml)},
    \end{align*}
    where $\tau = (n-m) \text{ mod }N$. If we define 
    \begin{math}
        S_k = \sum_{\tau = 0}^{N-1}c_{\tau} e^{-2\pi i \tau k/N},
    \end{math}
    it holds that
    \begin{equation*}
        \mathbb{E}[Z_k \overline{Z}_l] = S_k \sum_{m=0}^{N-1}e^{-2\pi im(k-l)/N} = \begin{cases}
        S_k N  \quad\quad &\text{if } k \equiv_{N} l,\\
        0 \quad\quad &\text{otherwise}.
    \end{cases}
    \end{equation*}
    Now, note that for any $k = 0, \dots, N-1$, we can write that
    \begin{equation*}
        \mathbb{E}[\rho_k^2] = \mathbb{E}[Z_k \overline{Z}_k] = \sum_{n,m=0}^{N-1}e^{-2\pi i k(n-m)}\mathbb{E}[z_n z_m] = \underbrace{\sum_{\tau = 0}^{N-1} c_{\tau} e^{-2\pi i k\tau}}_{S_k} \, N,
    \end{equation*}
    and then $\text{Cov}[Z]_{kl} = \mathbb{E}[Z_k \overline{Z}_l] = \mathbb{E}[\rho_k^2] \delta_{k \equiv_N l}$, which implies the thesis.
\end{proof}

\subsection{Amplitudes and phases of the DFT coefficients}
In this small section, we investigate closely the distributions of the amplitudes and phases of a Gaussian-distributed translation-invariant random vector (cf. Definition \ref{def:translationinputs}).
\begin{definition}[Rayleigh distribution]
    A scalar random variable $Y$ is said to be Rayleigh-distributed with
    parameter $\sigma > 0$, and we will write 
\begin{math}
Y \sim \mathrm{Rayleigh}(\sigma),
\end{math}
if its probability density function is given by
\begin{equation*}
p_Y(y)=
\begin{cases}
\dfrac{y}{\sigma^2}
\exp\!\left(-\dfrac{y^2}{2\sigma^2}\right)
& y \ge 0, \\[1em]
0 & y < 0.
\end{cases}
\end{equation*}
\end{definition}
We can note that the square root of the sum of two squared independent Gaussian random variables is distributed according to a Rayleigh distribution.
\begin{remark}
    Given $Z_1, Z_2 \sim \mathcal{N}(0,\sigma^2)$ independent random variables,
\begin{math}
Y = \sqrt{Z_1^2 + Z_2^2}
\end{math}
follows a Rayleigh distribution with parameter $\sigma$. 
\end{remark}
\begin{remark}[Statistics of Rayleigh distribution]\label{statisticsray}
    Consider 
    \begin{math}
    Y \sim \mathrm{Rayleigh}(\sigma).
    \end{math}
    Then, 
\begin{equation*}
\mathbb{E}[Y] = \sigma \sqrt{\frac{\pi}{2}},
\qquad
\mathrm{Var}(Y) = \frac{4-\pi}{2}\,\sigma^2, 
\qquad 
\mathbb{E}[Y^k]=\sigma^k 2^{k/2}\Gamma\!\left(1 + \frac{k}{2}\right),
\end{equation*}
where $\Gamma(\cdot)$ is the Gamma function and $k \geq 2$.
\end{remark}
In the next lemma we prove that the DFT coefficients of translation-invariant Gaussian noise are independent and that their amplitudes and phases are Rayleigh- and uniform-distributed, respectively.
\begin{lemma}\label{app:Fourierdistribution}
   Consider a zero-mean translation-invariant Gaussian random vector $z \in \mathbb{R}^N$. Then, its $\rm{DFT}$ coefficients $Z_k$ are independent and 
    \begin{equation*}
        \rho_{k} \sim \rm{Rayleigh}\biggl( \sqrt{\frac{N \lambda_{k}}{2}} \biggr),\qquad 
        \varphi_k \sim \rm{Unif}([-\pi, \pi)),
    \end{equation*}
    where $\lambda_k$ is the $k$-th eigenvalue of the covariance matrix $\Sigma$ of the inputs.
\end{lemma}
\begin{proof}
    Any Gaussian-distributed random vector $z \in \mathbb{R}^N$ has a Gaussian DFT $Z = \rm{DFT}(x)$. Hence, each entry, and the corresponding real and imaginary parts are Gaussian-distributed. Since $\Sigma$ is circulant, the entries $Z_k$ are uncorrelated, and then independent. 
    We look now at the statistics of Re$(Z_k)$ and Im$(Z_k)$. The real part of the Fourier coefficient is 
    \begin{math}
        \text{Re}(Z_k) = \sum_{n=0}^{N-1} x_n \cos\left(\frac{2\pi kn}{N}\right).
    \end{math}
    We want to compute 
\begin{align*}
        \text{Var}(\text{Re}(Z_k)) &= \mathbb{E}\left[\left( \sum_{n=0}^{N-1} z_n \cos\left(\frac{2\pi kn}{N}\right) \right) \left( \sum_{m=0}^{N-1} z_m \cos\left(\frac{2\pi km}{N}\right) \right) \right]\\
        &= \sum_{n=0}^{N-1} \sum_{m=0}^{N-1} E[z_n z_m] \cos\left(\frac{2\pi kn}{N}\right) \cos\left(\frac{2\pi km}{N}\right).
\end{align*}
Since $E[z_n z_m] = c_{(n-m) \text{ mod }N}$,
\begin{equation}\label{split}
\text{Var}(\text{Re}(Z_k)) = \frac{1}{2} \sum_{m,n=0}^{N-1} c_{(n-m) \text{ mod }N} \left[ \cos\left(\frac{2\pi k(n-m)}{N}\right) + \cos\left(\frac{2\pi k(n+m)}{N}\right) \right],
\end{equation}
where we have applied the trigonometric identity $\cos(A) \cos(B) = 1/2\, [\cos(A-B) + \cos(A+B)]$.\\
We can split \cref{split} into two sums: $\text{Var}(\text{Re}(Z_k)) = S_1 + S_2$.
Define $\tau = (n-m) \mod N$. For any fixed $n$, as $m$ cycles through $0, \dots, N-1$, the index $r$ also cycles through $0, \dots, N-1$. Then,
\begin{equation*}
S_1 = \frac{1}{2} \sum_{n=0}^{N-1} \left( \sum_{m=0}^{N-1} c_{(n-m) \text{ mod } N} \cos\left(\frac{2\pi k(n-m)}{N}\right) \right).
\end{equation*}
The inner sum is the definition of the $k$-th eigenvalue $\lambda_k$ for the circulant covariance matrix $\Sigma$, since
\begin{equation*}
\sum_{r=0}^{N-1} c_r \cos\left(\frac{2\pi kr}{N}\right) = \text{Re}(\lambda_k) = \lambda_k.
\end{equation*}
Thus,
\begin{math}
S_1 = \frac{1}{2} \sum_{n=0}^{N-1} \lambda_k = (N \lambda_k)/2.
\end{math}
For $S_2$, we get that
\begin{equation*}
S_2 = \frac{1}{2} \sum_{r=0}^{N-1} c_r \sum_{m=0}^{N-1} \cos\left(\frac{2\pi k(2m+r)}{N}\right)
\end{equation*}
and, by expanding the cosine, it hods that
\begin{equation*}
S_2 = \frac{1}{2} \sum_{r=0}^{N-1} c_r \left[ \cos\left(\frac{2\pi kr}{N}\right) \sum_{m=0}^{N-1} \cos\left(\frac{4\pi km}{N}\right) - \sin\left(\frac{2\pi kr}{N}\right) \sum_{m=0}^{N-1} \sin\left(\frac{4\pi km}{N}\right) \right].
\end{equation*}
For $0 < k < N/2$, the value $2k$ is not a multiple of $N$. Therefore, the sum of the sinusoids over $N$ points is zero, i.e.
\begin{equation*}
\sum_{m=0}^{N-1} e^{j \frac{2\pi (2k) m}{N}} = 0.
\end{equation*} 
In conclusion, $S_2 = 0$. Then, $\text{Re}(Z_k) \sim \mathcal{N}(0, N\lambda_{k}/2)$. Since the same is true for the imaginary part, the first part of the thesis follows from the definition of a Rayleigh-distributed random variable. For a fixed frequency $k$, we denote now $A = \text{Re}(X_k)$, $B = \text{Im}(X_k)$ and $\sigma^2 = (\lambda_{k}N)/2$. Then, the joint distribution for $(A,B)$ is 
\begin{equation*}
    f_{A,B}(a,b) = \frac{1}{2\pi \sigma^2} e^{-\frac{a^2 + b^2}{2\sigma^2}}.
\end{equation*}
In the polar coordinates $A = R \cos(\phi)$ and $A = R \sin(\phi)$, we have $f_{R, \phi}(r, \varphi) = f_{A,b}(r \cos{\varphi}, r \sin{\varphi}) |J|$, for $J$ Jacobian of the transformation. Hence, 
\begin{equation*}
    f_{R, \phi}(r, \varphi) = \frac{r}{2\pi \sigma^2} e^{- r^2 / 2\sigma^2} = \frac{r}{\sigma^2} e^{-\frac{r^2}{2\sigma^2}} \frac{1}{2\pi}
\end{equation*}
and marginalising we obtain the thesis:
\begin{equation*}
    f_{\phi}(\varphi) = \int_{0}^{+\infty} f_{R, \phi}(r, \varphi) dr = \frac{1}{2\pi}\int_{0}^{+\infty} \frac{r}{\sigma^2} e^{-\frac{r^2}{2\sigma^2}} dr = \frac{1}{2\pi}.
\end{equation*}
\end{proof}

\subsection{Hermite polynomials}\label{appsec:hermite}
In this section, we recall some basic definitions about Hermite polynomials \citep{szegő1975orthogonal}, which happens to be useful tools to study learning dynamics for classical estimation problems.
In what follows, $\mathbb{P}_0$ denotes a univariate or multivariate normal Gaussian distribution. However, in some cases, $\mathbb{P}_0$ may admit a non-trivial covariance matrix. This will be specified explicitly when needed. 

\begin{definition}[Gaussian product] Given $f,g:\mathbb R^N\to \mathbb R$ and $\mathbb{P}_0$ multivariate normal Gaussian distribution, we define the ${\rm L}^2$-scalar product as
\begin{equation*} \langle f,g \rangle_{\mathbb{P}_0}=\mathbb E_{\mathbb P_0}\left[f(x)g(x)\right].
\end{equation*} 
The space ${\rm L}^2(\mathbb R^N,\mathbb P_0)$ contains all the measurable functions $f: \mathbb{R}^N \to \mathbb{R}$ such that $||f||^2=\langle f,f\rangle_{\mathbb P_0}<\infty$.
\end{definition}

\begin{definition}[Hermite polynomials]
The (non-normalised) probabilists' Hermite polynomials $(h_k)_k$ are defined by $h_0(x)=1, h_1(x)=x$ and, for $k \geq 1$, recursively by
\begin{equation*}
    h_{k+1}(x) = x h_k(x) - k h_{k-1}(x).
\end{equation*}
The first Hermite polynomials are
    \begin{align*}
        h_0(x) &= 1,\\
        h_1(x) &= x,\\
        h_2(x) &= x^2 -1,\\
        h_3(x) &= x^3 -3 x,\\
        h_4(x) &= x^4 - 6 x^2 + 3.
    \end{align*}
\end{definition}

\begin{lemma}[Orthogonality property]\label{app:ort}
    Let $\mathbb{P}_0$ be the multivariate normal Gaussian distribution.
    Consider $w_1, w_2 \in \mathbb{S}^{N-1}$ and set $\alpha = w_1 \cdot w_2$.
    Then, for any $i,j \in \mathbb{N}$,
    \begin{equation*}
        \mathbb{E}_{\mathbb{P}_0}[h_i(w_1 \cdot x) h_j(w_2 \cdot x)] = i! \,\alpha^i \;\delta_{ij}.
    \end{equation*}
\end{lemma}

\begin{definition}[Hermite expansion] \label{def:herm exp}
    Consider a function $f: \mathbb{R} \to \mathbb{R}$ that is square integrable with respect to the univariate Gaussian normal distribution $\mathbb{P}_0$. Then, there exists a unique sequence of real numbers $(c_k)_{k \in \mathbb{N}}$ called Hermite coefficients such that
    \begin{equation*}
        f(x) = \sum_{k = 0}^{\infty} \dfrac{c_k}{k!} h_k(x) \quad \text{ and } \quad c_k(x) = \mathbb{E}_{\mathbb{P}_0}[f(x) h_k(x)],
    \end{equation*}
    where $h_k$ is the $k$-th probabilists' Hermite polynomial.
\end{definition}

\subsection{Information exponent and likelihood ratios}\label{app:infolike}
\begin{definition}[Information exponent]\label{def:information_exponent} 
Given any function $f$ for which the Hermite expansion exists, its information exponent $k^* = k^*(f)$ is the smallest index $k \geq 1$ such that $c_k \neq 0$. 
\end{definition}

In order to compute the information exponent of the population correlation loss \eqref{def:population} for non-Gaussian distributions, we can exploit the \textit{likelihood ratio}. It allows us to rewrite the average over the non-Gaussian inputs as an average over the multivariate Gaussian distribution \mbox{$\mathbb{P}_0 = \mathcal{N}(0,\Sigma)$}.

\begin{figure*}[t!]
\centering
\includegraphics[width=0.9\linewidth]{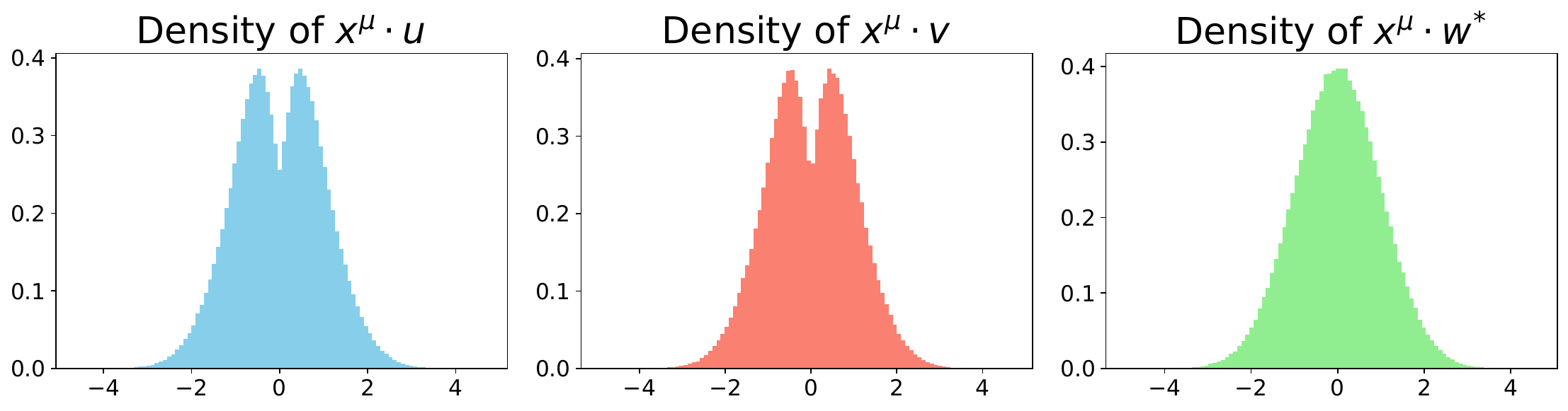}
\caption{\label{fig:projections}
\textbf{Distributions of the inputs drawn from the Fourier data model along the DFT phase vectors.} Consider the DFT phase vectors $u$ and $v$ (cf. Definition \ref{def:DFTbasis}) and a fixed vector $w^*$ uniformly drawn from the sphere in dimension $N = 30$. The DFT phase vectors span the only subspace of $\mathbb{R}^N$ along which the projected inputs are non-Gaussian-distributed. Indeed, (\textbf{left}) and (\textbf{middle}) are non-Gaussian, whereas (\textbf{right}) is Gaussian.
}
\end{figure*}

\begin{definition}[Likelihood ratio]\label{def:likelihood}
Let $\mathbb{P}$ and $\mathbb{P}_0$ be two probability distributions on $\mathbb{R}^N$. 
If $\mathbb{P}$ is absolutely continuous with respect to $\mathbb{P}_0$,
then there exists a measurable function 
\begin{equation*}
    \frac{\mathrm{d}\mathbb{P}}{\mathrm{d}\mathbb{P}_0} : \mathbb{R}^N \to [0,\infty)
\end{equation*}
such that for every measurable $A \subseteq \mathbb{R}^N$, we get that
\begin{equation*}
P(A) = \int_A \frac{\mathrm{d}\mathbb{P}}{\mathrm{d}\mathbb{P}_0} \, \mathrm{d}\mathbb{P}_0.
\end{equation*}
The function
\begin{math}
\ell = \mathrm{d}\mathbb{P} \backslash \mathrm{d}\mathbb{P}_0
\end{math}
is called ``likelihood ratio'' (or ``Radon--Nikodym derivative'')
of $\mathbb{P}$ with respect to $\mathbb{P}_0$. 
If both $\mathbb{P}$ and $\mathbb{P}_0$ admit densities $p$ and $q$ with respect to the Lebesgue measure, then, for any $y \in \mathbb{R}^N$, we have that
\begin{math}
    \ell(y) = p(y)/q(y).
\end{math}
\end{definition}

Assume that $\mathbb{P}$ is a non-Gaussian distribution, $\mathbb{P}_0$ is a Gaussian distribution and that the likelihood ratio between $\mathbb{P}$ and $\mathbb{P}_0$ depends only on the projection of the inputs along a fixed number of directions $\ell = \ell(w^*_1 \cdot x, \dots, w^*_k \cdot x)$.
Intuitively, the likelihood ratio (or more precisely, its norm) measures how different the distribution of the projection $s = w\cdot x$ is from a standard Gaussian distribution, when the weight vector $w$ belongs to the subspace spanned by $\{w^*_1, \dots, w^*_k\}$.
Given that, we can write 
\begin{equation}
    \mathbb{E}_{\mathbb{P}}[\sigma(w \cdot x)] 
    =  \mathbb{E}_{\mathbb{P}_0}[\sigma(w \cdot x) \ell(w^*_1 \cdot x, \dots,  w^*_k \cdot x)].
\end{equation}
The likelihood ratio will thus be the key object in our study to determine the number of samples that a given algorithm like SGD requires to find a projection that is distinctly non-Gaussian and yields a large likelihood ratio,
akin to how one would proceed in an analysis with the second moment method in hypothesis testing~\citep{kunisky2019notes}.

\section{Statistical properties of the Fourier data model}
Here, we recall what the definition of the Fourier data model \eqref{eq:model} is.
Assume that $(z^{\mu})_{\mu=0}^{n} \subseteq \mathbb{R}^N$ are $n$ samples drawn from $\mathbb{P}_0 = \mathcal{N}(0, \Sigma)$, with $\Sigma \in \mathbb{R}^{N \times N}$ circulant matrix. Consider $Z^{\mu} = \text{DFT}(z^{\mu})$, with $Z^{\mu}_k = \rho^{\mu}_k e^{i \varphi^{\mu}_k}$ for $k = 0, \dots, N-1$. 
Because of Lemma \ref{app:Fourierdistribution},
we have that $\rho^{\mu}_k \sim \text{Rayleigh}(\sqrt{\lambda_k N}/2)$, where $\lambda_k$ is the $k$-th eigenvalue of $\Sigma$ and that the phases are independent and $\varphi^{\mu}_k \sim \text{Unif}([-\pi, \pi))$. Note that, since the inputs are real, we have $Z^{\mu}_{N-k} = \overline{Z}^{\mu}_{k}$. 
For any $\varepsilon > 0$, choose $k_0 \neq 0, N/2$ and define $X^{\mu}$ by 
\begin{small}
\begin{equation*}
X^{\mu} =
\begin{pmatrix}
\rho^{\mu}_0\\
\vdots\\
\rho^{\mu}_{N-1}
\end{pmatrix}
\exp 
\begin{pmatrix}
i\,\varphi^{\mu}_0\\
\vdots\\
i\,\psi^{\mu}_{k_0}\\
\vdots\\
i\,\varphi^{\mu}_{N-1}
\end{pmatrix}, 
\quad\quad\quad
\psi^{\mu}_{k_0} = \varphi^{\mu}_{k_0} + \varepsilon \underbrace{ f({\varphi^{\mu}_{k_0}})}_{\text{HOCs}} + U^{\mu}, 
\end{equation*}
\end{small}
where $U^{\mu}$ is sampled from a discrete random variable, independent from $\varphi^{\mu}_k$, which takes values with equal probability in the angles corresponding to the 4th roots of unity, i.e. $\{0,\pi/2,\pi,3\pi/2\}$ with probability $1/4$. 
In order to have real inputs in pixel space, we ask that $X^{\mu}_{N-k_0} = \overline{X}^{\mu}_{k_0}$.
We call $\mathbb{P}$ the distribution of $x^{\mu} = \text{IDFT}(X^{\mu})$ in pixel space. 

The next lemma clarifies why adding a simple spike independent from the random phase would have not broken the uniformity of the phase, leading to gaussian inputs. Then, the dependence of $f$ from the phase is crucial to guarantees that $\mathbb{P}$ has some non-trivial higher-order statistics and turns out to be, in fact, non-Gaussian.

\begin{lemma}\label{app:perturbation}
    Denote by 
    $\mu = \rm{Unif}{(\mathbb{T}^N)}$ the distribution of a random vector on the torus $\mathbb{T} = \mathbb{R}/(2\pi \mathbb{Z})$. Then, if $\varphi \sim \mu$, $w^* \in\mathbb{R}^N$ and $\nu$ is a scalar random variable on $\mathbb{R}$ following any distribution, we get that $\varphi' = \varphi + \nu w^* \sim \mu$. 
\end{lemma}
\begin{proof}
    Consider $K$-measurable sets $A_k \subseteq \mathbb{T}$ and $A = \prod_{k=1}^K \in \mathbb{T}^N$. We have that 
    \begin{equation*}
        P(\varphi' \in A \mid \nu) = \prod_{k=1}^K P(\varphi_k + \nu v_k \in A_k \mid \nu) = \prod_{k=1}^K \underbrace{\mu(A_k - \nu v_k)}_{\mu(A_k)},
    \end{equation*}
    where we have noted that, being $\mu$ a Haar measure on the compact group $(\mathbb{T}^N, +)$, it is translational invariant on the torus. Then, by integrating out $\nu$, we conclude that 
    \begin{equation*}
        P(\varphi' \in A) = \prod_{k=1}^K \mu(A_k) = \mu(A).
    \end{equation*}
\end{proof}

We compute now the statistics of the inputs sampled from the Fourier data model \eqref{eq:model}. For simplicity, we will often drop the dependence on the index for each sample $\mu$.
By linearity of the DFT, it is clear that the new inputs have mean zero both in pixel and Fourier space. In Lemma \ref{app:second}, we prove that the new inputs $x$ and the original inputs $z$ have the same covariance matrix both in pixel space and in Fourier space.

\begin{lemma}[Second-order cumulants]\label{app:second}
Let $X \in \mathbb{C}^N$ be a random vector distributed according to the Fourier data model \eqref{eq:model} and consider $x = \rm{IDFT}(X)$. Then, $\rm{Cov}[X] = \rm{Cov}[Z]$ and $\rm{Cov}[x] = \Sigma$.
\end{lemma}
\begin{proof}
    Given that $x, z \in \mathbb{R}^N$ have mean zero, it holds that 
    \begin{equation*}
        \mathbb{E}[x x^{\top}] = F^{-1} \mathbb{E}[X X^*] {F^*}^{-1} \quad \text{and} \quad
        \mathbb{E}[z z^{\top}] = \Sigma = F^{-1} \mathbb{E}[Z Z^*] {F^*}^{-1}.
    \end{equation*}
    Because of that, it is sufficient to prove that $\mathbb{E}[XX^{*}] = \mathbb{E}[Z Z^{*}]$.
    Note that $U$ has been chosen such that 
    \begin{equation*}
        \mathbb{E}[e^{i m U}] = 
        \begin{cases}
        1  \quad\quad \text{if } m \equiv_{4} 0,\\
        0 \quad\quad \text{otherwise}.
    \end{cases}
    \end{equation*}
    We can compute $\mathbb{E}[X_k X_l]$, for $k,l = 0, \dots, N-1$. The entries for which $k, l \notin \{k_0, N - {k_0}\}$ are non-modified entries of the original inputs $Z$, and then they equal $\mathbb{E}[Z_k Z_l]$. Else, we have that
    \begin{small}
    \begin{equation*}
        \mathbb{E}[X_k \overline{X}_l] = 
        \begin{cases}
        \mathbb{E}[X_{k_0} \overline{X}_{k_0}] = \mathbb{E}[\rho_{k_0}^2]  &\quad \text{if } k = l = k_0,\\[2pt]
        \mathbb{E}[X_{N-k_0} \overline{X}_{N - k_0}] = \mathbb{E}[\overline{X}_{k_0} X_{k_0}] = \mathbb{E}[\rho_{k_0}^2] &\quad \text{if } k = l = N - k_0,\\[2pt]
        \boldsymbol{\mathbb{E}[X_{k_0} \overline{X}_{N-k_0}] = 
        \mathbb{E}[X_{k_0}^2] = \mathbb{E}[Z_{k_0}^2 e^{2i \varepsilon \sin{\varphi_{k_0}}}] \, \mathbb{E}[e^{i2U}] = 0} &\quad \text{if } k= k_0, l = N-k_0, \\[2pt]
        \mathbb{E}[X_{k_0} X_l] = \mathbb{E}[X_{k_0}] \mathbb{E}[X_l] = 0 &\quad \text{if } k = k_0, l \notin \{k_0, N-{k_0}\}. 
    \end{cases}
    \end{equation*}
    \end{small}
    In the case of the entry $(k_0, l)$, with $l \notin \{k_0, N - {k_0}\}$, we have used the independence coming from Gaussian translation-invariant inputs. Moreover,
    thanks to Lemma \ref{app:translation-invariance}, $\text{Cov}[Z]_{kl} = \mathbb{E}[\rho_{k}^2] \delta_{kl}$ and then we can conclude that
    $\mathbb{E}[X X^*] = \mathbb{E}[Z Z^*]$.
    Note that the choice of $U$ guarantees that $\text{Cov}[X]$ stays diagonal by dealing with the case $\boldsymbol{k = k_0, l = N-k_0}$, which instead would have broken translation invariance.
\end{proof}

\begin{figure*}[t!]
\centering
\includegraphics[width=0.32\linewidth]{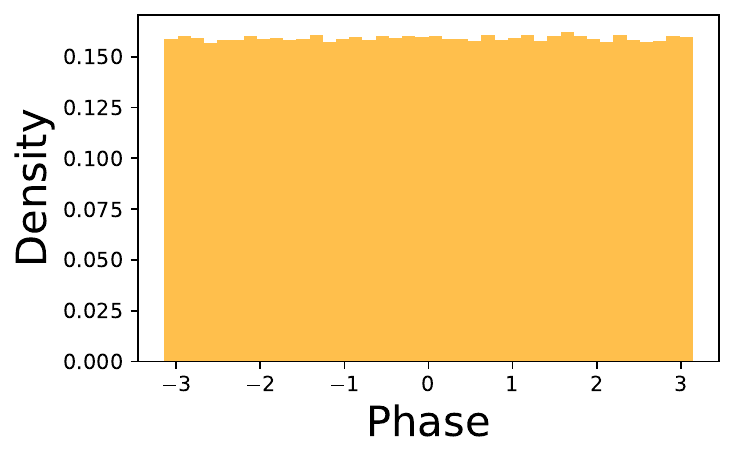}
\includegraphics[width=0.32\linewidth]{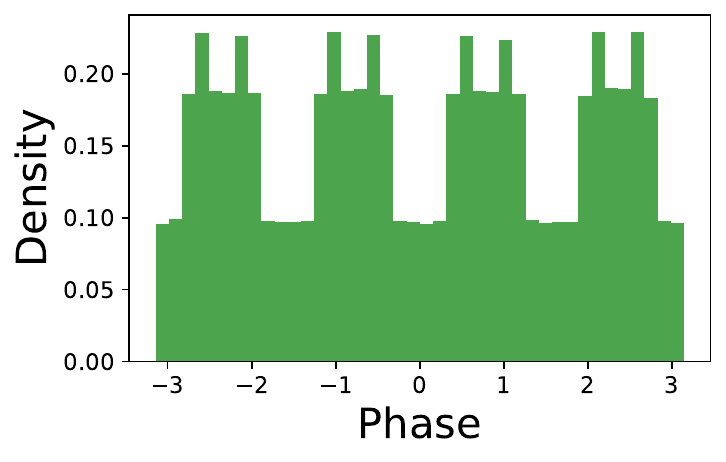}
\includegraphics[width=0.30\linewidth]{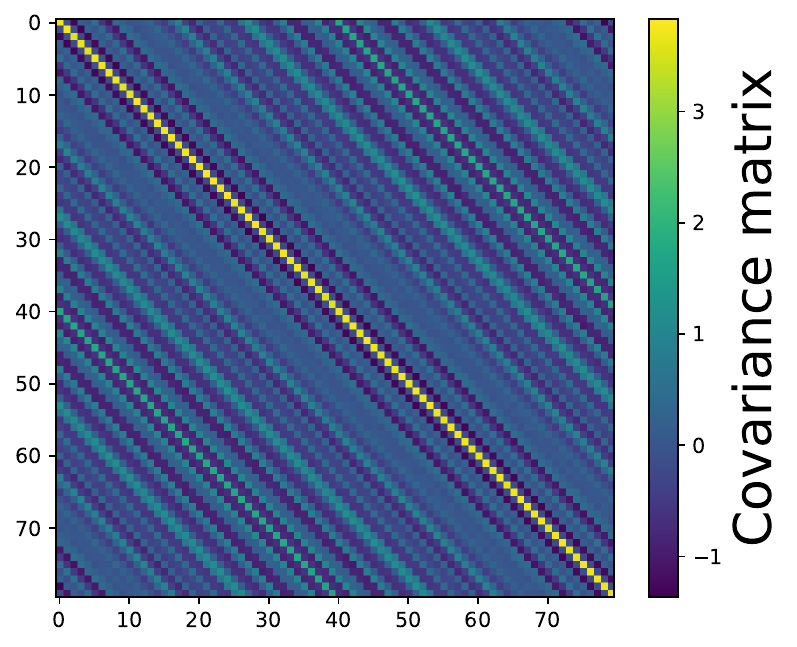}
\caption{\textbf{Illustration of the statistics of inputs distributed according to the Fourier data model.}
(\textbf{Left}) Uniform distribution of the phase of a non-modified frequency. (\textbf{Middle}) Spiked distribution of the phase of the modified frequency $k_0 = 6, \varepsilon = 2.5$.
(\textbf{Right}) Example of a circulant covariance of both baseline $\mathbb{P}_0$ and planted $\mathbb{P}$ distributions, in pixel space.}
\end{figure*}

\begin{remark}\label{app:scalings}
    Consider $x \in \mathbb{R}^N$ distributed according to $\mathbb{P}$ or $\mathbb{P}_0$. Recall that we denote by $F \in \mathbb{C}^{N \times N}$ the non-normalised $\rm{DFT}$ matrix and by $\widehat{F} = \frac{1}{\sqrt{N}} F \in \mathbb{C}^{N \times N}$ the normalised $\rm{DFT}$ matrix. Call $\lambda_k$ the $k$-th eigenvalue of the covariance matrix $\Sigma$ in pixel space. Then, if $X_F = F x$ and $X_{\widehat{F}} = \widehat{F}x$, because of Lemma \ref{app:translation-invariance}, we get that
    \begin{equation*}
        [\operatorname{Cov}(X_F)]_{kl} = \lambda_{k} N\delta_{kl} \quad \text{ and } \quad 
        [\operatorname{Cov}(X_{\widehat{F}})]_{kl} =  \lambda_{k}\delta_{kl}.
    \end{equation*}
    Note that, in the case of the non-normalised $\rm{DFT}$, the amplitudes $\rho_k^F$ are Rayleigh-distributed with parameter $\sigma^F_k = \sqrt{\frac{\lambda_k N}{2}}$. In the normalised case, we have $\sigma^{\widehat{F}}_k = \sqrt{\frac{\lambda_k}{2}}$. Then, thanks to Remark \ref{statisticsray} on the statistics of the Rayleigh distribution, 
    \begin{equation*}
        \mathbb{E}[(\rho_k^F)^2] = \lambda_k N \quad
    \text{ and } \quad \mathbb{E}[(\rho_k^{\hat{F}})^2] = \lambda_k. 
    \end{equation*}
    Moreover, 
    \begin{math}
        \mathbb{E}[\rho_k^F] = \sqrt{\frac{\lambda_k \pi N}{4}} \text{ and }
        \mathbb{E}[\rho_k^{\widehat{F}}] = \sqrt{\frac{\lambda_k \pi}{4}}.
    \end{math}
\end{remark}

The next definition is useful to write explicitly the non-trivial higher-order statistics of the inputs drawn from the Fourier data model in the case of $f(\varphi_{k_0}) = \sin{(\varphi_{k_0})}$.
\begin{definition}[Bessel functions]
    The Bessel function of first kind of order $m$-th is $J_m(z)$ such that
    \begin{equation*}
        \frac{1}{2\pi} \int_{0}^{2\pi} e^{i(m\varphi + z \sin{\varphi})}\mathrm{d}\varphi = J_{-m}(z) = J_m(z)(-1)^m.
    \end{equation*}
\end{definition}

In the next lemma, we show that the third-order cumulants of the inputs drawn from the Fourier data model \eqref{eq:model} are vanishing, and then they equal those of any Gaussian distribution $\mathbb{P}_0$.
\begin{lemma}[Third-order cumulants]\label{app:third}
    Consider a random vector $X \in \mathbb{C}^N$ distributed according to the Fourier data model \eqref{eq:model} and $x = \rm{IDFT}(X)$. Then, for any $t_1, t_2, t_3 \in \{ 0, \dots, N-1\}$, we have $\mathbb{E}[x_{t_1}x_{t_2}x_{t_3}] = 0$.
\end{lemma}
\begin{proof}
    Since 
    \begin{equation*}
        \mathbb{E}[x_{t_1} x_{t_2} x_{t_3}] = \frac{1}{N^3} \sum_{k,l,m=0}^{N-1} \mathbb{E}[X_k X_l X_m] e^{\frac{2\pi i}{N}(t_1 k + t_2 l + t_3 m)},
    \end{equation*}
    it is sufficient to show that $\mathbb{E}[X_k X_l X_m] = 0$, for any $k,l,m \in \{0, \dots, N-1\}$.
    Note that $\mathbb{E}[X_k X_l X_m] = \mathbb{E}[Z_k Z_l Z_m] = 0$ if $k,l,m \notin \{k_0, N-k_0\}$, since $Z$ is Gaussian. The cases in which one or two entries are $X_{k_0}$ or $X_{N-k_0}$ vanish since each entry has mean zero and because of the independence of the entries. To conclude,
    $\mathbb{E}[X_{k_0}^3] = -J_3(3\varepsilon) \mathbb{E}[\rho_{k_0}^3] \mathbb{E}[e^{i3U}] = 0$ by construction of $U$. Same for $\mathbb{E}[X_{N-k_0}^3]$.
\end{proof}

The next result, which allows to compute higher-order moments of a Gaussian distribution, will be used to compute the fourth-order statistics of inputs drawn from the Fourier data model \eqref{eq:model}. It basically tells that any higher-order moment of a zero-mean Gaussian random vector can be written in terms of its covariance matrix.
\begin{lemma}[Isserlis' theorem \citep{isserlis1918formula}]\label{isserlis}
If $x\in\mathbb{R}^N$ is a zero-mean Gaussian random vector, we have
    \begin{equation*}
    \mathbb{E}[x_{1} x_{2} \cdots x_{k}]=\hspace{-2em}
    \sum_{\text{all pairings of } \{1,\dots,k\}}
    \prod_{(a,b) \in \text{ pairing}} \mathbb{E}[x_a x_b].
\end{equation*}
\end{lemma}
Isserlis' theorem implies that, given $x \sim \mathcal{N}(0, \Sigma), \Sigma \in \mathbb{R}^{N\times N}$ and $t_1, t_2, t_3, t_4 \in \{0, \dots, N-1\}$, we can write 
\begin{equation*}
    \mathbb{E}[x_{t_1}x_{t_2}x_{t_3}x_{t_4}] = 
    \Sigma_{t_1t_2}\Sigma_{t_3t_4} + \Sigma_{t_1t_3}\Sigma_{t_2t_4} + \Sigma_{t_1t_4}\Sigma_{t_2t_3}.
\end{equation*}
We now compute the fourth-order moments of inputs sampled from the Fourier data model \eqref{eq:model}. We prove that the fourth-order moments of these inputs projected along a weight vector of unit norm coincide with the ones of a zero-mean Gaussian random vector with covariance $\Sigma$. However, these projections differ in distribution from the Gaussian ones if the weight belongs to the subspace spanned by the DFT phase vectors $u$ and $v$. This is the effect of the phase modification of the mode $k_0$, which results in an additional signal at the level of higher-order interactions for the new data points. Note that choosing $U$ such that $\mathbb{E}[e^{i4U}] = 1$ induces the non-vanishing higher-order correlations to be of order four.
\begin{lemma}[Fourth-order cumulants]\label{app:fourth}
    Consider a random vector $X \in \mathbb{C}^N$ distributed according to the Fourier data model \eqref{eq:model} and $x = \rm{IDFT}(X)$.
    Then, for any $t_1, t_2, t_3, t_4 \in \{0, \dots, N-1\}$, we have that
    \begin{equation*}
        \mathbb{E}[x_{t_1} x_{t_2} x_{t_3} x_{t_4}]\hspace{-0.2em} = 
        \hspace{-0.2em}
        \Sigma_{t_1t_2}\Sigma_{t_3t_4} + \Sigma_{t_1t_3}\Sigma_{t_2t_4} + \Sigma_{t_1t_4}\Sigma_{t_2t_3} + \frac{2}{N^4} J_4(4\varepsilon) \mathbb{E}[\rho_{k_0}^4] \cos{\hspace{-0.2em}\Biggl(
        \hspace{-0.2em}
        \frac{2\pi k_0}{N} (t_1 + t_2 + t_3 + t_4)\hspace{-0.3em}\Biggr)}.
    \end{equation*}
Moreover, for any $w \in \mathbb{R}^N$, it holds that
\begin{equation*}
    \mathbb{E}[(w \cdot x)^4] = 3(w^{\top} \Sigma w)^2 + \frac{2 J_4(4 \varepsilon)}{N^2} \mathbb{E}[\rho_{k_0}^4] \operatorname{Re}(w_{k_0}^4), 
\end{equation*}
where $\rm{Re}$$(w_{k_0})$ is the real part of 
\begin{equation*}
w_{k_0} = \frac{1}{\sqrt{N}} \sum_{t=0}^{N-1} w_t \exp (2\pi ik_0 t/N).
\end{equation*}
Assume now that $\|w\| =1$. Consider the case where $\theta \in (0, 2\pi]$ is such that $w = \cos{\theta}\, u + \sin{\theta}\, v$. Then, $\rm{Re}$$(w_{k_0}^4) = \cos{(4\theta)}/4$. Conversely, if $w \cdot u = w \cdot v = 0$, we get that $\rm{Re}$$(w_{k_0}^4) = 0$.
\end{lemma}
\begin{proof}
    For any $t_1, t_2, t_3, t_4 \in \{ 0, \dots, N-1\}$, we are going to compute 
    \begin{equation*}
        \mathbb{E}[x_{t_1}x_{t_2}x_{t_3}x_{t_4}] = \frac{1}{N^4} \sum_{k,l,m,n=0}^{N-1}\mathbb{E}[X_k X_l X_m X_n]\, e^{\frac{2\pi i}{N} (k t_1 + l t_2 + m t_3 + n t_4)}.
    \end{equation*}
    We first show that this fourth moment can be written as sum of two contributions: one term containing the fourth moments of Gaussian variables and a second term which involves the modified frequencies, i.e.
    \begin{align}\label{split2}
        \mathbb{E}[x_{t_1}x_{t_2}x_{t_3}x_{t_4}] & 
        = \frac{1}{N^4} \hspace{-0.4em}\sum_{k,l,m,n=0}^{N-1}\hspace{-0.7em}
        \mathbb{E}[Z_k Z_l Z_m Z_n]\, e^{\frac{2\pi i}{N} (k t_1 + l t_2 + m t_3 + n t_4)} \\
        & + 
        \frac{1}{N^4} \hspace{-0.8em} \sum_{\substack{k=l=m=n=k_0,\\
        k=l=m=n=N-k_0}}^{N-1} \hspace{-2em}\mathbb{E}[X_k X_l X_m X_n]\, e^{\frac{2\pi i}{N} (k t_1 + l t_2 + m t_3 + n t_4)}.
    \end{align}
    By construction, we know that when $k, l, m, m \notin \{k_0, N-k_0\}$ we have $\mathbb{E}[X_k X_l X_m X_n] = \mathbb{E}[Z_k Z_l Z_m Z_n]$. Also, because of the independence of the entries, we have that $\mathbb{E}[X_{k_0} X_k X_l X_m] = \mathbb{E}[X_{k_0}^2 X_k X_l] = \mathbb{E}[X_{k_0}^3 X_k] = 0$, for $k,l,m \notin \{ k_0, N-k_0\}, k\neq l \neq m$, nor their conjugates. What's more, $\mathbb{E}[X_{k_0}^2 X_{k}^2] = \mathbb{E}[\rho^2_{k_0} X_{k}^2 e^{2 i \varepsilon \sin{\varphi_{k_0}}}] \mathbb{E}[e^{i2U}] = 0$, for $k \notin \{k_0, N-k_0\}$. Also, $\mathbb{E}[X_{k_0} X_{k}^3] = 0$ since $X_{k_0}$ has zero mean. The non-vanishing contributions are
    \begin{small}
     \begin{equation*}
        \mathbb{E}[X_k X_l X_m X_n] \hspace{-0.2em}= \hspace{-0.2em}
        \begin{cases}
        \mathbb{E}[X_{k_0} X_{N\hspace{-0.1em}-\hspace{-0.1em}k_0} X_m X_{N-m}] = \mathbb{E}[\rho_{k_0}^2] \mathbb{E}[\rho_{m}^2]  \hspace{-0.8em} & \text{if } k = k_0, l = N\hspace{-0.1em}-\hspace{-0.1em}k_0, m \neq \{k, N\hspace{-0.1em}-\hspace{-0.1em}k\}, n =\hspace{-0.1em} \hspace{-0.2em} N-\hspace{-0.2em}m,\\[2pt]
        \mathbb{E}[X_{k_0}^2 X_{N-k_0}^2] = \mathbb{E}[X_{k_0}^2 \overline{X}_{k_0}^2] = \mathbb{E}[\rho_{k_0}^4]  &\text{if } k = l = k_0, m = n = N-k_0, \\[2pt]
        \mathbb{E}[X_{N-k_0}^4] = 
        \mathbb{E}[\overline{X}_{k_0}^4] = J_4(4\varepsilon)\mathbb{E}[\rho_{k_0}^4] & \text{if } k= l = m = n = N-k_0,\\[2pt]
        \mathbb{E}[X_{k_0}^4] = 
        \mathbb{E}[X_{N -k_0}^4]
         = J_4(4\varepsilon)\mathbb{E}[\rho_{k_0}^4] & \text{if } k= l = m = n = k_0.
    \end{cases}
    \end{equation*}
    \end{small}
    Note that the last two cases don't vanish because, by construction of $U$, we have that $E[e^{i4U}]=1$.
    We know that $a = \text{Re}(Z_k) \sim \mathcal{N}(0, \mathbb{E}[\rho_{k}^2]/2)$ and $b = \text{Im}(Z_k) \sim \mathcal{N}(0, \mathbb{E}[\rho_{k}^2]/2)$ for $k \neq 0, N/2$. Since $\rho_{k_0}^2 = a^2 + b^2$, it follows that 
    \begin{align*}
        \mathbb{E}[\rho_{k_0}^4] = \mathbb{E}[a^4] + \mathbb{E}[b^2] + 2 \mathbb{E}[a^2]\mathbb{E}[b^2] = 
        3 \frac{\mathbb{E}[\rho_{k_0}^2]^2}{2} + \frac{\mathbb{E}[\rho_{k_0}^2]^2}{2}
        = 2 \mathbb{E}[\rho_{k_0}^2]^2.
    \end{align*}
    Now, recall that thanks to the Isserlis' Theorem \ref{isserlis} for Gaussian random variables we have 
    \begin{equation*}
        \mathbb{E}[Z_k Z_l Z_m Z_n] = \mathbb{E}[\rho_k^2] \mathbb{E}[\rho_m^2] (\delta_{k,N-l}\delta_{m,N-n} + \delta_{k,N-m}\delta_{l,N-n} + \delta_{k,N-n}\delta_{l,N-m}), 
    \end{equation*}
    which is non vanishing if and only if the indices can be partitioned in two pairs of conjugates. Then, $\mathbb{E}[X_k X_l X_m X_n] = \mathbb{E}[Z_k Z_l Z_m Z_n]$ for any $(k,l,m,m) \neq (k_0, k_0, k_0, k_0)$ and $(k,l,m,m) \neq (N-k_0, N-k_0,N-k_0, N-k_0)$, and \cref{split2} follows.
    The first addendum of \cref{split2}, by definition of DFT, is simply equal to $\mathbb{E}[z_{t_1}z_{t_2}z_{t_3}z_{t_4}]$.
    Now, thanks again to the Isserlis' theorem for $z$, we have that 
    \begin{math}
        \mathbb{E}[z_{t_1}z_{t_2}z_{t_3}z_{t_4}] = \Sigma_{t_1t_2}\Sigma_{t_3t_4} + \Sigma_{t_1t_3}\Sigma_{t_2t_4} + \Sigma_{t_1t_4}\Sigma_{t_2t_3}.
    \end{math}
    It remains to compute the second addendum in \cref{split2}, which turns out to be
    \begin{equation*}
        \frac{J_4(4\varepsilon) \mathbb{E}[\rho_{k_0}^4]}{N^4} \hspace{-0.2em}\Biggl[ e^{\frac{2 \pi i k_0}{N}(t_1+t_2+t_3+t_4)} \hspace{-0.3em}+\hspace{-0.2em} e^{\frac{2 \pi i (N-k_0)}{N}(t_1+t_2+t_3+t_4)}   \Biggr]\hspace{-0.4em} = \hspace{-0.3em}
        \frac{2 J_4(4\varepsilon) \mathbb{E}[\rho_{k_0}^4]}{N^4} \hspace{-0.07em}\cos{\hspace{-0.2em}\Biggl(\hspace{-0.3em}\frac{2\pi k_0(t_1\hspace{-0.3em}+t_2\hspace{-0.3em}+t_3\hspace{-0.3em}+t_4\hspace{-0.1em}) }{N}\hspace{-0.3em}\Biggr)}.
    \end{equation*}
    
    For the second part of the thesis, consider $w \in \mathbb{R}^N$ and hence
    \begin{align*}
        \mathbb{E}[(w \cdot x)^4] & = \sum_{k,l,m,n=0}^{N-1} w_k w_l w_m w_n \mathbb{E}[x_k x_l x_m x_n]\\ 
        & = 3(w^{\top} \Sigma w)^2 + 
        \frac{2 J_4(4\varepsilon) \mathbb{E}[\rho_{k_0}^4]}{N^4} \sum_{k,l,m,n=0}^{N-1} w_k w_l w_m w_n \cos{\Biggl(\frac{2\pi k_0(t_1+t_2+t_3+t_4) }{N}\Biggr)}.
    \end{align*}
    Then, using the fact that
    \begin{math}
        \cos{\alpha} = (e^{i\alpha} + e^{-i\alpha})/2, 
    \end{math}
    we can write
    \begin{equation*}
        2 \hspace{-0.9em}\sum_{k,l,m,n=0}^{N-1} \hspace{-0.9em}
        w_k w_l w_m w_n \cos{\hspace{-0.3em}\Biggl(\hspace{-0.3em}\frac{2\pi k_0(t_1\hspace{-0.3em}+\hspace{-0.2em}t_2\hspace{-0.3em}+\hspace{-0.3em}t_3\hspace{-0.3em}+\hspace{-0.3em}t_4)}{N}\hspace{-0.3em}\Biggr)} 
        \hspace{-0.4em}=\hspace{-1.2em}
        \sum_{k,l,m,n=0}^{N-1} \hspace{-0.8em} w_k w_l w_m w_n\Biggl(\hspace{-0.3em} e^{\frac{2\pi i k_0}{N}(k+l+m+n)} \hspace{-0.3em}+ \hspace{-0.2em}e^{-\frac{2\pi i k_0}{N}(k+l+m+n)}\hspace{-0.3em}\Biggr).
    \end{equation*}
    Since
    \begin{equation*}
         \sum_{k,l,m,n=0}^{N-1} w_k w_l w_m w_n e^{\frac{2\pi i k_0}{N}(k+l+m+n)} = 
        \Biggl( \underbrace{\sum_{k=0}^{N-1} w_k e^{\frac{2\pi i k_0 k}{N}}}_{w_{k_0} (\sqrt{N})}\Biggr)^4.
    \end{equation*}
    By defining 
    \begin{math}
        w_{k_0} = 1/\sqrt{N} \sum_{k=0}^{N-1} w_k e^{\frac{2\pi i k_0 k}{N}}, 
    \end{math}
    we can conclude that 
    \begin{equation*}
        \mathbb{E}[(w \cdot x)^4] = 3(w^{\top}\Sigma w)^2 + \frac{J_4(4 \varepsilon)\mathbb{E}[\rho_{k_0}^4]}{N^2} (w_{k_0}^4 + \overline{w}_{k_0}^4) = 
        3(w^{\top} \Sigma w)^2 + \frac{2 J_4(4 \varepsilon)\mathbb{E}[\rho_{k_0}^4]}{N^2} \mathbb{E}[\rho_{k_0}^4] \text{Re}(w_{k_0}^4).
    \end{equation*}
    What's more, consider $w$ such that $\|w\|=1$ and $\theta \in (0,2\pi]$ with $w = \cos{\theta}\, u + \sin{\theta}\, v$, for $u, v$ the DFT phase vectors corresponding to the frequency $k_0$. Then, 
    \begin{math}
        w_{k_0}^4 = (\frac{w}{\sqrt{2}}(u + iv))^4 = \frac{1}{4} (\cos{\theta} + i \sin{\theta})^4
    \end{math}
    and 
    $\text{Re}(w_{k_0}^4) = \frac{1}{4}\cos{4\theta}$.
    On the other hand, consider $w \cdot u = w \cdot v = 0$. 
    Since 
    \begin{align*}
    w_{k_0} &= \frac{1}{\sqrt{N}} \sum_{t=0}^{N-1} w_t \exp\left(- \frac{2\pi i k_0 t}{N}\right) 
    = \frac{1}{\sqrt{N}} \sum_{t=0}^{N-1} w_t \left[ \cos\left(\frac{2\pi k_0 t}{N}\right) - i \sin\left(\frac{2\pi k_0 t}{N}\right) \right] \\
    &= \frac{1}{\sqrt{N}} \left[ \sum_{t=0}^{N-1} w_t \left( \sqrt{\frac{N}{2}} u_t \right) - i \sum_{t=0}^{N-1} w_t \left( \sqrt{\frac{N}{2}} v_t \right) \right]
    = \frac{1}{\sqrt{2}} \left( \sum_{t=0}^{N-1} w_t u_t - i \sum_{t=0}^{N-1} w_t v_t \right) \\
    &= \frac{1}{\sqrt{2}} \left( (w \cdot u) - i (w \cdot v) \right) = 0,
\end{align*}
we have the thesis.
\end{proof}

In the next lemma, we compute the expression of the inputs sampled from the Fourier data model in pixel space.
\begin{lemma}\label{app:pixelspace}
    In pixel space, the inputs $(x^{\mu})_{\mu}^n \subseteq \mathbb{R}^N$ sampled from the Fourier data model are such that, for $k = 0, \dots, N-1$, 
    \begin{equation*}
        x_k^{\mu} = z_k^{\mu} + \frac{2\rho^{\mu}_{k_0}}{N} \Biggl[ \cos{\biggl( \frac{2\pi k k_0}{N} + \varphi_{k_0}^{\mu} + \varepsilon f{(\varphi_{k_0}^{\mu}) + U^{\mu}} \biggr)} \Biggr].
    \end{equation*}
\end{lemma}
\begin{proof}
    We know that $X_s = Z_s $ if $s \neq k_0, N-k_0$, and additionally $X_{k_0} = \rho_{k_0}e^{i(\varphi_{k_0} + \varepsilon \sin{\varphi_{k_0}}  +U)}$ and $X_{N-k_0} =\bar{X}_{k_0}$. We get that
    \begin{align*}
        x_k = \frac{1}{N}\sum_{s=0}^{N-1}X_s e^{2\pi i s k/N} = \frac{1}{N} \hspace{-0.4em}\sum_{s\neq k_0, N-k_0}\hspace{-0.6em} Z_s e^{2\pi i sk/N} + \frac{1}{N}\biggl( X_{k_0}e^{2\pi in k_0 /N} + X_{N-k_0} e^{2\pi i(N-k_0)/N} \biggr).
    \end{align*}
    Since 
    \begin{equation*}
         \frac{1}{N} \hspace{-0.5em}\sum_{s\neq k_0, N-k_0}\hspace{-0.6em} Z_s e^{2\pi i sk/N} = z_k - \frac{1}{N} \biggl( Z_{k_0}e^{2\pi i s k_0/N} + Z_{N-k_0}e^{2\pi i s(N-k_0)/N} \biggr),
    \end{equation*}
    the thesis follows.
\end{proof}

\section{Hardness of learning the phase}\label{app:hardnessoflearning}
In this section we prove that online SGD requires a long time - equivalently, a large number of samples - to distinguish inputs drawn from Gaussian white noise, $\mathbb{P}_0 = \mathcal{N}(0, \mathds{1})$, and isotropic inputs drawn from the Fourier data model \eqref{eq:model}. In this setting, both distributions have identity covariance matrix, which implies that SGD cannot exploit information coming from low-order statistics of the data to perform classification.

We start by computing the coefficients $c_{ij}^{\ell}$ defined in \cref{eq:defcoefficients}. We will show that they coincide with the coefficients of the Hermite expansion of the likelihood ratio between the distributions $\mathbb{P}$ and $\mathbb{P}_0$.
These coefficients determine the low-order terms in the expansion of the population classification loss \eqref{def:population}.

\begin{lemma}[Coefficients of the likelihood ratio]\label{app:coefficients}
    Consider the following coefficients, with expectation of the inputs $x$ taken with respect to $\mathbb{P}$:
    \begin{equation}\label{eq:defcoefficients}
        c_{ij}^{\ell} = \mathbb{E}_{\mathbb{P}}\biggl[h_i\biggl(\frac{v \cdot x}{\sigma_C}\biggr) h_j\biggl(\frac{u \cdot x}{\sigma_B}\biggr)\biggr],
    \end{equation}
    where $h_i, h_j$ are the $i$-th and $j$-th probabilist Hermite polynomials, for $i+j \leq 4$. We set $\sigma_C = \sqrt{v^{\top} \Sigma v}$ and $\sigma_B = \sqrt{u^{\top} \Sigma u}$.
    Then, we have that $c_{ij}^{\ell} = 0$ when $i + j < 4$. Moreover, $c_{13}^{\ell} = c_{31}^{\ell} = 0$,
    \begin{equation*}
        c_{22}^{\ell} = -\frac{1}{2} c_{40}^{\ell} 
        \qquad 
        \text{ and } 
        \qquad
        c_{40}^{\ell} = c_{04}^{\ell} = \frac{J_4(4\varepsilon)}{\lambda_{k_0}^2 N^2}\mathbb{E}[\rho_{k_0}^4].
    \end{equation*}
\end{lemma}
\begin{proof}
Recall that the DFT vectors $u$ and $v$ are eigenvectors of $\Sigma$ for the same eigenvalue $\lambda_{k_0}$. Then, we obtain that $\sigma_C = \sigma_B =\sqrt{\lambda_{k_0}}$. 
\begin{itemize}
    \item Cases $i + j = 4$.
    For $i=4, j=0$, we have that
    \begin{equation*}
        c_{40}^{\ell} = \mathbb{E}\biggl[h_4\biggl(\frac{v \cdot x}{\sigma_C}\biggr)^4\biggr] = 
        \mathbb{E}\biggl[\frac{(v \cdot x)^4}{\sigma_C^4}\biggr] - 6 \, \mathbb{E}\biggl[\frac{(v \cdot x)^2}{\sigma_{C}^2}\biggr] + 3.
    \end{equation*}
    Since $\mathbb{E}[(v\cdot x)^2] = \sigma_C^2$ and thanks to Lemma \ref{app:fourth}, we can choose $\theta=\pi/2$ and get 
    \begin{equation*}
        c_{40}^{\ell} = \frac{J_4(4\varepsilon)}{\lambda_{k_0}^2 N^2}\mathbb{E}[\rho_{k_0}^4].
    \end{equation*}
    For $i = 0, j = 4$, we choose $\theta = 0$ and by the same reasoning as before we get $c_{04}^{\ell} = c_{40}^{\ell}$.
    Similarly, 
    \begin{align*}
        c_{22}^{\ell} & = \mathbb{E}\biggl[h_2\biggl( \frac{v \cdot x}{\sigma_C} \biggl) h_2\biggl( \frac{u \cdot x}{\sigma_B} \biggr)\biggr] = 
        \frac{1}{\lambda_{k_0}^2} \mathbb{E}[(v \cdot x)^2 (u \cdot x)^2] - \frac{1}{\lambda_{k_0}}\Big[\mathbb{E}[(v \cdot x)^2] + \mathbb{E}[(u \cdot x)^2]\Big] + 1 \\
        & = \frac{1}{\lambda_{k_0}^2}\mathbb{E}[(v \cdot x)^2 (u \cdot x)^2] - 1. 
    \end{align*}
    By exploiting the orthonormality of $u$ and $v$ and Lemma \ref{app:fourth}, we have 
    \begin{equation*}
        \mathbb{E}[(v \cdot x)^2 (u \cdot x)^2] = \sum_{k,l,m,n=0}^{N-1} u_k u_l v_m v_n \mathbb{E}[x_k x_l x_m x_n] = \lambda_{k_0}^2 + T_4,   
    \end{equation*}
    where
    \begin{equation*}
        T_4 = \frac{2}{N^4} J_4(4 \varepsilon) \mathbb{E}[\rho_{k_0}^4] \sum_{k,l,m,n=0}^{N-1} u_k u_l v_m v_n \cos{\frac{2\pi k_0}{N} (k+l+n+m)}.
    \end{equation*}
    Define now the angle \mbox{$\theta = 2 \pi k_0 /N$}, and note that 
    \begin{equation*}
        \cos{\theta r} = \frac{e^{i \theta r} + e^{-i \theta r}}{2}, \;\;
        \sin{\theta r} = \frac{e^{i \theta r} - e^{-i \theta r}}{2i}.
    \end{equation*}
    Then, exploiting the definition of $u$ and $v$, we can write $T_4$ as
    \begin{small}
    \begin{equation*}
        T_4 = -\frac{2 J_4(4 \varepsilon) \mathbb{E}[\rho_{k_0}^4]}{32 N^4} \hspace{-1em}
        \sum_{\substack{s_1, s_2 = \pm 1\\ r_1, r_2 = \pm 1 \\ t = \pm 1}} \hspace{-0.9em}r_1 r_2 \Biggl( \sum_{k = 0}^{N-1}e^{i \theta(s_1 + t)k} \Biggr)\hspace{-0.3em}
        \Biggl( \sum_{l = 0}^{N-1}e^{i \theta(s_2 + t)l} \Biggr)\hspace{-0.3em}
        \Biggl( \sum_{n = 0}^{N-1}e^{i \theta(r_1 + t)n} \Biggr)\hspace{-0.3em}
        \Biggl( \sum_{m = 0}^{N-1}e^{i \theta(r_2 + t)m} \Biggr)\hspace{-0.3em} \frac{4}{N^2}.
    \end{equation*}
    \end{small}
    Since, for any $q\in \mathbb{Z}$, we have that
    \begin{equation*}
        \sum_{q' = 0}^{N-1} e^{i\theta q q'} = 
        \begin{cases}
        N  \quad\quad \text{if } k_0 q \equiv_{N} 0,\\
        0 \quad\quad \text{otherwise},
    \end{cases}
    \end{equation*}
    the factors in $T_4$ are non vanishing if and only if $s_1 + t = s_2 + t = r_1 + t = r_2 + t = 0$, corresponding to the configurations 
    \begin{equation*}
        (s_1, s_2, r_1, r_2, t) = (1, 1, 1, 1, -1) \text{ and }(s_1, s_2, r_1, r_2, t) = (-1, -1, -1, -1, 1).
    \end{equation*}
    Each of the two non vanishing factors produces a contributions of $N^2$. Hence, we conclude that 
    \begin{math}
        T_4 = - J_4(4\varepsilon)\mathbb{E}[\rho_{k_0}^4]/(2N^2)
    \end{math}
    and then $c_{22}^{\ell} = T_4/\lambda_{k_0}^2 = - J_4(4\varepsilon)\mathbb{E}[\rho_{k_0}^4]/(2N^2 \lambda_{k_0}^2).$
    With a similar computation, we can prove that $c_{31}^{\ell} = c_{13}^{\ell} = 0$.
    Indeed, we have 
    \begin{align*}
        c_{31}^{\ell} = \frac{1}{\lambda_{k_0}^2}\mathbb{E}[(v \cdot x)^3 (u \cdot x)] = 3 (u^{\top}\Sigma u)(u^{\top} \Sigma v) + T_4', 
    \end{align*}
    where 
    \begin{align*}
        T_4' & = \frac{2}{N^4} J_4(4 \varepsilon) \mathbb{E}[\rho_{k_0}^4] \sum_{k,l,m,m=0}^{N-1} u_k u_l u_m v_n \cos{\frac{2\pi k_0}{N} (k+l+n+m)}\\
        &= \frac{J_4(4 \varepsilon) \mathbb{E}[\rho_{k_0}^4]}{2 N^6} \hspace{-1.3em}
        \sum_{\substack{s_1, s_2, s_3 = \pm 1 \\ r = \pm 1, t = \pm 1}}\hspace{-0.9em} r \Biggl( \sum_{k = 0}^{N-1}e^{i \theta(s_1 + t)k} \Biggr)\hspace{-0.3em}
        \Biggl( \sum_{l = 0}^{N-1}e^{i \theta(s_2 + t)l} \Biggr)\hspace{-0.3em}
        \Biggl( \sum_{m = 0}^{N-1}e^{i \theta(s_3 + t)m} \Biggr)\hspace{-0.3em}
        \Biggl( \sum_{n = 0}^{N-1}e^{i \theta(r + t)n} \Biggr).
    \end{align*}
    Since the only two non vanishing configurations are $(s_1, s_2, s_3, r, t) = (1,1,1,1,-1)$ and $(s_1, s_2, s_3, r, t) = (-1,-1,-1,-1,1)$, and each one contributes with $N^2$, it turns out that $T_4' = 0$.
    By recalling that $u$ and $v$ are orthogonal, we conclude that $c_{31}^{\ell} = 0$.
    \item Cases $i+j=3, i + j = 1$ or $i,j=0$. Because of Lemma \ref{app:third}, it is clear that $c_{30}^{\ell} = c_{03}^{\ell} = c_{21}^{\ell} = c_{12}^{\ell} = 0 $. Since the inputs have zero mean,
    $c_{10}^{\ell} = c_{01}^{\ell}=0$. Clearly, the null term is $c_{00}^{\ell} = 1$. 
    \item Cases $i+j = 2$. Since $\Sigma v = \lambda_{k_0} v$ and $\Sigma u = \lambda_{k_0} u$, we have that
    \begin{align*}
        c_{20}^{\ell} =\frac{1}{\sigma_C^2} \mathbb{E}[(v \cdot x)^2] -1 = \frac{1}{\lambda_{k_0}} v^{\top}\Sigma v -1 = 0, \\
        c_{02}^{\ell} =\frac{1}{\sigma_B^2} \mathbb{E}[(u \cdot x)^2] -1 = \frac{1}{\lambda_{k_0}} u^{\top}\Sigma v -1 = 0,\\
        c_{11}^{\ell} = \mathbb{E}\biggl[\biggl(\frac{v \cdot x}{\sigma_C}\biggr)\biggl(\frac{u \cdot x}{\sigma_B}\biggr)\biggr] = \frac{1}{\lambda_{k_0}} v^{\top}\Sigma u = 0.
    \end{align*}
\end{itemize}
\end{proof}

We now prove the main theorem of this section (Theorem \ref{thm:hard} in the main text), which provides the sample complexities to recover the phase information in case of isotropic inputs.
\begin{customthm}{3}
    Sample $(x^{\mu})_{\mu =1}^n \subseteq \mathbb{R}^N$ from the Fourier data model \eqref{eq:model} with $\Sigma = \mathds{1}$. Apply online $\mathrm{SGD}$ \eqref{algo} to the correlation loss \eqref{correlationloss} and define the overlaps \mbox{$\alpha_{u,n} = w_n \cdot u$} and $\alpha_{v,n} = w_n \cdot v$, where $(u, v)$ are the $\rm{DFT}$ phase vectors. Then, if 
    \begin{footnotesize}$1/(N^2 \log^2{N}) \hspace{-0.3em}\ll \delta_N \hspace{-0.3em} \ll 1/(N \log{N})$
    \end{footnotesize} 
    and {\boldmath $n \gg N^3 \log^2 N$}, there is \textbf{weak recovery} for $(u, v)$ i.e., for some $\eta >0$,
    \begin{small}
    \begin{equation*}
        \lim_{N \to \infty} \hspace{-0.3em} P( \,\lvert \alpha_{u,n} \rvert \geq \eta) = 1 \quad \text{and} \quad
        \lim_{N \to \infty} \hspace{-0.3em} P( \,\lvert \alpha_{v,n} \rvert \geq \eta) = 1.
    \end{equation*}
    \end{small}
    Conversely, when 
    \begin{footnotesize}$
    \delta_N \ll 1/(N \log{N})$\end{footnotesize}
    and $\boldsymbol{n \ll N^3}$, there is \textbf{no recovery} for $(u,v)$ i.e., in probability,
    \begin{equation*}
        \lvert \alpha_{u,n} \rvert, \lvert \alpha_{v,n} \rvert \xrightarrow[N \to +\infty]{}0.
    \end{equation*}
\end{customthm}
\begin{proof}
    By taking the expectation in the correlation loss \eqref{correlationloss} with respect to the data distribution, the population correlation loss reads as  
    \begin{equation}\label{def:populationloss}
        L(w) = 1 - \frac{1}{2}\mathbb{E}_{\mathbb{P}}[\sigma(w \cdot x)] + \frac{1}{2} \mathbb{E}_{\mathbb{P}_0}[\sigma(w \cdot z)].
    \end{equation}
    Since $w$ has unitary norm, $w \cdot z \sim \mathcal{N}(0,1)$ and then $\mathbb{E}_{\mathbb{P}_0}[\sigma(w \cdot x)] = c_0^{\sigma}$, where $c_0^{\sigma}$ is the zeroth - order Hermite coefficient for the activation function $\sigma$. Moreover, note that the likelihood ratio 
    \begin{math}
        \ell = \text{d} \mathbb{P} \backslash \text{d}\mathbb{P}_0
    \end{math}
    depends only on the projections $u \cdot x$ and $v \cdot x$, by construction of the Fourier data model. Indeed,  the inputs sampled from the baseline distribution $\mathbb{P}_0$ are Gaussian distributed and translation-invariant, which implies that their (orthogonal) DFT coefficients $X_k$ are independent. Then, to obtain the new inputs of the Fourier data model, we perform the modification of the phase of the mode $k_0$. Hence, if the densities in Fourier space are $p$ and $p_0$, we get the following factorization:
    \begin{equation*}
        \ell(x) = \frac{\text{d}\mathbb{P}}{\text{d}\mathbb{P}_0}(x) =
        \frac{\text{d}p}{\text{d} p_0}(X) = \frac{\text{d}p(X_{k_0}, X_{N-k_0})}{\text{d}p_0(X_{k_0}, X_{N-k_0})} \prod_{k \neq k_0, N-k_0}\frac{\text{d}p(X_k, X_{N-k})}{\text{d}p_0(X_k, X_{N-k})} = 
        \frac{\text{d}p(X_{k_0}, X_{N-k_0})}{\text{d}p_0(X_{k_0}, X_{N-k_0})}.
    \end{equation*}
    Also, we have that
     \begin{math}
        X_{k_0} = \sum_{n = 0}^{N-1}x_n/\sqrt{N} \,e^{-2\pi k_0n/N} = \frac{1}{\sqrt{2}}(x \cdot u - i\, x \cdot v),
    \end{math}
    where $u$ and $v$ are the DFT phase vectors. Since $X_{N-k_0} = \overline{X}_{k_0}$, the likelihood ratio $\ell$ depends only on the projections of the inputs along the DFT phase vectors. 
    In what follows, with a slight abuse of notation, we will write $\ell = \ell(x) = \ell(x \cdot v, x\cdot u)$. Therefore, we have that
    \begin{equation}\label{eq:useoflike}
        \mathbb{E}_{\mathbb{P}}[\sigma(w \cdot x)] = \mathbb{E}_{\mathbb{P}_0}[\sigma(w \cdot x) \ell(v \cdot x, u \cdot x)].
    \end{equation}
    We can now expand $\ell$ and $\sigma$ in Hermite polynomials and get that
    \begin{equation*}
        \ell(v \cdot x, u \cdot x) = \sum_{i, j = 0}^{+\infty} \frac{c_{ij}^{\ell}}{i! j!} h_i(v \cdot x) h_j(u \cdot x) \quad\text{ and }\quad
        \sigma(w \cdot x) = \sum_{k=0}^{+\infty}\frac{c_k^{\sigma}}{k!}h_k(w\cdot x),
    \end{equation*}
    with coefficients
    \begin{equation*}
        c_{ij}^{\ell} = \mathbb{E}_{\mathcal{N}(0,\mathds{1}_2)}[\ell(Z_1, Z_2) h_{i}(Z_1) h_j(Z_2)] \quad \text{ and } \quad 
        c_k^{\sigma} = \mathbb{E}[\sigma(Z) h_k(Z)].
    \end{equation*}
    We now want to write $c_{ij}^{\ell}$ in a different form, such that it is possible to explicitly compute them.
    To do so, define the projection map $T : \mathbb{R}^N \to \mathbb{R}^2$ such as $T(x) = (Z_1, Z_2) = [(x \cdot v)/\sigma_C, (x \cdot u)/\sigma_B]$. Under $\mathbb{P}_0$, we get that \mbox{$T(x) \sim \mathcal{N}(0, \mathds{1}_2)$}, since $u$ and $v$ are orthonormal.
    Then, $T_{\#}\mathbb{P}_0 = \mathcal{N}(0, \mathds{1}_2)$ and 
    \begin{align*}
        c_{ij}^{\ell} 
        & = \mathbb{E}_{\mathcal{N}(0,\mathds{1}_2)}[\ell(Z_1, Z_2) h_{i}(Z_1) h_j(Z_2)] = 
        \mathbb{E}_{T_\# \mathbb{P}_0}[\ell(Z_1, Z_2) h_{i}(Z_1) h_j(Z_2)]
        \\[2pt]
        & = \mathbb{E}_{\mathbb{P}_0}\Big[\ell ( T(x) ) \, h_i\Big( \frac{x \cdot u}{\sigma_C} \Big) h_j \Big(\frac{x \cdot u}{\sigma_B}\Big)\Big] = 
        \mathbb{E}_{\mathbb{P}}\Big[h_{i}\Big(\frac{x \cdot v}{\sigma_C}\Big) h_j\Big(\frac{x \cdot u}{\sigma_B}\Big)\Big].
    \end{align*}
    Note that the coefficients $c_{i,j}^{\ell}$ for $i+j \leq 4$ have been computed in Lemma \ref{app:coefficients}.
    From \cref{eq:useoflike}, it holds that
    \begin{equation}\label{eq:decomposition}
        \mathbb{E}_{\mathbb{P}}[\sigma(w \cdot x)] = \sum_{i,j,k=0}^{+\infty} \frac{c_{ij}^\ell c_k^{\sigma}}{i!j!k!}\mathbb{E}_{\mathbb{P}_0}[h_i(v \cdot x)h_j(u \cdot x)h_k(w \cdot x)].
    \end{equation}
    We can now decompose $w$ such that $w = \alpha_v v + \alpha_u u + w_{\perp}$, where $\alpha_v = v \cdot w$ and $\alpha_u = u \cdot w$. 
    Recall that for independent normal Gaussian variables $Z_1, Z_2$ and coefficients $a, b >0$ it holds that
    \begin{equation}\label{eq:sumhermite}
        h_k(a Z_1 + b Z_2) = \sum_{s=0}^k \binom{k}{s} a^s b^{k-s} h_s(Z_1) h_{k-s}(Z_2),
    \end{equation}
    for any order $k \in \mathbb{N}$.
    Since the projections of the inputs along $u, v$ and $w_{\perp}$ are independent, we can apply formula \cref{eq:sumhermite} twice and we get that
    \begin{equation*}
        h_k(w \cdot x) = \sum_{a + b+ c = k} \frac{k!}{a! b! c!} \alpha_v^a \alpha_u^b h_a(x \cdot v)h_b(x \cdot u) h_c(x \cdot w_{\perp}).
    \end{equation*}
    By plugging $h_k(w \cdot x)$ into \cref{eq:decomposition}, we finally obtain that
    \begin{align*}
        \mathbb{E}_{\mathbb{P}}[\sigma(w \cdot x)] & = \hspace{-0.8em}\sum_{i,j,k=0}^{+\infty} \sum_{a+b+c = k} \hspace{-0.5em}\frac{c_{ij}^{\ell}c_k^{\sigma}k!}{i!j!k!a!b!c!} \underbrace{\mathbb{E}_{\mathbb{P}_0}[h_i(v \cdot x)h_a(v \cdot x)]}_{i! \delta_{ia}}
        \underbrace{\mathbb{E}_{\mathbb{P}_0}[h_j(u \cdot x)h_b(u \cdot x)]}_{j!\delta_{jb}}
        \underbrace{\mathbb{E}_{\mathbb{P}_0}[h_c(w_{\perp} \cdot x)]}_{\delta_{c0}} \\
        & = \sum_{i + j = k}^{+\infty}\frac{c_{ij}^{\ell}c^{\sigma}_k}{i!j!} \alpha_v^i \alpha_u^j = 
        \sum_{i, j = 0}^{+\infty}\frac{c_{ij}^{\ell}c^{\sigma}_{i+j}}{i!j!} \alpha_v^i \alpha_u^j,
    \end{align*}
    where we have used the orthogonality property of the Hermite polynomials (Lemma \ref{app:ort}).
    Then, thanks to Lemma \ref{app:coefficients} (note that here $\lambda_{k_0} = 1$) we are able to compute the coefficients of the likelihood ratio $c_{ij}^{\ell}$.
    We get that
    \begin{align*}
        L(w) & = 1 -\frac{1}{2} \biggl[ c_1^{\sigma}\overbrace{(c_{10}^{\ell} \alpha_v + c_{01}^{\ell} \alpha_u )}^{0} + c_2^{\sigma}\overbrace{(c_{20}^{\ell} \alpha_v^2 /2  + c_{11}^{\ell}\alpha_v \alpha_u + c_{02}^{\ell} \alpha_u^2/2)}^{0}\\
        & + c_3^{\sigma}\overbrace{(c_{30}^{\ell}\alpha_v^3/3! + c_{21}^{\ell}\alpha_v^2 \alpha_u/2 + c_{12}^{\ell}\alpha_v \alpha_u^2/2 + c_{03}^{\ell}\alpha_u^3/3!)}^{0}\\[10pt]
        & + c_4^{\sigma}\underbrace{( c_{40}\alpha_v^4/4! + c_{31}^{\ell}\alpha_v^3 \alpha_u/3! + c_{22}^{\ell}\alpha_v^2 \alpha_u^2/4 + c_{13}^{\ell}\alpha_v \alpha_u^3/3! + c_{04}\alpha_u^4/4!)}_{J_4(4\varepsilon)(\alpha_v^4 + \alpha_u^4)/4! \, - \,\frac{J_4(4 \varepsilon)}{2}\alpha_{v}^2 \alpha_u^2/4} + \text{others}\biggr]\\
        & = 1 - \frac{c_4^{\sigma}J_4(4\varepsilon)}{48}(\alpha_v^4 + \alpha_u^4) + c_4^{\sigma}\frac{J_4(4 \varepsilon)}{16} \alpha_{v}^2 \alpha_u^2 + \text{others.}
    \end{align*}
Since the first non-vanishing terms in the expansion of the population correlation loss are of order 4 in the sum of the degrees of the overlaps $\alpha_u$ and $\alpha_v$, we can apply the main results from \citep{benarous2021online}, with information exponent $k = 4$, to get the thesis. Note that the population correlation loss is not simply monotonic in the overlaps, because it depends on both the DFT basis vectors $u$ and $v$, making the loss landscape more complex. 
To deal with that, it is possible to perform an analysis similar to the one in \citep{bardone2024sliding}.
\end{proof}

\begin{figure}
\centering
\includegraphics[width=0.5\linewidth]{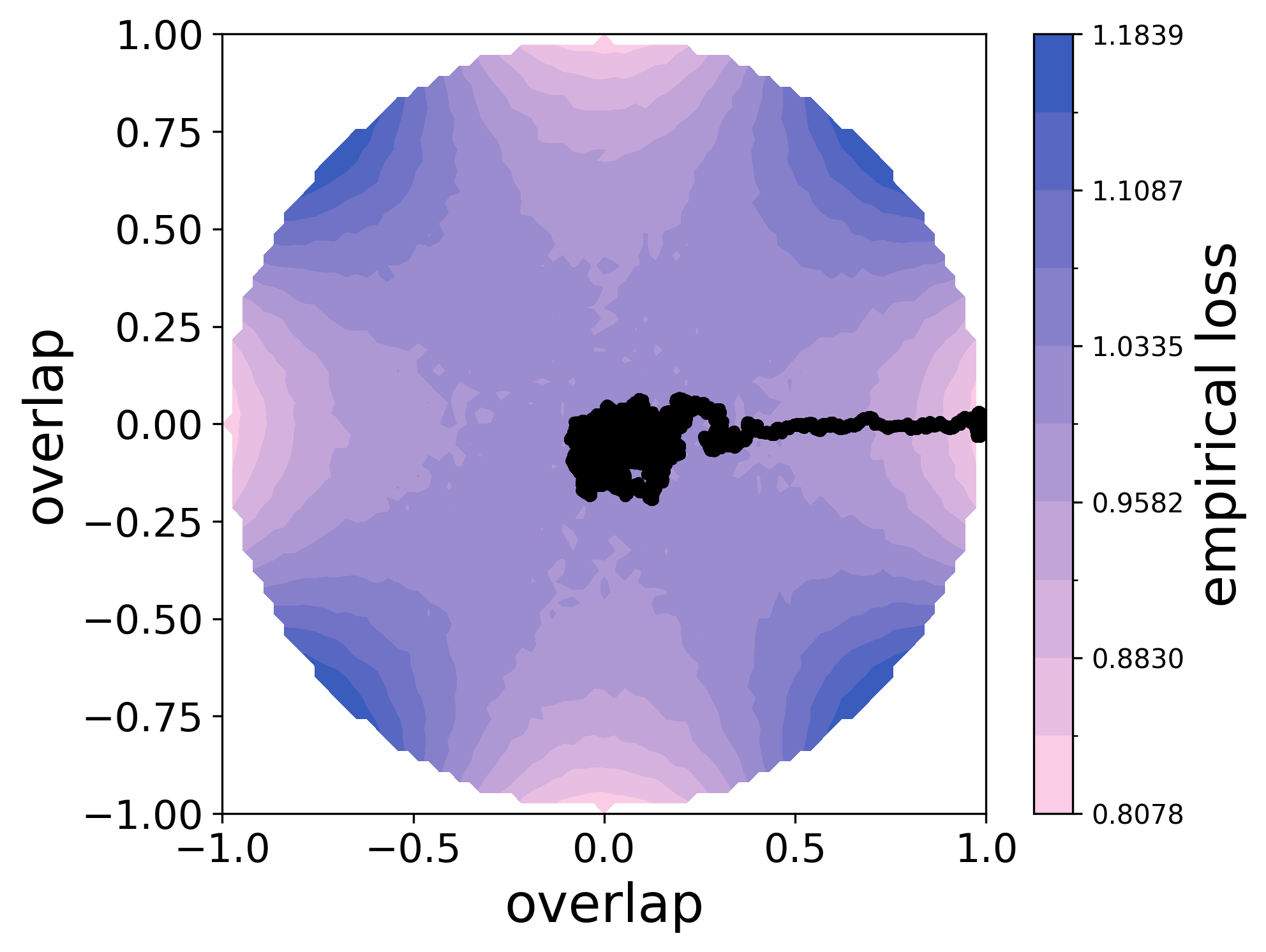}
\caption{\label{fig:landscape}\textbf{Level sets of the empirical correlation loss.} Consider the empirical correlation loss \eqref{correlationloss} with inputs sampled from the Fourier data model \eqref{eq:model} and $\Sigma = \mathds{1}$. Here, the activation function $\sigma$ is the $4$-th Hermite polynomial $\sigma(s) = h_4(s) = s^4 - 6 s^2 + 3$. 
We plot the level sets of this loss, while its 3D version is shown in \Cref{fig:fig2} (right).
As before, one run of SGD ($\bullet$) is displayed: SGD wanders for a long time in the search phase before recovering the non-Gaussian signal. On the $x$- and $y$-axes we have the overlaps $w \cdot u$ and $w \cdot v$ of the weight vector $w$ with the DFT phase vectors $u$ and $v$. Note that the minima of the loss landscape are obtained at $w = \cos{\theta}\,u + \sin{\theta}\,v$ for the angles $\theta = 0, \pi/2, \pi, 3/2\pi$, which are the values assumed by the discrete latent variable $U$ in the Fourier data model. Note that we are demonstrating experimentally that, beyond the weak recovery proved in \cref{thm:hard}, one of the DFT vectors is recovered strongly.}
\end{figure}

\section{Expansion of the population loss for non-isotropic inputs}\label{app:expansionloss}
We now focus on the setting in which the inputs are non-isotropic: both classes of inputs have a non-trivial circulant covariance $\Sigma$. In the next proposition, we expand the population correlation loss \eqref{def:population} in a form such that it is possible to explicitly identify the low-order terms in the overlaps.
\begin{customprop}{4}
       Consider $w \in \mathbb{R}^N$, $\alpha_u = w \cdot u$ and $\alpha_v = w \cdot v$, where $(u,v)$ are the $\rm{DFT}$ phase vectors.
       Let $\lambda_{k_0}$ be the $k_0$-th eigenvalue of a circulant matrix $\Sigma$ and sample the inputs $x^{\mu}$ from the Fourier data model \eqref{eq:model} with circulant covariance $\Sigma$. 
    Then, the population correlation loss \eqref{def:population} is such that 
    \begin{small}
    \begin{align}\label{eq:lossexp}
        L(w)\hspace{-0.1em} & = 1 -\hspace{-0.2em}\lambda_{k_0}^2
        \biggl(c_{04} \alpha_u^4 + c_{22} \alpha_u^2\alpha_v^2 + c_{04}\alpha_v^4 \biggr) \biggl[c_4^{\sigma}
        + c_6^{\sigma} \frac{(\sigma_{\Sigma}^2 -1)}{2} 
        \biggr]\hspace{-0.4em} + \rm{h.o.t.},
    \end{align}
    \end{small}
    with $\sigma_{\Sigma} = \sqrt{w^{\top}\Sigma w}$. The coefficients $c_{04}, c_{22}$ and $ c_{04}$ depend on the likelihood ratio between $\mathbb{P}$ and $\mathbb{P}_0$, while $c_{4}^{\sigma}$ and $c_{6}^{\sigma}$ are the $4$-th and $6$-th order Hermite coefficient of the activation function $\sigma$.
\end{customprop}
\begin{proof}
    The population correlation loss to be expanded is
    \begin{equation*}
        L(w) = 1 - \frac{1}{2}\mathbb{E}_{\mathbb{P}}[\sigma(w \cdot x)] +\frac{1}{2} \mathbb{E}_{\mathbb{P}_0}[\sigma(w \cdot z)],
    \end{equation*}
where we can write 
\begin{math}
    \mathbb{E}_{\mathbb{P}}[\sigma(w \cdot x)] = \mathbb{E}_{\mathbb{P}_0}[\sigma(w \cdot x) \ell(u \cdot x, v \cdot x)].  
\end{math}
We start by expanding in Hermite polynomials the activation function $\sigma$.
If the inputs $x$ are distributed according to the baseline distribution $\mathbb{P}_0$, by defining $a = w \cdot x$ and $\sigma_{\Sigma} = \sqrt{w^{\top} \Sigma} w$, we have that
$A= a / \sigma_{\Sigma} \sim \mathcal{N}(0,1)$. 
By Hermite expansion in the variable $w \cdot x = \sigma_{\Sigma} A$, it holds that
\begin{equation*}
    \sigma(\sigma_{\Sigma} A) = \sum_{n = 0}^{\infty} \frac{c_n^{\sigma}}{n!} h_n(\sigma_{\Sigma}A),
\end{equation*}
where $c_n^{\sigma}$ and $h_n$ are the $n$-th order Hermite coefficient and Hermite polynomial, respectively.
We can now apply a rescaling formula for Hermite polynomials to $h_n(\sigma_{\Sigma} A)$, which yields
\begin{equation}\label{eq:rescaling}
    h_n(\sigma_{\Sigma}A) = \sum_{m = 0}^{\lfloor N/2 \rfloor} \frac{n!}{(n-2m)! m!} \sigma_{\Sigma}^{n-2m} \biggl( \frac{\sigma_{\Sigma}^2-1}{2}\biggr)^m h_{n-2m}(A).
\end{equation}
Define now 
\begin{equation*}
    \rho_{u} = \frac{w^{\top}\Sigma u}{\sigma_{\Sigma}\sigma_B},
    \rho_{v} = \frac{w^{\top}\Sigma v}{\sigma_{\Sigma}\sigma_C}, 
    \rho_{uv} = \frac{v^{\top}\Sigma v}{\sigma_{B}\sigma_C},
\end{equation*}
where 
\begin{math}
    \sigma_B = \sqrt{u^{\top}\Sigma u} \text{ and } \sigma_C = \sqrt{v^{\top}\Sigma v}.
\end{math}
Note that, since the DFT phase vectors $u$ and $v$ are orthogonal eigenvectors of $\Sigma$, we get that $\rho_{uv} = 0$. Then, we can write
\begin{equation*}
    A = \rho_v B + \rho_u C + \sqrt{1-\rho_u^2 -\rho_v^2}\; Z,
\end{equation*}
where we have defined the rescaled variables $B$ and $C$ as $B = (x \cdot v)/\sigma_B$ and \mbox{$C = (x \cdot u)/\sigma_C$}, which are independent. 
If $Z \sim \mathcal{N}(0,1)$ is independent of $B$ and $C$, we get that $B, C, Z$ are independent standard Gaussian random variables, and therefore i.i.d.
Since $\rho_u^2 + \rho_v^2 + \tau^2 = 1$,
if we define $\tau$ as $\tau = \sqrt{1-\rho_u^2 - \rho_v^2}$, we can use the formulas for the sum of Hermite polynomials and conclude that 
\begin{equation}\label{eq:sum}
    h_{n-2m} = \sum_{n-2m = i+j+k} \frac{(n-2m)!}{i!j!k!} \rho_v^i \rho_u^j \tau^k h_i(C) h_j(B) h_k(Z). 
\end{equation}
Putting together \cref{eq:rescaling} and \cref{eq:sum}, we end up with 
\begin{equation}\label{eq:easy}
    \sigma(\sigma_{\Sigma}A) = \sum_{n = 0}^{\infty} \sum_{m=0}^{\lfloor N/2 \rfloor} \sum_{n-2m=i+j+k}^{\infty} 
    \frac{c_n^{\sigma}}{m!i!j!k!} \sigma_{\Sigma}^{n-2m} \biggl( \frac{\sigma_{\Sigma}^2-1}{2}  \biggr)^m \rho_v^i \rho_u^j \tau^k h_i(C) h_j(B) h_k(Z).
\end{equation}
Now that we have expanded the activation function $\sigma$, we want to similarly expand the likelihood ratio $\ell$ in Hermite polynomials. 
Define $\tilde{\ell}$ as the rescaled likelihood ratio such that $\ell(c, b) = \tilde{\ell}(c / \sigma_C, b / \sigma_B)$, with $c = v \cdot x$ and $b = u \cdot x$. 
We cannot directly expand $\ell$ in Hermite polynomials because we cannot compute its Hermite coefficients (which are 
\begin{math}
    \mathbb{E}[\ell(C, B) h_i(C) h_j(B)]
\end{math}), but we are actually able to compute the Hermite coefficients of its rescaled version $\tilde{\ell}$. Indeed, we have that
\begin{equation}\label{eq:rescaled_like}
    \ell(c, b) = \tilde{\ell}(C,B) = \sum_{i,j=0}^{+\infty} \frac{c_{ij}^{\tilde{\ell}}}{i!j!} h_{i}(C)h_j(B),
\end{equation}
with coefficients given by
\begin{align*}
    c_{ij}^{\tilde{\ell}} 
    & = \mathbb{E}_{B,C}[\tilde{\ell}(B,C) h_i(C)h_j(B)]\\[2pt]
    & = \mathbb{E}_{B,C}[\ell(\sigma_B B,\sigma_C C) h_i(C)h_j(B)]\\[2pt]
    & = \mathbb{E}_{\mathbb{P}}\biggl[h_i\biggl(\frac{v \cdot x}{\sigma_C}\biggr) h_j\biggl(\frac{u \cdot x}{\sigma_B}\biggr)\biggr],
\end{align*}
where we have used the definition of Hermite coefficients for $\tilde{\ell}$ in the the first equality and the definition of $\tilde{\ell}$ in the second one.
In what follows, we will always denote them by $c_{ij}^{\tilde{\ell}}$.
Note that we have already computed $c_{ij}^{\tilde{\ell}}$ in Lemma \ref{app:fourth}, where we have called them simply
$c_{ij}^{\ell}$. From now on, we will refer to them as $c_{ij}^{\ell}$.
Hence, 
\begin{align*}
    & \mathbb{E}_{\mathbb{P}}[\sigma(w \cdot x)]  
    = \mathbb{E}_{\mathbb{P}_0}[\sigma(w \cdot x) \ell(v \cdot x, u \cdot x)] \\
    & = \hspace{-0.9em}\sum_{s,t,n=0}^{+\infty}\hspace{-0.2em} \sum_{m=0}^{\lfloor n/2 \rfloor} \hspace{-0.2em}\sum_{i+j+k = n-2m} \hspace{-0.4em} \frac{c_{st}^{\ell} c_n^{\sigma}\, \sigma_{\Sigma}^{n-2m}} {s!t!m!i!j!k!} \hspace{-0.1em}
    \Biggl( \hspace{-0.3em}\frac{\sigma_{\Sigma}^2 - 1}{2} \hspace{-0.3em}\Biggr)^m  \hspace{-0.8em}
    \rho_v^i \rho_u^j \tau^k \underbrace{\mathbb{E}_{B}[h_i(B)h_s(B)]}_{i!\delta_{is}}
    \underbrace{\mathbb{E}_{C}[h_j(C)h_t(C)]}_{j!\delta_{jt}} 
    \mathbb{E}_{Z}[h_k(Z)] \\
    & = \sum_{n=0}^{+\infty} \sum_{m=0}^{\lfloor n/2 \rfloor} \sum_{i+j = n-2m} \frac{c_{ij}^{\ell} c_n^{\sigma}}{m!i!j!} \sigma_{\Sigma}^{n-2m} 
    \Biggl( \hspace{-0.3em}\frac{\sigma_{\Sigma}^2 - 1}{2} \hspace{-0.3em}\Biggr)^m \hspace{-0.7em}\rho_v^i \rho_u^j,
\end{align*}
where the last equality follows from the fact that the only non vanishing terms are the ones with $k=0, i=s$ and $j=t$. 
If $\alpha_v = v \cdot w$ and $\alpha_u = u \cdot w$, we get that
\begin{equation*}
    \rho_v^i = \frac{(w^{\top} \Sigma v)^i}{\sigma_{\Sigma}^i \sqrt{\lambda_{k_0}^i}} = \frac{\lambda_{k_0}^i \alpha_v^i}{\sigma_{\Sigma}^i \sqrt{\lambda_{k_0}^i}} = \frac{\lambda_{k_0}^{i/2} \alpha_v^j}{\sigma_{\Sigma}^i}
\end{equation*}
and, analogously, 
\begin{math}
    \rho_u^j = (\lambda^{j/2}_{k_0}\alpha_u^j)/\sigma_{\Sigma}^j. 
\end{math}
Then, defining $n = i + j + 2m$, we get that
\begin{align}
    \mathbb{E}_{\mathbb{P}}[\sigma(w \cdot x)] 
    & = \sum_{i,j,m = 0}^{+\infty}\frac{c_{ij}^{\ell} c_{i+j+2m}^{\sigma}}{m!i!j!} \Biggl( \frac{\sigma_{\Sigma}^2 -1}{2}   \Biggr)^{m} \lambda_{k_0}^{(i+j)/2} \alpha_{v}^i \alpha_u^j \\\label{eq:gauss}
    & = \underbrace{\sum_{i,j=0, m\geq 0}^{+\infty} \frac{c_{2m}^{\sigma}}{m!}\Biggl( \frac{\sigma_{\Sigma}^2 - 1}{2}\Biggr)^m}_{\mathbb{E}_{\mathbb{P}_0}[\sigma(w \cdot x)]} + 
    \sum_{\substack{m\geq 0, \\ i>0 \text{ or } j>0}} \frac{c_{ij}^{\ell} c_{i+j+2m}^{\sigma}}{m!i!j!} \Biggl( \frac{\sigma_{\Sigma}^2 -1}{2} \Biggr)^m \lambda_{k_0}^{(i+j)/2} \alpha_v^i \alpha_u^j.
\end{align}
Note that the first term in \cref{eq:gauss} is exactly $\mathbb{E}_{\mathbb{P}_0}[\sigma(w \cdot x)]$. Indeed, we can take the expectation with respect to $B,C$ and $Z$ in \cref{eq:easy} and, as a result, only the addendum corresponding to $k=0$ survives.
Therefore, we conclude that 
\begin{equation*}
    L(w) = 1 -\frac{1}{2} \sum_{\substack{m\geq 0, \\ i>0 \text{ or } j>0}} \frac{c_{ij}^{\ell} c_{i+j+2m}^{\sigma}}{m!i!j!} \Biggl( \frac{\sigma_{\Sigma}^2 -1}{2} \Biggr)^m \lambda_{k_0}^{(i+j)/2} \alpha_v^i \alpha_u^j.
\end{equation*}
Since $\sigma$ is even, only the terms with $i+j+2m \in 2\mathbb{N}$ do not vanish. We compute them up to sixth order in the total degree of the overlaps.
\begin{itemize}
    \item \textbf{Order 2:} $i+j+2m = 2$. We get that
    \begin{equation*}
        L_2 = \lambda_{k_0}\biggl[ c_2^{\sigma} \underbrace{\sum_{i=0}^2 \dfrac{\alpha_v^i \alpha_u^{2-i}}{i!(2-i)!}c_{i,2-i}^{\ell}}_{i+j = 2,\, m=0} \biggr] = 0,
    \end{equation*}
    since $c_{ij}^{\ell}=0$ when $i+j = 2$ because of Lemma \ref{app:coefficients}.
    \item \textbf{Order 4:} $i+j+2m = 4$. We get that
    \begin{equation*}
    L_4 \hspace{-0.3em}= \hspace{-0.4em}\biggl[\underbrace{\lambda_{k_0}^2 c_4^{\sigma} \hspace{-0.3em} \sum_{i=0}^4 \hspace{-0.3em}\dfrac{\alpha_v^i \alpha_u^{4-i}}{i!(4-i)!}c_{i,4-i}^{\ell}}_{i+j = 4, \; m=0} \hspace{-0.3em}+ \hspace{-0.3em}\underbrace{\lambda_{k_0} c_6^{\sigma} \frac{(\sigma_{\Sigma}^2 -1)}{2} \hspace{-0.3em}\sum_{i=0}^2 \hspace{-0.3em}\dfrac{\alpha_v^i \alpha_u^{2-i}}{i! (2-i)!} c_{i,2-i}^{\ell}}_{i+j = 2, \; m=1} \hspace{-0.1em}\biggr] \hspace{-0.3em}= \hspace{-0.3em}
    \lambda_{k_0}^2 c_4^{\sigma} \hspace{-0.3em}\sum_{i=0}^4 \hspace{-0.3em}\dfrac{\alpha_v^i \alpha_u^{4-i}}{i!(4-i)!}c_{i,4-i}^{\ell}.
    \end{equation*}
    \item \textbf{Order 6:} $i+j+2m = 6$. We get that
    \begin{equation*}
    L_6 = \biggl[\underbrace{\lambda_{k_0}^3 c_6^{\sigma} \sum_{i=0}^4 \dfrac{\alpha_v^i \alpha_u^{6-i}}{i!(6-i)!}c_{i,6-i}^{\ell}}_{i+j = 6, \; m=0} + \underbrace{\lambda_{k_0}^2 c_6^{\sigma} \frac{(\sigma_{\Sigma}^2 -1)}{2} \sum_{i=0}^4 \dfrac{\alpha_v^i \alpha_u^{4-i}}{i! (4-i)!} c_{i,4-i}^{\ell}}_{i+j = 4, \; m=1} \biggr].
    \end{equation*}
\end{itemize}
Then, the population correlation loss reads as
\begin{align*}
        L(w) \hspace{-0.3em} & = \hspace{-0.3em}1 \hspace{-0.3em}-\hspace{-0.3em}\lambda_{k_0}^2\hspace{-0.2em} \biggl[ c_4^{\sigma} \hspace{-0.2em}\sum_{i=0}^4 \hspace{-0.3em}\dfrac{\alpha_v^i \alpha_u^{4-i}}{i!(4-i)!}c_{i,4-i}^{\ell} \hspace{-0.3em}+ \hspace{-0.3em}c_6^{\sigma} \frac{(\sigma_{\Sigma}^2 -1)}{2} \hspace{-0.4em}\sum_{i=0}^4 \hspace{-0.3em}\dfrac{\alpha_v^i \alpha_u^{4-i}}{i! (4-i)!} c_{i,4-i}^{\ell} \hspace{-0.2em}\biggr]  \hspace{-0.4em}+ \hspace{-0.3em}
        \lambda_{k_0}^3 c_6^{\sigma} \hspace{-0.3em}\sum_{i=0}^4 \hspace{-0.3em}\dfrac{\alpha_v^i \alpha_u^{6-i}}{i!(6-i)!}c_{i,6-i}^{\ell} \hspace{-0.3em}+\hspace{-0.2em}
        \text{h.o.t.}\\
        & = 1 -\lambda_{k_0}^2 \biggl(\frac{c_{04}^{\ell}}{4!} \alpha_u^4 + \frac{c_{22}^{\ell}}{4} \alpha_u^2\alpha_v^2 + \frac{c_{04}^{\ell}}{4!}\alpha_v^4 \biggr)
        \biggl[ c_4^{\sigma} 
        + c_6^{\sigma} \frac{(\sigma_{\Sigma}^2 -1)}{2} 
        \biggr]+ \text{h.o.t.}
    \end{align*}
    where the last equality follows from the computations of the coefficients $c_{ij}^{\ell}$ in Lemma \ref{app:coefficients}.
    By defining $c_{04} = c_{04}^{\ell}/4!, c_{22} = c_{22}^{\ell}/4$ and $c_{40} = c_{40}^{\ell}/4!$, we get the thesis.
    We conclude by discussing briefly the order of the quadratic form associated to the covariance matrix $\Sigma$. 
    We can write $\sigma_{\Sigma}^2 - 1$ as 
    \begin{equation*}
        \sigma_{\Sigma}^2 -1 = (\lambda_{k_0} -1)(\alpha_u^2 + \alpha_v^2) + \sum_{m = 1}^{N-2} (\lambda_m -1)(\alpha_{u_m}^2 + \alpha_{v_m}^2),
    \end{equation*}
    where we have denoted by $(u_m, v_m)_{m=1}^{N-1}$ the DFT basis eigenvectors of $\Sigma$, excluding the phase eigenvectors $u,v$. Since the degree of $\sigma^2_{\Sigma}-1$ is two in the overlaps, it induces a total contribution of order six. We include it anyway in the explicit expression of the loss because this contribution becomes dominant in case of eigenvalues scaling extensively with the dimension of the inputs - a case we will treat below. In conclusion, in ``h.o.t.'' (higher-order terms), we can find terms of order six or more.
\end{proof}

\section{Learning the phase at quasi-linear sample complexity with SGD}\label{sec:non-iso}
The goal of this section is to study the dynamics of online SGD on our classification task in a realistic non-isotropic setting. Specifically, we consider inputs drawn from the Fourier data model \eqref{eq:model} with a non-trivial circulant covariance matrix $\Sigma$.
We focus on the setting where a finite number of eigenvalues $\lambda_1, \dots, \lambda_M$ of $\Sigma$ identify a principal subspace, and all the remaining eigenvalues are precisely equal to one. 
Additionally, we include $\lambda_{k_0}$, the eigenvalues corresponding to the DFT phase vectors, in the list of the non-trivial eigenvalues. 
As a result, the principal subspace of the inputs is spanned by the eigenvectors associated with $(\lambda_{k_0}, \lambda_1, \cdots, \lambda_M)$.
In particular, we are going to analyse the case in which these non-trivial eigenvalues are extensive $O(N)$ with the dimension of the inputs; this is a realistic scenario, given the power-law decay observed in the spectrum of real images (cf. \cref{fig:fig3}).

Inspired by the setting proposed by \citet{benarous2022high}, we consider the correlation loss \cref{correlationloss} plus a quartic penalization term $\beta \|w\|^4$, for $\beta >0$, and online SGD without normalisation. More precisely, for any $t \geq 1$, the iteration reads as
\begin{equation}\label{eq:sgdnon}
    w_t = w_{t-1} + \delta_N\,
        \nabla L(w,x_{t}, y_{t}) \bigl|_{w = w_{t-1}},
\end{equation}
for $\delta_N = O(1/N)$. Here, $L(w)$ and $L(w,Y)$ denote the population loss and the point-wise loss respectively, where $Y = (x_t,y_t)$ indicates the data point sampled at time $t$. 

We study the dynamics of a collection of summary statistics $\boldsymbol{\alpha}$ during training. 
In particular, our summary statistics track the evolution of the weight vector projected onto the principal components of inputs drawn from the Fourier data model, including the DFT phase vectors.

\subsection{Literature review}\label{app:literaturereview}
We quickly recap the tools needed to address our classification problem on inputs drawn from the Fourier data model \eqref{eq:model}. These tools are borrowed from the results proven by \citet{benarous2022high} and \citet{benarous2025spectral}, where the authors analyse tensor PCA and classification on a Gaussian mixture model.
Assume we have 
\begin{math}
    H(w,Y) = L(w,Y) - L(w), 
\end{math}
where, as usual, the population correlation loss is $L(w) = \mathbb{E}_{Y}[L(w,Y)]$. 
Define also 
\begin{math}
    V(w) = \mathbb{E}_Y[\nabla H(w) \otimes\nabla H(w)],
\end{math}
which is essentially the covariance matrix for $\nabla H$ evaluated at the weight vector $w$. 
Consider a set of $S$ summary statistics $\boldsymbol{\alpha} = (\alpha_i)_{i=1}^S$, depending on the ambient dimension $N$, i.e. $\boldsymbol{\alpha = \alpha}(N)$, and denote with $J = \nabla \alpha_i$ the Jacobian of the summary statistics.
The following is a regularity assumption:
\begin{definition}
    We say that $(\boldsymbol{\alpha}, L)$ is ``$\delta_N$-localizable'' with localizing sequence $(E_K)_K$ if there is an exhaustion by compacts $(E_K)_K$ of $\,\mathbb{R}^S$ and constants $C_K$  independent of $N$ such that 
    \begin{enumerate}
\item \begin{math}
\max_i \sup_{w \in \boldsymbol{\alpha}^{-1}(E_K)} \left\| \nabla^2 \alpha_i \right\|_{\mathrm{op}}
\le C_K \, \delta_N^{-1/2}
\qquad\text{and}\qquad
\max_i \sup_{w \in \boldsymbol{\alpha}^{-1}(E_K)}
\left\| \nabla^3 \alpha_i \right\|_{\mathrm{op}}\le C_K.
\end{math}
\item
\begin{math}
\sup_{w \in \boldsymbol{\alpha}^{-1}(E_K)} \| \nabla L(w) \| \le C_K
\qquad \text{and} \qquad
\sup_{w \in \boldsymbol{\alpha}^{-1}(E_K)}
\mathbb{E}\!\left[ \| \nabla H \|^8 \right]
\le C_K \, \delta_N^{-4}.
\end{math}
\item
\begin{math}
\max_i \hspace{-0.6em}\sup\limits_{w \in \boldsymbol{\alpha}^{-1}(E_K)}\hspace{-0.6em}
\mathbb{E}\!\left[ \langle \nabla H, \nabla \alpha_i \rangle^4 \right]\hspace{-0.4em}
\le C_K \delta_N^{-2}, \quad
\max_i \hspace{-0.6em}\sup\limits_{w \in \boldsymbol{\alpha}^{-1}(E_K)}\hspace{-0.6em}
\mathbb{E}\!\left[ \left\langle
\nabla^2 \alpha_i, \hspace{-0.2em}
\nabla H \otimes \nabla H \hspace{-0.2em} -\hspace{-0.2em} V
\right\rangle^2\right] \hspace{-0.2em} = \hspace{-0.2em} o\!\left( \delta_N^{-3} \right).
\end{math}
\end{enumerate}
\end{definition}
\noindent
Consider now the first and second-order differential operators 
\begin{equation}\label{operators}
    \mathcal{A}_N = \sum_{i}\partial_i L' \partial_i \quad\text{ and }\quad \mathcal{L}_N = \frac{1}{2}\sum_{i,j} V_{ij}\partial_i \partial_j,
\end{equation}
or, equivalently, $\mathcal{A}_N = \langle L' , \nabla \rangle$ and $\mathcal{L}_N = \frac{1}{2} \langle V, \nabla^2 \rangle$.
\begin{definition}
    The summary statistics $\boldsymbol{\alpha}$ are ``asymptotically closable'' for learning rate $\delta_N$ if $(\boldsymbol{\alpha}, L)$ is $\delta_N$-localizable with localizing sequence $(E_K)_K$ and furthermore there exist locally Lipschitz functions $h: \mathbb{R}^S \to \mathbb{R}^S$ and $\Sigma: \mathbb{R}^S \times \mathbb{R}^{S \times S}$ such that
    \begin{align}
        \sup_{w \in \boldsymbol{\alpha}^{-1}(E_K)}
        \bigl\|
        \bigl(-\mathcal{A}_N + \delta_N \mathcal{L}_N \bigr) \boldsymbol{\alpha} - h\bigl(\boldsymbol{\alpha}\bigr)
        \bigr\|
        \to 0,
        \sup_{w \in \boldsymbol{\alpha}^{-1}(E_K)}
        \bigl\|
        \delta_N J_n V J_n^{\mathsf T} - \Sigma\bigl(\boldsymbol{\alpha}\bigr)
        \bigr\|
        \to 0.
    \end{align}
    We call $\Sigma$ \textbf{diffusion matrix} and $h$ \textbf{effective drift} for $\boldsymbol{\alpha}$. If $\mathcal{A}_N$ and $\delta_N \mathcal{L}_N$ admit themselves limits, i.e\
    \begin{math}
        \sup
        \bigl\|
        \mathcal{A}_N \boldsymbol{\alpha} - A_{\boldsymbol{\alpha}} (\boldsymbol{\alpha})
        \bigr\|
        \to 0
    \end{math}
    and 
    \begin{math}
        \sup
        \bigl\|
        \delta_N \mathcal{L}_N \boldsymbol{\alpha} - G (\boldsymbol{\alpha})
        \bigr\|
        \to 0,
    \end{math}
    for some $A_{\boldsymbol{\alpha}},G: \mathbb{R}^S \to \mathbb{R}^S$, 
    where the supremum is taken over
    \begin{math}
        w \in \boldsymbol{\alpha}^{-1}(E_K)\delta_N
    \end{math},
    we call $A_{\boldsymbol{\alpha}}$ \textbf{population drift} and $G$ \textbf{population corrector}.
\end{definition}
We present now the theorem providing the evolution of the summary statistics at linear time, originally Theorem 2.2. in \citep{benarous2022high}. The solution to the given SDE is called \textbf{effective dynamics}.

\begin{theorem}[Effective dynamics]\label{thm:effective}
Let $(w_t)_{t} \subseteq \mathbb{R}^N$ be the estimators given by online $\rm{SGD}$ \eqref{eq:sgdnon} initialized from
$w_0 \sim \mu_0(\mathbb{R}^N)$, with learning rate $\delta_N$, applied to the loss $L'(\cdot,\cdot)$.
For the family of summary statistics $\boldsymbol{\alpha}= \bigl(\alpha_i\bigr)_{i=1}^S,$
let $\boldsymbol{\alpha}_N(t)$ be the linear interpolation of
\begin{math}
\bigl(\boldsymbol{\alpha}(w_{\lfloor t \delta_N^{-1} \rfloor})\bigr)_{t}.
\end{math}
Suppose that $\boldsymbol{\alpha}$ are asymptotically closable with learning rate $\delta_N$,
effective drift $h$ and diffusion matrix $\Sigma$, and that the pushforward
of the initial data satisfies 
\begin{math}
\boldsymbol{\alpha}_\#\mu_0 \to \nu \text{ weakly for some measure } \nu = \nu(\mathbb{R}^k).
\end{math}
Then
\begin{math}
\boldsymbol{\alpha}_N(t) \to (\boldsymbol{\alpha}_t)_t \text{ weakly as \mbox{$N \to \infty$}},
\end{math}
where $(\boldsymbol{\alpha}_t)_t$ solves the stochastic differential equation
\begin{equation}
\mathrm{d}\boldsymbol{\alpha}_t = h(\boldsymbol{\alpha}_t)\,\mathrm{d}t + \Sigma(\boldsymbol{\alpha}_t)^{1/2}\,\mathrm{d}B_t,
\end{equation}
initialized from $\nu$, and $(B_t)_t$ is a standard Brownian motion in
$\mathbb{R}^S$.
\end{theorem}

\subsection{Results for the Fourier data model}
We want to apply \cref{thm:effective} to our classification task, in presence of a non-trivial covariance matrix for both classes of inputs $\mathbb{P}$ and $\mathbb{P}_0$. We assume that both classes share a principal subspace spanned by some leading principal components $(u^m, v^m)_{m=1}^M$ together with the DFT phase vectors $(u, v)$. We want to describe the dynamics of the overlap of the perceptron weight vector and these principal components. First of all, we formalise the definition of the summary statistics we deal with.
\begin{definition}[Summary statistics]\label{app:defstatistics}
    For a fixed integer $M \in \mathbb{N}$, assume that $(\lambda_m)_{m=1}^M$ are eigenvalues of the covariance matrix $\Sigma$ of inputs sampled from the Fourier data model \eqref{eq:model}, with $\lambda_m > 1$. 
    In addition, assume also that the eigenvalue $\lambda_{k_0}$ corresponding to the $\rm{DFT}$ phase vectors is such that $\lambda_{k_0} > 1$.
    For simplicity, set all the other eigenvalues to one.
    We define the summary statistics $\boldsymbol{\alpha}$ as
    \begin{equation*}
        \boldsymbol{\alpha} = (\alpha_{u}, \alpha_v, (\alpha_{u_m}, \alpha_{v_m})_{m=1}^M , \omega_{\perp}),
    \end{equation*} 
    where $\alpha_{u}$ and $\alpha_v$ are the overlaps of the weight vector $w$ with the $\rm{DFT}$ phase vectors $u$ and $v$. 
    The other statistics $\alpha_{u_1}, \alpha_{v_1}, \dots, \alpha_{u_M}, \alpha_{v_M}$ are the overlaps of the weight vector with the $\rm{DFT}$ basis vectors associated to the set of the eigenvalues identifying the principal subspace $(\lambda_m)_{m=1}^M$.
    Moreover, $\omega_{\perp}= w \cdot w_{\perp}$, where $w_{\perp}$ is the projection of the weight vector $w$
    onto the subspace orthogonal to that spanned by the
    whole set of principal components $(u, v, (u^m, v^m)_{m=1}^M)$.
\end{definition}
Due to the definition of our (linear) statistics, we simply get that the corrector term is $\mathcal{L}_N \boldsymbol{\alpha} = 0$. The population drift is given by
\begin{math}
    A_{\boldsymbol{\alpha}} = \lim_{N \to \infty}\mathcal{A}_N \boldsymbol{\alpha},
\end{math}
if this limit exists.
To identify the effective dynamics for the summary statistics, one should also compute the diffusive matrix.

\subsubsection{Near-isotropic inputs}
In the next lemma, we compute the population drift for the summary statistics, which specifies the population dynamics, when the non-trivial eigenvalues $\lambda_{k_0}, \lambda_1, \dots, \lambda_M$ are $O(1)$.

\begin{lemma}[Population drift] \label{app:populationdrift}
    Consider the population classification loss \eqref{def:population} and the summary statistics introduced in Definition \ref{app:defstatistics}. Assume that all the non-trivial eigenvalues ($\lambda_{k_0}, (\lambda_m)_{m=1}^M$) of the covariance matrix of the inputs are $O(1)$. Then, the population drift 
    \begin{equation*}
    A_{\boldsymbol{\alpha}} = (A_u, A_v, (A_{u_m}, A_{v_m})_{m=1}^M, A_{\omega_{\perp}})
    \end{equation*}
    is as follows: by defining $R = \|w\|^2$, the drifts for the $\rm{DFT}$ phase vectors are
    \begin{align*}
    A_u(\boldsymbol{\alpha}) & = \lim_{N \to +\infty}\mathcal{A}_N \alpha_u \\
    & = \frac{1}{2} \lambda_{k_0}^2 
    \Biggl[ \sum_{i=0,2} \frac{\alpha_v^i \alpha_u^{3-i}c_{i,4-i}^{\ell}}{i! (3-i)!} \biggl( c_4^{\sigma}\hspace{-0.3em}  + c_6^{\sigma}\hspace{-0.1em}\frac{\sigma_{\Sigma}^2 -1}{2} \biggr)\hspace{-0.3em} + c_6^{\sigma}\hspace{-0.5em} \sum_{i=0,2,4}\hspace{-0.7em} \frac{\alpha_v^i \alpha_u^{4-i}c_{i,4-i}^{\ell}}{i!(4-i)!} \hspace{-0.2em} (\lambda_{k_0}-1)\alpha_u \Biggr] + \hspace{-0.1em} 4\beta R^2 \alpha_u,
    \end{align*}
    \begin{align*}
     A_v(\boldsymbol{\alpha}) & 
     = \lim_{N \to +\infty}\mathcal{A}_N \alpha_v\\ 
     & = \frac{1}{2} \lambda_{k_0}^2 
    \Biggl[ \sum_{i=1,3} \frac{\alpha_v^i \alpha_u^{3-i}c_{i+1,3-i}^{\ell}}{i! (3-i)!} \biggl( c_4^{\sigma}\hspace{-0.3em}  +  c_6^{\sigma}\hspace{-0.1em}\frac{\sigma_{\Sigma}^2 -1}{2} \biggr)\hspace{-0.3em} + c_6^{\sigma}\hspace{-0.5em} \sum_{i=0,2,4}\hspace{-0.7em} \frac{\alpha_v^i \alpha_u^{4-i}c_{i,4-i}^{\ell}}{i!(4-i)!} (\lambda_{k_0}-1)\alpha_v \Biggr] + \hspace{-0.1em} 4\beta R^2 \alpha_v.
    \end{align*}
    For the other principal components $( u^m, v^m)_{m=1}^M$, we get that the drifts are
    \begin{align*}
         A_{u_m}(\boldsymbol{\alpha}) = \lim_{N \to +\infty}\mathcal{A}_N \alpha_{u_m} = \frac{1}{2} c_6^{\sigma} \lambda_{k_0}^2 \sum_{i=0, 2, 4} \frac{\alpha_v^i \alpha_u^{4-i}}{i!(4-i)!}c_{i,4-i}^{\ell} (\lambda_{m} -1) \alpha_{u_m} + 4\beta R^2 \alpha_{u_m}
    \end{align*}
    and same for $(v^m)_m$. For the orthogonal part $\omega_{\perp}$, we have that
    \begin{equation*}
          A_{\omega_{\perp}}(\boldsymbol{\alpha}) = \lim_{N \to +\infty}\mathcal{A}_N \omega_{\perp} =  4\beta R^2 \omega_{\perp}.
    \end{equation*}
\end{lemma}
\begin{proof}
Since
\begin{math}
    \|w\|^4 = ( \alpha_u^2 + \alpha_v^2 + \dots + \alpha_{u_M} + \alpha_{v_M} + \omega_{\perp}^2)^2,
\end{math}
we get that
the gradient of the population correlation loss can be written as
\begin{equation*}
    \nabla L(w) = (\mathcal{A}_N \alpha_u) u + (\mathcal{A}_v \alpha_v)v + \sum_{m=1}^M ((\mathcal{A}_N \alpha_{u_m}) u^m + (\mathcal{A}_N \alpha_{v_m}) v^m) + 4(R^2 r_{\perp}) w_{\perp}. 
\end{equation*}
Hence,
\begin{equation*}
    \mathcal{A}_N \alpha_u = -\frac{1}{2} \lambda_{k_0}^2 
    \Biggl[ \sum_{i=0}^3 \frac{\alpha_v^i \alpha_u^{3-i}c_{i,4-i}^{\ell}}{i! (3-i)!} \biggl( c_4^{\sigma} +  c_6^{\sigma}\frac{\sigma_{\Sigma}^2 -1}{2} \biggr) +  c_6^{\sigma}\sum_{i=0}^4 \frac{\alpha_v^i \alpha_u^{4-i}c_{i,4-i}^{\ell}}{i!(4-i)!} (\lambda_{k_0}-1)\alpha_u \Biggr] + 4\beta R^2 \alpha_u,
\end{equation*}
\begin{equation*}
    \mathcal{A}_N \alpha_v = -\frac{1}{2} \lambda_{k_0}^2 
    \Biggl[ \sum_{i=0}^3 \frac{\alpha_v^i \alpha_u^{3-i}c_{i+1,3-i}^{\ell}}{i! (3-i)!} \biggl(c_4^{\sigma} + c_6^{\sigma}\frac{\sigma_{\Sigma}^2 -1}{2} \biggr) + c_6^{\sigma}\sum_{i=0}^4 \frac{\alpha_v^i \alpha_u^{4-i}c_{i,4-i}^{\ell}}{i!(4-i)!} (\lambda_{k_0}-1)\alpha_v \Biggr] + 4\beta R^2 \alpha_v,
\end{equation*}
\begin{equation*}
    \mathcal{A}_N \alpha_{u_m} = -\frac{1}{2} c_6^{\sigma} \lambda_{k_0}^2 \sum_{i=0}^4 \frac{\alpha_v^i \alpha_u^{4-i}}{i!(4-i)!} (\lambda_{m} -1) \alpha_{u_m} + 4\beta R^2 \alpha_{u_m}, 
\end{equation*}
\begin{equation*}
    \mathcal{A}_N \alpha_{v_m} = -\frac{1}{2} c_6^{\sigma} \lambda_{k_0}^2 \sum_{i=0}^4 \frac{\alpha_v^i \alpha_u^{4-i}}{i!(4-i)!} (\lambda_{m} -1) \alpha_{v_m} + 4\beta R^2 \alpha_{v_m}.
\end{equation*}
The population drift for $\omega_{\perp}$ is simply
\begin{equation*}
    \mathcal{A}_N {\omega_{\perp}} = - 4R^2 \beta \omega_{\perp}. 
\end{equation*}
We have assumed that $\lambda_{k_0}$ and the other eigenvalues do not scale with the ambient dimension, and so none of the right-hand-sides depend on $N$. Then, the drifts stay the same when taking the limit for $N \to +\infty$.
The thesis follows from the observation that $c_{13}^{\ell} = c_{31}^{\ell} = 0$, because of Lemma \ref{app:coefficients}.
\end{proof}

In view of the described population dynamics, we claim in Conjecture \ref{conjecture:near-isotropic} that SGD cannot successfully perform classification at linear sample complexity when the leading eigenvalues of the covariance matrix of the inputs are non-extensive. More precisely, we expect (and provide an heuristic derivation in Remark \ref{expect}) that it is possible for SGD to exit the search phase only at cubic sample complexity, as in the setting treated in \cref{thm:hard} for isotropic data points. 

\begin{remark}[Cubic sample complexity for near-isotropic inputs]\label{expect}
    Consider the setting of Conjecture \ref{conjecture:near-isotropic}, i.e. when a finite number $O(1)$ of non-unit eigenvalues of the covariance matrix is $O(1)$. 
    It is possible to heuristically derive, following \citet{damian2023smoothing} and \citet{ricci2025feature}, that online SGD requires a cubic sample complexity to escape the search phase. Indeed, the evolution of any summary statistic, e.g. $\alpha_v$, is given by 
    \begin{equation*}
        \dot{\alpha}_v = \dfrac{\mathbb{E}[g \cdot v]^2}{\alpha_v \mathbb{E}[\|g\|^2]},
    \end{equation*}
 where $g = \nabla L(w; x, y)$. Since $g$
 is a random vector in $\mathbb{R}^N$ where each coordinate is $O(1)$, we have that $\mathbb{E}[\|g\|^2] \approx N$. Following the  calculations in Lemma \ref{app:populationdrift}, we get $\mathbb{E}[g \cdot v] \approx \alpha_v^3$. Then, $\dot{\alpha}_v = \alpha_v^5/N$ which, initialised at $\alpha(0) = 1/\sqrt{N}$, implies that $\alpha_v = O(1)$ in approximately $N^3$ steps.   
\end{remark}

\subsubsection{Power-law decaying inputs}
In this section we compute the population drifts when
the top eigenvalues $((\lambda_m)_m, \lambda_{k_0})$ are extensive $O(N)$ with respect to the ambient dimension $N$.

The presence of a shared leading principal subspace for the covariance matrix of both classes has the effect of inducing an extensive signal-to-noise ratio, which allows a faster recovery - from cubic (as in the isotropic case) to quasi-linear. An extensive signal-to-noise ratio has been considered in \citet{benarous2018pca} and it was indeed crucial to effectively solve the tensor PCA problem. 

Note that, since this subspace is shared among the two classes of inputs, it is not immediately clear that it impacts the speed of solving the task, as it is not task-relevant. Nevertheless, we establish that it accelerates the learning of the information contained in higher-order cumulants of the inputs, namely in the phase.

Furthermore, since the eigenvalues are extensive, the previous approach breaks down, as we would have $A_{\boldsymbol{\alpha}} = \infty$. 
What we can do is zooming close to the equator and study the effective dynamics for the rescaled summary statistics, as suggested by \citep{benarous2022high}.

Consider then a rescaling of the summary statistics in a microscopic neighborhood of zero, which captures the initial phase of their evolution from a random start of the weight vector. 
The new rescaled statistics are defined as
\begin{equation}
    \boldsymbol{m} = (\sqrt{N}\alpha_u, \sqrt{N}\alpha_v,\; (\sqrt{N}\alpha_{u_m}, \sqrt{N}\alpha_{v_m})_{m=1}^M,\; \omega_{\perp}).
\end{equation}
Note that the rescaled statistics are initialised at the non-vanishing measure
\begin{equation*}
    \nu = \lim_{N \to +\infty} \, \boldsymbol{m}_\# \mu_0 = \mathcal{N}(0,1)\otimes \dots \otimes \mathcal{N}(0,1) \otimes \delta_1.
\end{equation*}
In the next lemma, we compute the population drift $A_{\boldsymbol{m}}(\boldsymbol{m})$ for the rescaled statistics. Note that from the following calculations it is clear that this population drift would vanish in case of near-isotropic inputs. Conversely, $A_{\boldsymbol{m}}(\boldsymbol{m})$ is not identically zero when $\lambda_{k_0}$ is extensive. 
More precisely, the minimal scaling required for the induced signal-to-noise is $\lambda_{k_0}^2 \approx N$, which is the same required to solve tensor PCA for a fourth-order tensor with SGD at linear sample complexity, as proven in \citet{benarous2018pca}.

\begin{figure*}
\centering
\includegraphics[width=0.45\linewidth]{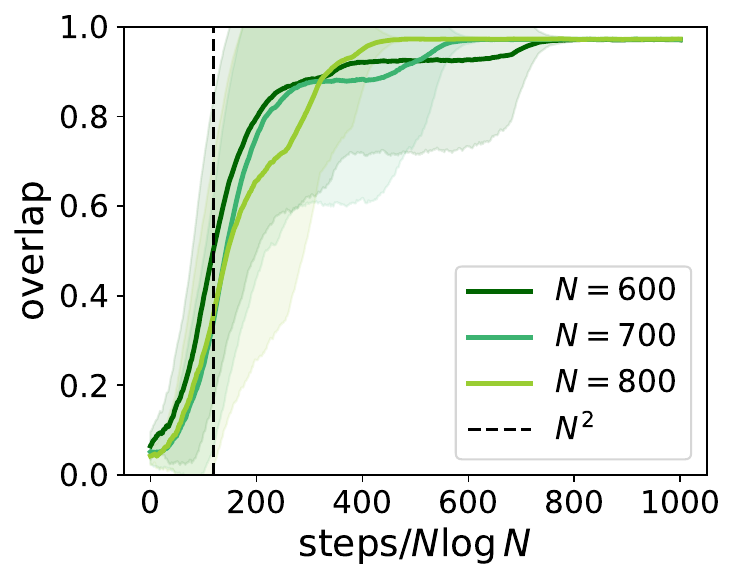}\label{fig:landscape}
\caption{\textbf{Fast recovery of the phase information.} Consider the correlation loss \eqref{correlationloss} and non-isotropic inputs sampled from the Fourier data model \eqref{eq:model}, with circulant covariance matrix in which only the eigenvalue corresponding to the modified mode $k_0$ in the phase is extensive ($\lambda_{k_0} \approx \sqrt{N}$), the other ones are set equal to one. Recall that the covariance matrix is shared among the inputs of both classes. We run SGD to distinguish the two classes of inputs, in dimensions $N = 600, 700, 800$. 
SGD weakly recovers the subspace spanned by the DFT phase vectors at quasi-linear sample complexity - on the $y$- axis we have the norm of the projection of the weight vector in the subspace spanned by the DFT phase vectors. The reason why such fast recovery happens is explained in Conjecture \ref{conjecture:power-law}. Note that the signal is recovered within $n \ll N^2$ steps, as happens in the case of information exponent $k^* = 2$ for single index models. We use the $\log \cosh$ activation function.}
\end{figure*}

\begin{lemma}[Rescaled population drift] \label{app:rescaledpopulationdrif}
    Consider the correlation classification loss \eqref{def:population} and the rescaled summary statistics \eqref{app:rescaledstatistics}. Assume that $\lambda_{k_0} \approx \sqrt{N}$ and $\sqrt{N} \lesssim \lambda_m \lesssim N$. Then, the population drift for the rescaled statistics is 
    \begin{equation*}
    A_{\boldsymbol{m}} = (A_{m_u}, A_{m_v}, (A_{m_{u_M}}, A_{m_{v_M}})_{m=1}^M, A_{\omega_{\perp}})
    \end{equation*}
    such that
    \begin{equation*}
    A_{m_u} = \lim_{N \to +\infty}\mathcal{A}_N m_u = \sum_{i=0,2} \frac{m_v^i m_u^{3-i}c_{i,4-i}^{\ell}}{2 i! (3-i)!}c_4^{\sigma} + \frac{c_6^{\sigma}}{4} \sum_{m=1}^M (m_{u_m}^2 + m_{v_m}^2) + 4\beta R^2 m_u,
    \end{equation*}
    \begin{align*}
    A_{m_v} = \lim_{N \to +\infty}\mathcal{A}_N m_v = 
    \sum_{i=1,3} \frac{\alpha_v^i \alpha_u^{3-i}c_{i+1,3-i}^{\ell}}{2i! (3-i)!} c_4^{\sigma}\hspace{-0.3em} + 
    \frac{c_6^{\sigma}}{4} \sum_{m=1}^M (m_{u_m}^2 + m_{v_m}^2) + 4\beta R^2 m_v
    \end{align*}
    for the statistics $m_u$ and $m_v$ associated to the $\rm{DFT}$ phase vectors. Conversely, for any $m = 1, \dots, M$, we get
    \begin{align*}
        A_{m_{u_m}} = \lim_{N \to +\infty}\mathcal{A}_N m_{u_m} = \frac{c_6^{\sigma}}{2} \sum_{i=0, 2, 4} \frac{m_v^i m_u^{4-i}}{i!(4-i)!}c_{i,4-i}^{\ell} \;m_{u_m} + 4\beta R^2 m_{u_m}
    \end{align*}
    and same for $m_{v_m}$. The drift for $\omega_{\perp}$ stays the same as the previous case and implies that
    \begin{math}
        \dot{\omega}_{\perp} = - 4\beta R^2 \omega_{\perp}.
    \end{math}
\end{lemma}
\begin{proof}
By definition of $\sigma_{\Sigma}$, we have that
\begin{align*}
    \sigma_{\Sigma}^2 -1 & = (\lambda_{k_0} -1)(\alpha_u^2 + \alpha_v^2) + \sum_{m = 1}^M (\lambda_m -1)(\alpha_{u_m}^2 + \alpha_{v_m}^2)\\
    & = \frac{(\sqrt{N}-1)}{N}(m_u^2 + m_v^2) + \frac{(N-1)}{N}\sum_{m=1}^M (m_{u_m}^2 + m_{v_m}^2).
\end{align*}
Then, assuming that $\lambda_m^2 \approx N$, we get
      \begin{align*}
    \mathcal{A}_N m_u  
    & \hspace{-0.2em} = \hspace{-0.1em}\overbrace{\frac{\sqrt{N}}{2} N\hspace{-0.2em} 
    \Biggl[ \sum_{i=0,2} \frac{m_v^i m_u^{3-i}\hspace{-0.2em}c_{i,4-i}^{\ell}}{\sqrt{N}^3 i! (3-i)!} \biggl( \hspace{-0.3em}c_4^{\sigma}}^{\Theta(1)}\hspace{-0.3em}  + \hspace{-0.2em}
    \frac{c_6^{\sigma}(\sqrt{N}\hspace{-0.2em}-\hspace{-0.2em}1)}{2N}(m_u^2\hspace{-0.2em} + \hspace{-0.2em}m_v^2)\hspace{-0.2em} + \hspace{-0.2em}\overbrace{\frac{c_6^{\sigma}(N-1)}{2N}\hspace{-0.4em}\sum_{m=1}^M(m_{u_m}^2\hspace{-0.5em} + \hspace{-0.2em} m_{v_m}^2)}^{\Theta(1)}\hspace{-0.2em}\biggr)\hspace{-0.3em}\\
    & + 
    c_6^{\sigma}\hspace{-0.5em} \sum_{i=0,2,4}\hspace{-0.4em} \frac{m_v^i m_u^{4-i}c_{i,4-i}^{\ell}}{\sqrt{N}^4i!(4-i)!} \hspace{-0.2em} \;\frac{(\sqrt{N}-1)}{\sqrt{N}}m_u \Biggr] + \hspace{-0.1em} 4\beta R^2 m_u.
    \end{align*}
Taking the limit for $N \to +\infty$, the population drift for $m_u$ reads as
\begin{equation*}
    A_{m_u} = \sum_{i=0,2} \frac{m_v^i m_u^{3-i}c_{i,4-i}^{\ell}}{i! (3-i)!}c_4^{\sigma} + \frac{c_6^{\sigma}}{2} \sum_{m=1}^M (m_{u_m}^2 + m_{v_m}^2) + 4\beta R^2 m_u.
\end{equation*}
    By doing the same for the other principal components, we get that
    \begin{align*}
    \hspace{-0.3em}\mathcal{A}_N m_{v} 
     & = \hspace{-0.3em}\frac{\sqrt{N}}{2} N
    \Biggl[ \sum_{i=0,2}\hspace{-0.3em} \frac{m_v^i m_v^{3-i}c_{i+1,3-i}^{\ell}}{\sqrt{N}^3 i! (3-i)!} \hspace{-0.2em}\biggl(\hspace{-0.3em}c_4^{\sigma}\hspace{-0.3em}  + 
    \frac{c_6^{\sigma}(\sqrt{N}\hspace{-0.3em}-\hspace{-0.3em}1)}{2N}(m_u^2 \hspace{-0.3em}+\hspace{-0.2em} m_v^2)\hspace{-0.3em} + \hspace{-0.3em}\frac{c_6^{\sigma}(N-1)}{2N}\hspace{-0.3em}\sum_{m=1}^M(m_{v_m}^2 + m_{v_m}^2)\hspace{-0.3em}\biggr)\hspace{-0.3em}\\
    & + c_6^{\sigma}\hspace{-0.5em} \sum_{i=0,2,4}\hspace{-0.4em} \frac{m_v^i m_u^{4-i}c_{i,4-i}^{\ell}}{\sqrt{N}^4i!(4-i)!} \hspace{-0.2em} \;\frac{(\sqrt{N}-1)}{\sqrt{N}}m_v \Biggr] + \hspace{-0.1em} 4\beta R^2 m_v,\\
     \mathcal{A}_{N}m_{u_{m}} \hspace{-0.1em} & = \underbrace{\frac{\sqrt{N}}{2} c_6^{\sigma} \ N \sum_{i=0, 2, 4} \frac{ m_v^i m_u^{4-i}}{\sqrt{N^4}i!(4-i)!}c_{i,4-i}^{\ell} \frac{(N -1)}{\sqrt{N}} m_{u_m}}_{\Theta(1)} + 4\beta R^2 m_{u_m},
    \end{align*}
and same for $m_{v_{m}}$, for any $m = 1, \dots, M$. By taking the limit for $N \to +\infty$, we obtain the thesis. Note that the calculations do not change for the larger range of scalings $\sqrt{N} \lesssim \lambda_m \ll N$.
\end{proof}

\begin{figure*}
\centering
\includegraphics[width=0.47\linewidth]{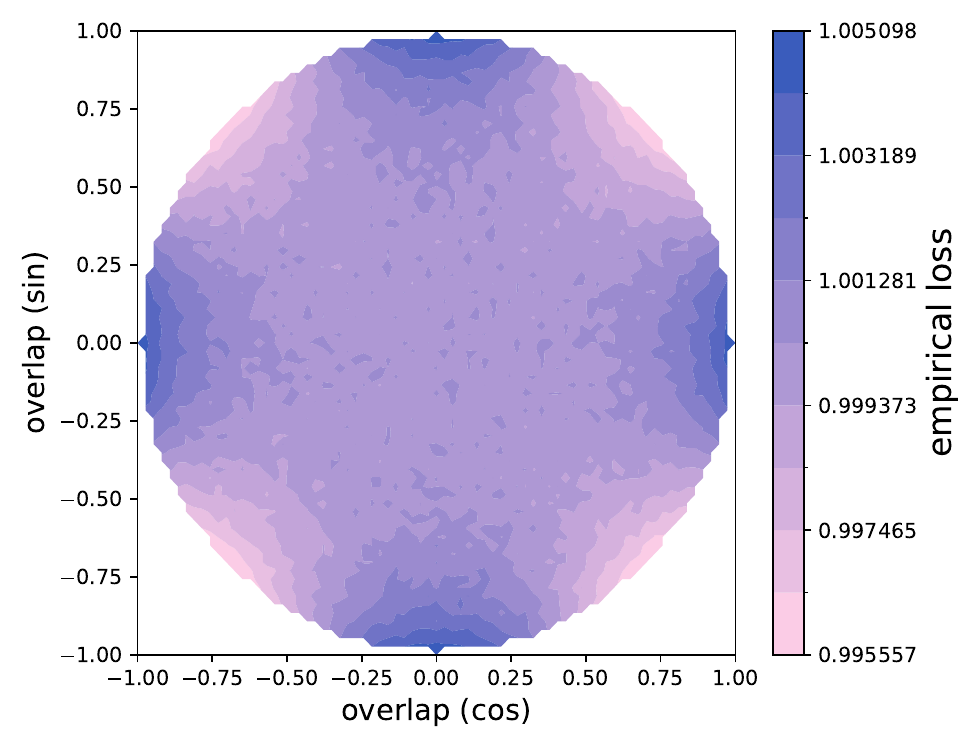}
\includegraphics[width=0.47\linewidth]{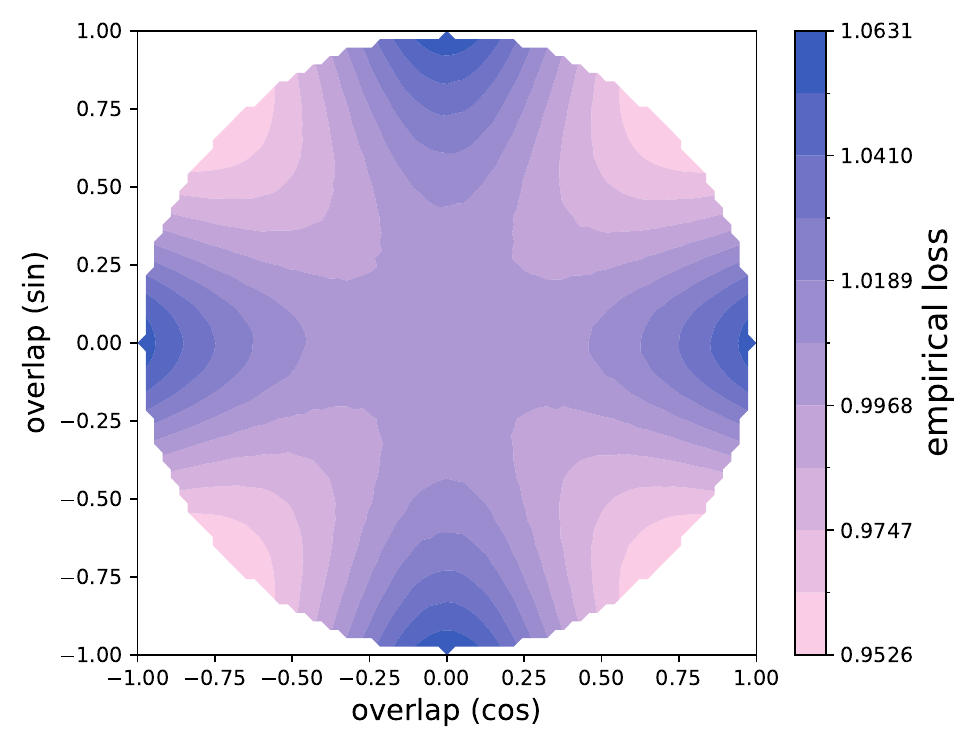}
\caption{\textbf{Level sets of the empirical loss for isotropic (left) and power-law decaying (right) inputs sampled from the Fourier data model.}  
The acceleration due to the increase of the drift magnitude in the anisotropic setting corresponds to a less flat landscape (note the pink steeps are closer to the origin), which resolves in a faster escape from the search phase. On the right, we have $\lambda_{k_0} = \sqrt{N}$. In both figures, we use $N^3$ samples to empirically approximate the landscape, for $N = 200$. We use the $\log \cosh$ activation function.}
\end{figure*}

\begin{remark}\label{app:vanishing}
    Note that, if the leading eigenvalues in Lemma \ref{app:rescaledpopulationdrif} are near-isotropic, the population drift $A_{\boldsymbol{m}}$ for the rescaled statistics turns out to be identically zero, due to the limit taken for $N \to +\infty$. The same happens if $\lambda_{k_0}^2 = o(N)$.
\end{remark}

\end{document}